\documentclass{article}

\usepackage[margin=1.0in]{geometry}
\usepackage[utf8]{inputenc} 
\usepackage[T1]{fontenc}    
\usepackage{url}            
\usepackage{booktabs}       
\usepackage{amsfonts}       
\usepackage{nicefrac}       
\usepackage{microtype}      
\usepackage{xcolor}
\usepackage{lineno}

\usepackage{graphicx}
\usepackage{biblatex}
\usepackage{algorithm}
\usepackage[noend]{algpseudocode}
\usepackage{amsmath}
\usepackage{amsthm}
\usepackage{enumitem}
\usepackage{authblk}
\algnewcommand{\IIf}[1]{\State\algorithmicif\ #1\ \algorithmicthen}
\algnewcommand{\EndIIf}{\unskip\ \algorithmicend\ \algorithmicif}

\usepackage{todonotes}
\usepackage{datenumber}
\usepackage{eqname}
\usepackage{float}
\usepackage{setspace}

\newtheorem{property}{Property}
\newtheorem{theorem}{Theorem}

\newcommand{\nisha}{\textcolor{black}}
\newcommand{\bala}{\textcolor{black}}
\newcommand{\jordan}{\textcolor{black}}

\title{Explainable AI for Trees: From Local Explanations to Global Understanding}

\author[1]{Scott M. Lundberg}
\author[1,2]{Gabriel Erion}
\author[1]{Hugh Chen}
\author[1,2]{Alex DeGrave}
\author[3]{Jordan M. Prutkin}
\author[4,5]{Bala Nair}
\author[6]{Ronit Katz}
\author[6]{Jonathan Himmelfarb}
\author[6]{Nisha Bansal}
\author[1,*]{Su-In Lee}

\affil[1]{{\small Paul G. Allen School of Computer Science and Engineering, University of Washington}}
\affil[2]{{\small Medical Scientist Training Program, University of Washington}}
\affil[3]{{\small Division of Cardiology, Department of Medicine, University of Washington}}
\affil[4]{{\small Department of Anesthesiology and Pain Medicine, University of Washington}}
\affil[5]{{\small Harborview Injury Prevention and Research Center, University of Washington}}
\affil[6]{{\small Kidney Research Institute, Division of Nephrology, Department of Medicine, University of Washington}}
\affil[*]{{\small Corresponding: suinlee@cs.washington.edu}}

\begin{document}


\setcounter{page}{1}

\date{}

{\setstretch{1}
\maketitle
}

\noindent \small {\bf One sentence summary:} Explanations for ensemble tree-based predictions; a unique exact solution that guarantees desirable explanation properties.
\\

\begin{abstract}
\noindent
Tree-based machine learning models such as random forests, decision trees, and gradient boosted trees are the most popular non-linear predictive models used in practice today, yet comparatively little attention has been paid to explaining their predictions. Here we significantly improve the interpretability of tree-based models through three main contributions: 1) The first polynomial time algorithm to compute optimal explanations based on game theory. 2) A new type of explanation that directly measures local feature interaction effects. 3) A new set of tools for understanding global model structure based on combining many local explanations of each prediction. We apply these tools to three medical machine learning problems and show how combining many high-quality local explanations allows us to represent global structure while retaining local faithfulness to the original model. These tools enable us to i) identify high magnitude but low frequency non-linear mortality risk factors in the general US population, ii) highlight distinct population sub-groups with shared risk characteristics, iii) identify non-linear interaction effects among risk factors for chronic kidney disease, and iv) monitor a machine learning model deployed in a hospital by identifying which features are degrading the model's performance over time. Given the popularity of tree-based machine learning models, these improvements to their interpretability have implications across a broad set of domains.
\end{abstract}

\section{Introduction} 
Machine learning models based on trees are the most popular non-linear models in use today \cite{kaggle2017,friedman2001elements}. Random forests, gradient boosted trees, and other tree-based models are used in finance, medicine, biology, customer retention, advertising, supply chain management, manufacturing, public health, and many other areas to make predictions based on sets of input features (Figure~\ref{fig:overview}A left). In these applications it is often important to have models that are \emph{both} accurate and interpretable, where being interpretable means that we can understand how the model uses the input features to make predictions \cite{lundberg2017unified}. Yet while there is a rich history of {\it global} interpretation methods for trees that summarize the impact of input features on the model as a whole, much less attention has been paid to {\it local} explanations that explain the impact of input features on individual predictions (i.e. for a single sample) (Figure~\ref{fig:overview}A).

\begin{figure}
  \centering
  \includegraphics[width=1.0\textwidth]{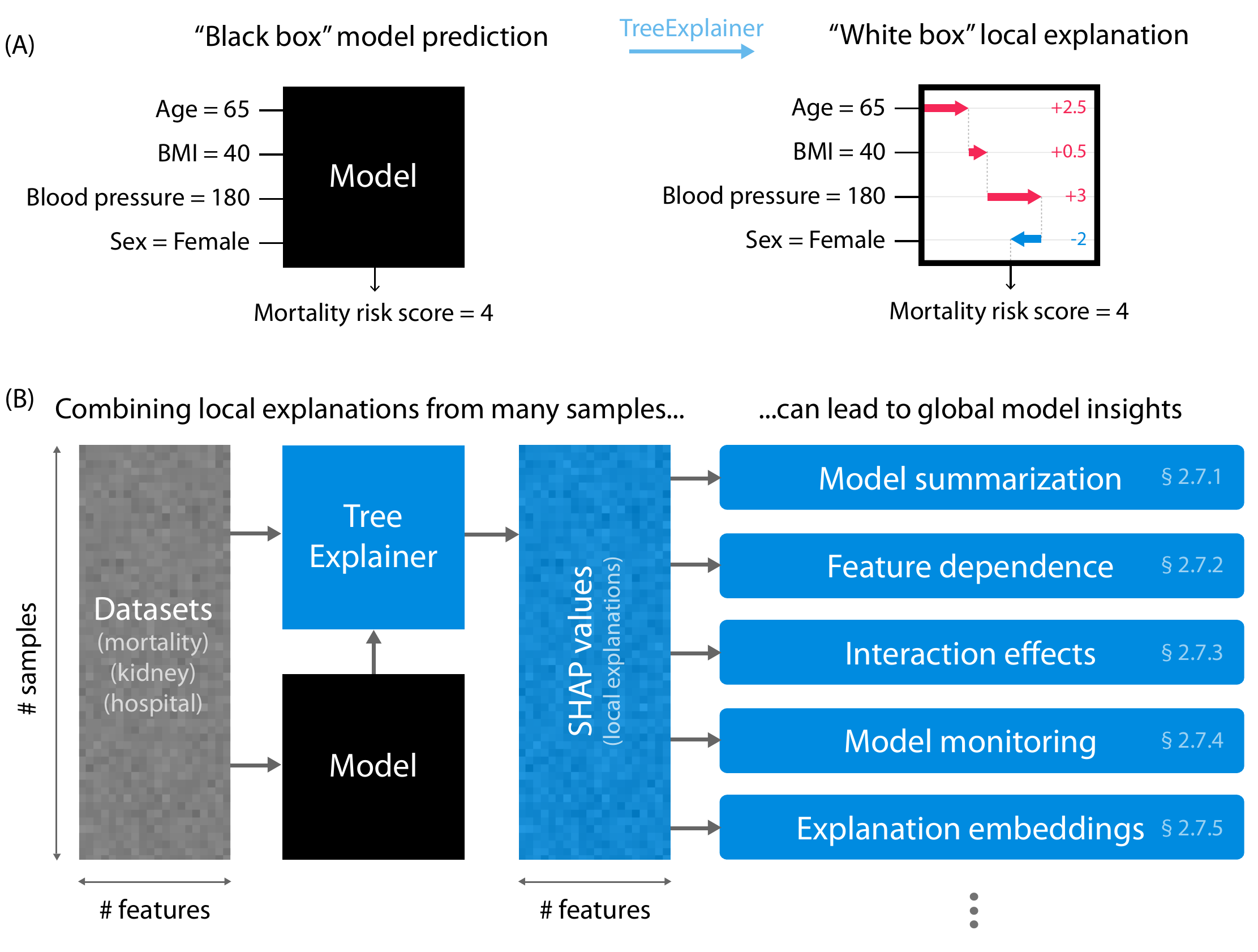}
  \caption[Paper overview]{{\bf Local explanations based on TreeExplainer enable a wide variety of new ways to understand global model structure.} (A) A local explanation based on assigning a numeric measure of credit to each input feature (Section~\ref{sec:tree_explainer}). (B) By combining many local explanations we can represent global structure while retaining local faithfulness to the original model. To demonstrate this we use three illustrative medical datasets to train gradient boosted decision trees and then compute local explanations based on SHapley Additive exPlanation (SHAP) values (Section~\ref{sec:tree_explainer}). Computing local explanations across all samples in a dataset enables many new tools for understanding global model structure (Section~\ref{sec:building_blocks}).} 
  \label{fig:overview}
\end{figure}

There are three ways we are aware of to explain individual predictions from trees (i.e. local explanation methods): 1) reporting the decision path; 2) an unpublished heuristic approach that assigns credit to each input feature \cite{treeinterpreter}; and 3) a variety of model-agnostic approaches that require executing the model many times for each explanation \cite{ribeiro2016should,datta2016algorithmic,vstrumbelj2014explaining,lundberg2017unified,baehrens2010explain}. These methods have the following limitations: 1) Simply reporting the decision path of a prediction is unhelpful for most models, particularly those based on multiple trees. 2) The behavior of the heuristic credit allocation approach has not yet been carefully analyzed, and as we show in Section~\ref{sec:inconsistent_local}, it is strongly biased to alter the impact of features based on their tree depth. 3) Since model-agnostic methods rely on post-hoc modeling of an arbitrary function, they can be slow and suffer from sampling variability (Section~\ref{sec:slow_agnostic}).

Here we propose TreeExplainer, a new local explanation method for trees that enables the tractable 
computation of \emph{optimal} local explanations, as defined by desirable properties from game theory (Section~\ref{sec:tree_explainer}).
It bridges theory to practice by building on our previous 
model-agnostic work based on the classic game-theoretic Shapley values \cite{lundberg2017unified,shapley1953value} and leads to three notable improvements: 

\begin{enumerate}[leftmargin=*]
\item {\it TreeExplainer enables the exact computation of optimal local explanations for tree-based models.} 
The classic Shapley values can be considered ``optimal'' in the sense that within a large class of approaches they are the only way to measure feature importance while maintaining several natural properties from cooperative game theory \cite{lundberg2017unified}. Unfortunately, in general these values can only be approximated since computing them exactly is NP-hard \cite{matsui2001np}, requiring a summation over all feature subsets. Sampling based approximations have been proposed \cite{vstrumbelj2014explaining,lundberg2017unified}, but using these methods to compute low variance versions of the results in this paper for even our smallest dataset would take years of CPU time (Section~\ref{sec:slow_agnostic}). However, by focusing specifically on trees we were able to develop an algorithm that computes local explanations based on the exact Shapley values in polynomial time. This enables us to provide local explanations that come with \emph{theoretical guarantees} of {\it local accuracy} and {\it consistency} \cite{lundberg2017unified} (defined in Section~\ref{sec:tree_explainer}; \ref{sec:methods_shap_unity}).

\item {\it TreeExplainer extends local explanations to directly capture feature interactions.}
Local explanations that assign a single number to each input feature are very intuitive, but they cannot directly represent \emph{interaction} effects. We provide a theoretically grounded way of measuring local interaction effects based on a generalization of Shapley values proposed in game theory literature \cite{fujimoto2006axiomatic}. We show that this can provide valuable insights into a model's behavior (Section~\ref{sec:interaction_effects}).

\item {\it TreeExplainer provides a new set of tools for understanding global model structure based on many local explanations.} The ability to efficiently and exactly compute local explanations using Shapley values across an entire dataset enables a whole range of new tools to understand the global behavior of the model (Figure~\ref{fig:overview}B; Section~\ref{sec:building_blocks}). 
We show that combining many local explanations allows us to represent global structure while retaining {\it local faithfulness} \cite{ribeiro2018anchors} to the original model, which produces more detailed and accurate representations of model behavior.
\end{enumerate}

The need to explain predictions from tree models is widespread. It is particularly important in medical applications, where the patterns uncovered by a model are often even more important than the model's prediction performance \cite{shortliffe2018clinical,lundberg2018nature}. We use three medical datasets to demonstrate the value of TreeExplainer (\ref{sec:methods_datasets}); they represent three types of loss functions (\ref{sec:methods_model_training_details}): 1) {\it Mortality} -- a mortality dataset with 14,407 individuals and 79 features based on the NHANES I Epidemiologic Followup Study \cite{cox1997plan}, where we model the risk of death over twenty years of followup. 2) {\it Chronic kidney disease} -- a kidney dataset that follows 3,939 chronic kidney disease patients from the Chronic Renal Insufficiency Cohort study over 10,745 visits with the goal of using 333 features to classify if patients will progress to end-stage renal disease within 4 years. 3) {\it Hospital procedure duration} -- a hospital electronic medical record dataset with 147,000 procedures and 2,185 features, where we predict the duration of an upcoming procedure.

We discuss why tree models are the most appropriate models in many situations, both because of their accuracy (Section~\ref{sec:why_trees_accurate}), and their interpretability (Section~\ref{sec:why_trees_interpretable}). We discuss the need for better local explanations of tree-based models (Sections~\ref{sec:inconsistent_local}-\ref{sec:slow_agnostic}), and how we address that need with TreeExplainer (Section~\ref{sec:tree_explainer}). We then extend local explanations to capture interaction effects (Section~\ref{sec:shap_interaction_values}). Finally, we demonstrate the value of the new explainable AI tools enabled by combining many local explanations from TreeExplainer (Section~\ref{sec:building_blocks}). To enable the wide use of TreeExplainer, high-performance implementations have also been released and integrated with many major tree-based machine learning packages (\mbox{\url{https://github.com/suinleelab/treeexplainer-study}}).

\section{Results}

\begin{figure}
  \centering
  \makebox[\textwidth][c]{\includegraphics[width=1.15\textwidth]{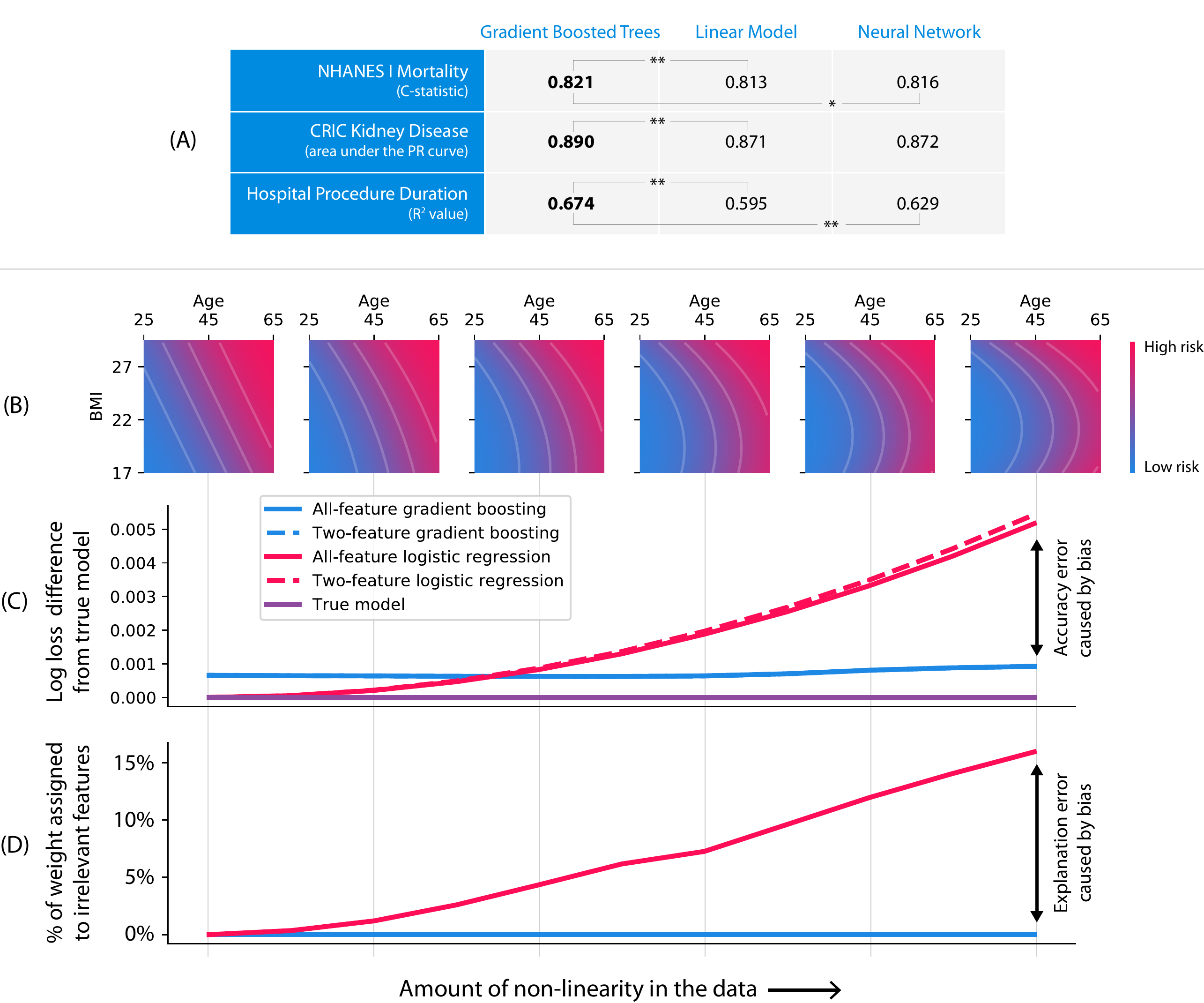}}
  \caption{{\bf Gradient boosted tree models can be both more accurate than neural networks and more interpretable than linear models.} (A) Gradient boosted tree models outperform both linear models and neural networks on all our medical datasets. (**) represents a P-value $< 0.01$, and (*) represents a P-value of $0.03$ (\ref{sec:methods_model_training_details}).
  (B-D) Linear models exhibit explanation error as well as accuracy error in the presence of non-linearity. (B) Data generating models used for the simulation, ranging from linear to quadratic along the body mass index (BMI) dimension. (C) The test performance of linear logistic regression (red) is better than gradient boosting (blue) up until a specific amount of non-linearity. 
  Not surprisingly, the bias of the linear model is higher than the gradient boosting model as shown by the steeper slope as we increase the non-linearity. 
  (D) As the true function becomes more non-linear the linear model assigns more credit (coefficient weight) to features that were not used by the data generating model.}
  \label{fig:performance_and_synth}
\end{figure}

\subsection{Tree-based models can be more accurate than neural networks} 
\label{sec:why_trees_accurate}
Tree-based ensemble methods such as random forests and gradient boosted trees achieve state-of-the-art performance in many domains. They have a long history of use in machine learning \cite{friedman2001elements},
and new high-performance implementations are an active area of research \cite{chen2016xgboost,ke2017lightgbm,prokhorenkova2018catboost,pedregosa2011scikit}. While deep learning models are more appropriate in fields like image recognition, speech recognition, and natural language processing, tree-based models consistently outperform standard deep models on tabular-style datasets where features are individually meaningful and do not have strong multi-scale temporal or spatial structures \cite{chen2016xgboost}. A balance of computational efficiency, ease of use, and high accuracy have made tree-based models the most popular non-linear model type; four out of the five most popular machine learning models used by data scientists involve trees \cite{kaggle2017}. The three medical datasets we examine here all represent tabular-style data, and gradient boosted trees outperform both deep learning and linear regression across all three datasets (Figure~\ref{fig:performance_and_synth}A) (\ref{sec:methods_model_training_details}).

\subsection{Tree-based models can be more interpretable than linear models}
\label{sec:why_trees_interpretable}

While it is well-known that the bias/variance trade-off in machine learning has implications for model accuracy, it is less appreciated that the trade-off also affects interpretability. While simple high-bias models (such as linear models) are often easy to understand, they are also more sensitive to model mismatch -- where the true relationships in the data do not match the form of the model. 

To illustrate why low-bias models can be more interpretable than high-bias models we compare gradient boosted trees with lasso regularized linear 
logistic regression using the mortality dataset. We simulated a binary outcome based on a participant's age and body mass index (BMI) (\ref{sec:methods_sim_bias_experiement_details}), and varied the amount of non-linearity in the simulated relationship (Figure~\ref{fig:performance_and_synth}B). As expected, when we increase the non-linearity, the bias of the linear model causes a drop in accuracy (Figure~\ref{fig:performance_and_synth}C). However, what is perhaps unexpected is that it also causes a drop in interpretability (Figure~\ref{fig:performance_and_synth}D). We know that the model should only depend on age and BMI, but even a moderate amount of non-linearity in the true relationship causes the linear model to start using other irrelevant features (Figure ~\ref{fig:performance_and_synth}D). This means that even when a linear model can achieve the same test accuracy as a gradient boosted tree model, the gradient boosted tree model is preferable because its connection to the training data is more interpretable 
(Figure~\ref{fig:performance_and_synth}D; \ref{sec:methods_sim_bias_experiement_details}).

\subsection{Current local explanations for tree-based models are inconsistent}
\label{sec:inconsistent_local}
Despite the long history of approaches designed to compute global measures of feature importance in ensemble tree models 
(\ref{sec:methods_previous_global_tree}), to our knowledge, there are only two approaches to quantify a feature's \emph{local} importance for an individual prediction.
The first is simply reporting the decision path, which is unhelpful for ensembles of many trees. The second is an unpublished heuristic approach (proposed by {\it Saabas}) that explains a prediction by following the decision path and attributing changes in the expected output of the model to each feature along the path (\ref{sec:methods_previous_local_tree}). The Saabas method has not been well studied, and we demonstrate here that it is strongly biased to alter the impact of features based on their distance from the root of a tree (Supplementary Figure~\ref{fig:tree_shap_performance}A). This causes Saabas values to be {\it inconsistent}, which means we can modify a model to make a feature clearly more important, and yet the Saabas value attributed to that feature will decrease (Supplementary Figure~\ref{fig:and_trees_simple}). The difference this makes can be seen by examining trees representing multi-way AND functions. No feature in an AND function should have any more credit than another, yet Saabas values give splits near the root much less credit than splits near the leaves (Supplementary Figure~\ref{fig:tree_shap_performance}A). Consistency is critical for an explanation method because it makes comparisons among feature importance values meaningful.


\subsection{Model-agnostic local explanations are slow and variable}
\label{sec:slow_agnostic}
Model-agnostic local explanation approaches can be used to explain tree models (\ref{sec:methods_previous_agnostic}), but they rely on post-hoc modeling of an arbitrary function and thus can be slow and/or suffer from sampling variability when applied to models with many input features. To illustrate this we generated random datasets of increasing size and then explained (over)fit XGBoost models with 1,000 trees. This experiment shows a linear increase in complexity as the number of features increases; model-agnostic methods take a significant amount of time to run over these datasets, even though we allowed for non-trivial estimate variability (Supplementary Figure~\ref{fig:tree_shap_performance}D; \ref{sec:methods_agnostic_convergence}) and only used a moderate numbers of features (Supplementary Figure~\ref{fig:tree_shap_performance}C). Calculating low variance estimates of the Shapley values for the results in this paper would be intractable; just the chronic kidney disease dataset experiments would have taken almost 2 CPU days for basic explanations, and over 3 CPU years for interaction values (Section~\ref{sec:shap_interaction_values}) (Supplementary Figure~\ref{fig:tree_shap_performance}E-F; \ref{sec:methods_agnostic_convergence}). While often practical for individual explanations, model-agnostic methods can quickly become impractical for explaining entire datasets (Supplementary Figure~\ref{fig:tree_shap_performance}C-F).

\subsection{TreeExplainer provides fast local explanations with guaranteed consistency}
\label{sec:tree_explainer}

Here we introduce a new local feature attribution method for trees, TreeExplainer, which can exactly compute the classic Shapley values from game theory \cite{shapley1953value}. TreeExplainer bridges theory to practice by reducing the complexity of exact Shapley value computation from exponential to polynomial time. 
%
%
This is important since within the class of {\it additive feature attribution methods}, a class that we have shown contains many previous approaches to local feature attribution \cite{lundberg2017unified}, results from game theory imply the Shapley values are the only way to satisfy three important properties:
{\it local accuracy}, {\it consistency}, and {\it missingness} (\ref{sec:methods_shap_unity}). Local accuracy 
states that when approximating the original model $f$ for a specific input $x$, the explanation's attribution values should sum up to the output $f(x)$. Consistency 
states that if a model changes so that some feature's contribution increases or stays the same regardless of the other inputs, that input's attribution should not decrease. Missingness 
is a trivial property satisfied by all previous explanation methods (\ref{sec:methods_shap_unity}). 

Shapley values, as applied here to feature importance, are defined as the sequential impact on the model's output of observing each input feature's value, averaged over all possible feature orderings (Supplementary Figure~\ref{fig:number_line}; Equation~\ref{eq:shapley} in Methods). This means for each of \emph{all possible} orderings, we introduce features one at a time into a conditional expectation of the model's output, then attribute the change in expectation to the feature 
that was introduced. Since Shapley values can be computed using any set function, not just conditional expectations, we will use the more specific term, {\it SHapley Additive exPlanation (SHAP) values} \cite{lundberg2017unified}, to clarify that we are using conditional expectations to measure the impact of a set of features on the model.

Perhaps surprisingly, the independently developed Saabas values are computed the same way as SHAP values, but rather than averaging over all feature orderings, Saabas values only consider the single ordering defined by a tree's decision path. This connection leads to two new insights: 1) The bias and consistency problems of Saabas values (Section~\ref{sec:inconsistent_local}) result from a failure to average feature impacts over all orderings. 2) For an infinite ensemble of random fully-developed trees on binary features, Saabas values effectively consider all orderings and so converge to the SHAP values. In practice, however, tree ensemble models are not infinite, random, or fully developed; to guarantee consistency we need to compute SHAP values exactly. TreeExplainer makes this possible by \emph{exactly} computing SHAP values in \emph{low order polynomial time}. 
This represents an exponential complexity improvement over previous exact Shapley methods. By default, TreeExplainer computes conditional expectations using tree traversal, but it also provides an option that enforces feature independence and supports explaining a model's loss function (\ref{sec:methods_tree_explainer}).

Efficiently and exactly computing the Shapley values guarantees that explanations will always be consistent and locally accurate. This results in several improvements over previous local explanation methods:

\begin{itemize}[leftmargin=*]

\item {\it TreeExplainer impartially assigns credit to input features regardless of their depth in the tree.} In contrast to Saabas values, TreeExplainer allocates credit uniformly among all features participating in multi-way AND operations (Supplementary Figures~\ref{fig:tree_shap_performance}A-B) and avoids inconsistency problems (Supplementary Figure~\ref{fig:and_trees_simple}).


\item {\it For moderate sized models, TreeExplainer is several orders of magnitude faster than model-agnostic alternatives, and has zero estimation variability} (Supplementary Figures~\ref{fig:tree_shap_performance}C-F). Since solutions from model-agnostic sampling methods are approximate, there is always the additional burden of checking their convergence and accepting a certain amount of noise in their estimates. This burden is eliminated by TreeExplainer's exact explanations.

\item {\it TreeExplainer consistently outperforms alternative methods across a benchmark of 21 different local explanation metrics} (Figure~\ref{fig:cric_tile}; Supplementary Figures~\ref{fig:corrgroups60_tile}-\ref{fig:independentlinear60_tile}). We designed 21 metrics to comprehensively evaluate the performance of local explanation methods, and applied these metrics to eight different explanation methods across three different model types and three datasets (\ref{sec:benchmark_metrics}). The results for the chronic kidney disease dataset are shown in Figure~\ref{fig:cric_tile}, and demonstrate consistent performance improvement for TreeExplainer. 
\item {\it TreeExplainer matches human intuition across a benchmark of 12 user study scenarios} (Supplementary Figure~\ref{fig:human_tile}). We evaluated how well explanation methods match human intuition by comparing their outputs with human consensus explanations of 12 scenarios based on simple models. In contrast to the heuristic Saabas values, Shapley value based explanation methods agree with human intuition in all the scenarios we tested (\ref{sec:methods_user_study}).
\end{itemize}

TreeExplainer simultaneously addresses the consistency issues faced by the heuristic Saabas values 
(Section~\ref{sec:inconsistent_local}), and the computational issues faced by model-agnostic methods 
(Section~\ref{sec:slow_agnostic}). This leads to fast practical explanations with strong theoretical guarantees that result in improved performance across many quantitative metrics.

\begin{figure*}
  \centering
  \makebox[\textwidth][c]{\includegraphics[width=1.2\textwidth]{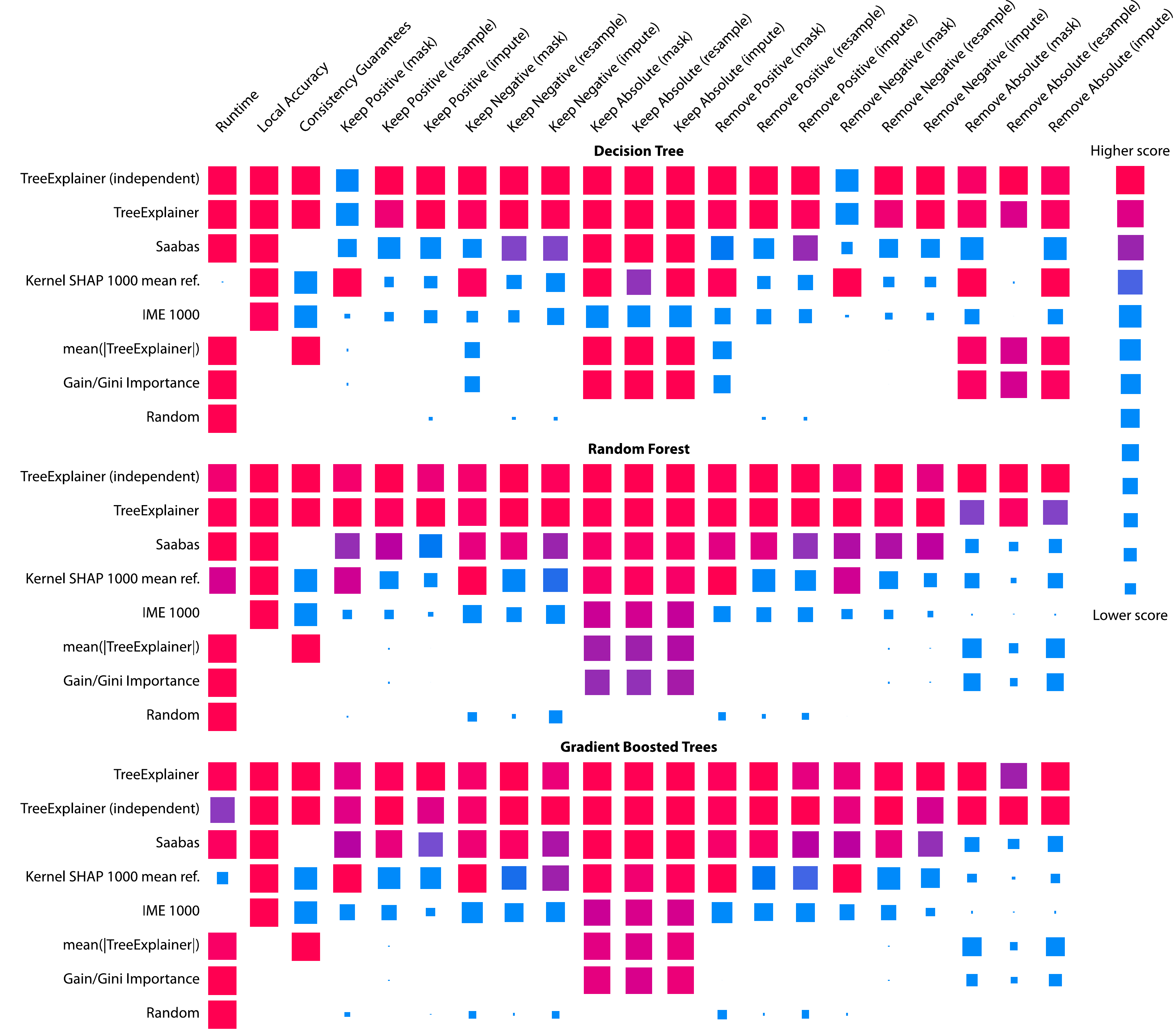}}
  \caption{{\bf Explanation method performance across 21 different evaluation metrics and three classification models in the chronic kidney disease dataset.} Each tile represents the performance of a local explanation method on a given metric for a given model. Within each model the columns of tiles are scaled between the minimum and maximum value, and methods are sorted by their overall performance. TreeExplainer outperforms previous approaches not only by having theoretical guarantees of consistency, but also across a large set of other metrics (\ref{sec:benchmark_metrics}). When these experiments were repeated in two synthetic datasets, TreeExplainer remained the top performing method (Supplementary Figures~\ref{fig:corrgroups60_tile}~and~\ref{fig:independentlinear60_tile}). Note that, as expected, Saabas becomes a better approximation to the Shapley values (and so a better attribution method) as the number of trees increases (\ref{sec:methods_shap_unity}).
  }
  \label{fig:cric_tile}
\end{figure*}

\subsection{TreeExplainer extends local explanations to measure interaction effects}
\label{sec:shap_interaction_values}
Traditionally, local explanations based on feature attribution assign a single number to each input feature. The simplicity of this natural representation comes at the cost of combining main and interaction effects. While interaction effects between features can be reflected in the global patterns of many local explanations, their distinction from main effects is lost in each local explanation
(Section~\ref{sec:feature_dependence}; Figure~\ref{fig:overview_and_dependence}B-G).

Here we propose {\it SHAP interaction values} as a new richer type of local explanation (\ref{sec:methods_shap_interaction_values}). These values use the `Shapley interaction index,' a relatively recent concept from game theory, to capture local interaction effects. 
They follow from generalizations of the original Shapley value properties \cite{fujimoto2006axiomatic} and allocate credit not just among each player of a game, but among all pairs of players. 
The SHAP interaction values consist of a matrix of feature attributions (the main effects on the diagonal and the interaction effects on the off-diagonal) and have uniqueness guarantees similar to SHAP values \cite{fujimoto2006axiomatic}. By enabling the separate consideration of main and interaction effects for individual model predictions, TreeExplainer can uncover important patterns that might otherwise be missed (Section~\ref{sec:interaction_effects}).

\subsection{Local explanations from TreeExplainer can be used as building blocks for global understanding}
\label{sec:building_blocks}
We present five new methods that combine many local explanations to provide global insight into a model's behavior. This allows us to retain local faithfulness to the model while still capturing global patterns, resulting in richer, more accurate representations of the model's behavior. 
Each application presented below illustrates how local explanations can be used as building blocks for explainable machine learning. For all experiments we use gradient boosted trees since they have high accuracy (Figure~\ref{fig:performance_and_synth}A), low bias (Figure~\ref{fig:performance_and_synth}B-D), and support fast exact local explanations through TreeExplainer (Sections~\ref{sec:tree_explainer}~and~\ref{sec:shap_interaction_values}). 

\begin{figure*}
  \centering
  \makebox[\textwidth][c]{\includegraphics[width=1.2\textwidth]{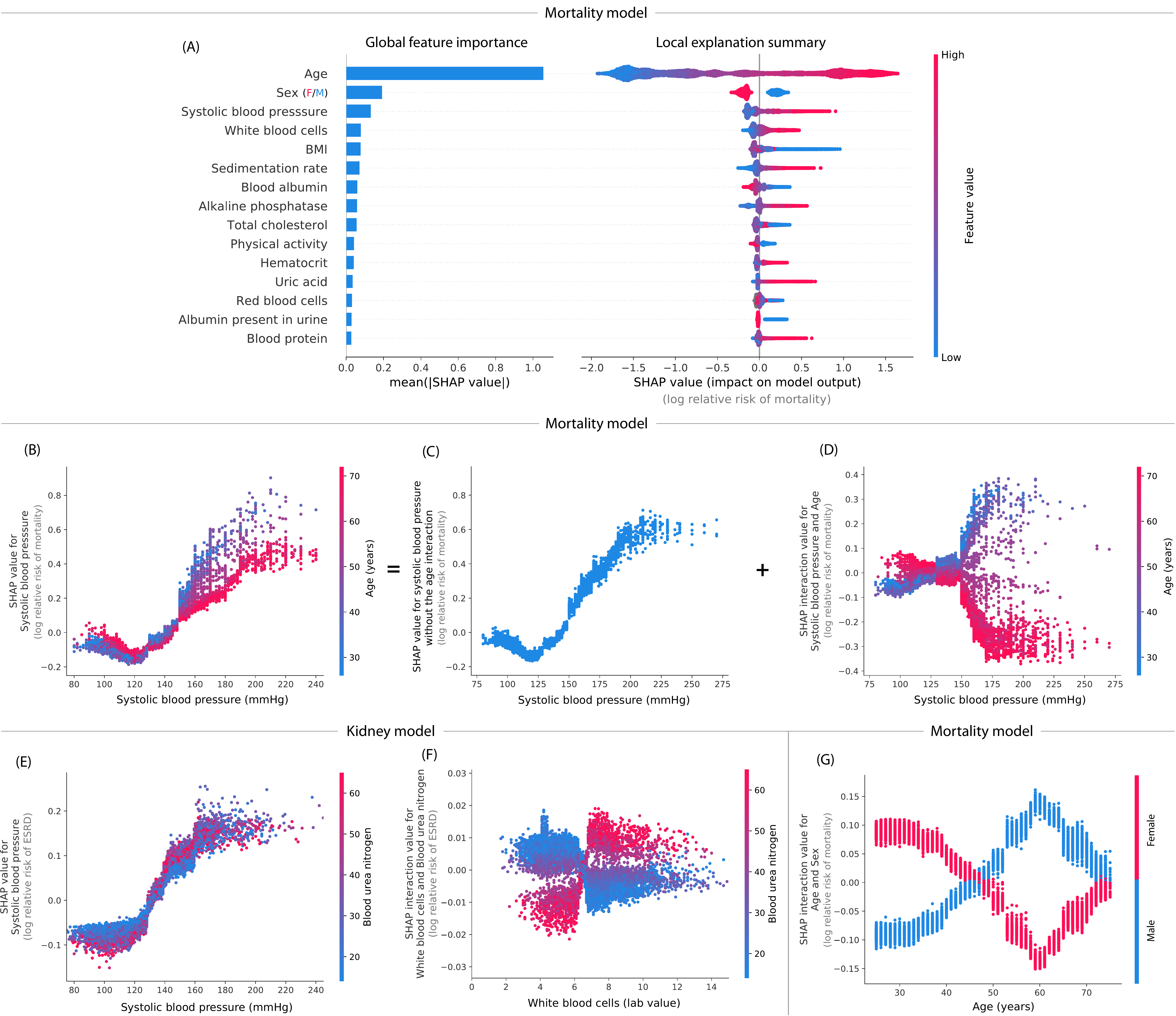}}
  \caption{{\bf By combining many local explanations we can provide rich summaries of both an entire model and individual features.} (A) Bar chart (left) and SHAP summary plot (right) for a gradient boosted decision tree model trained on the mortality dataset. The long right tails in the summary plot are from rare but high-magnitude risk factors. (B) SHAP dependence plot of systolic blood pressure vs. its SHAP value in the mortality model. A clear interaction effect with age is visible that increases the impact of early onset high blood pressure. (C) Using SHAP interaction values we can remove the interaction effect of age from the model. (D) Plotting just the interaction effect of systolic blood pressure with age shows how the effect of systolic blood pressure on mortality risk varies with age. Adding the y-values of C and D produces B. (E) A dependence plot of systolic blood pressure vs. its SHAP value in the kidney model shows an increase in kidney disease risk at a systolic blood pressure of 125 (which parallels the increase in mortality risk). (F) Plotting the SHAP interaction value of `white blood cells' with `blood urea nitrogen' shows that high white blood cell counts increase the negative risk conferred by high blood urea nitrogen. (G) Plotting the SHAP interaction value of sex vs. age in the mortality model shows how the differential risk of men and women changes over their lifetimes.}
  \label{fig:overview_and_dependence}
\end{figure*}

\subsubsection{Local model summarization reveals rare high-magnitude effects on mortality risk and increases feature selection power}
\label{sec:model_summary}

Combining local explanations from TreeExplainer across an entire dataset enhances traditional global representations of feature importance by:
1) avoiding the inconsistency problems of current methods (Supplementary Figure~\ref{fig:and_trees_simple}), 2) increasing the power to detect true feature dependencies in a dataset (Supplementary Figure~\ref{fig:feature_selection}), and 3) enabling us to build {\it SHAP summary plots} that succinctly display the magnitude, prevalence, and direction of a feature's effect. 
SHAP summary plots avoid conflating the magnitude and prevalence of an effect into a single number, and so reveal rare high magnitude effects. 
Figure~\ref{fig:overview_and_dependence}A illustrates its benefits using the mortality dataset: (left) a standard bar-chart based on the average magnitude of the SHAP values, and (right) a set of beeswarm plots where each dot corresponds to an individual person in the study. 
The position of the dot on the x-axis is the impact that feature has on the model's prediction for that person. When multiple dots land at the same x position they pile up to show density. 

Figure~\ref{fig:overview_and_dependence}A (right) reveals the direction of effects, such as men (blue) having a higher mortality risk than women (red); and the distribution of effect sizes, such as the long right tails of many medical test values. These long tails mean features with a low global importance can yet be extremely important for specific individuals. Interestingly, rare mortality effects always stretch to the right, which implies there are many ways to die abnormally early when medical measurements are out-of-range, but not many ways to live abnormally longer (\ref{sec:methods_model_summary}). 

\subsubsection{Local feature dependence reveals both global patterns and individual variability in mortality risk and chronic kidney disease}
\label{sec:feature_dependence}

{\it SHAP dependence plots} show how a feature's value (x-axis) impacted the prediction (y-axis) of every sample (each dot) in a dataset (Figures~\ref{fig:overview_and_dependence}B~and~E; \ref{sec:methods_feature_dependence}). This provides richer information than traditional partial dependence plots (Supplemental Figure \ref{fig:nhanes_pdp_sbp}). For the mortality model this reproduces the standard risk inflection point of systolic blood pressure \cite{sprint2015randomized}, while also highlighting that the impact of blood pressure risk is different for people of different ages (Figure~\ref{fig:overview_and_dependence}B). Many individuals have a recorded blood pressure of 180 mmHg in the mortality dataset, but the impact of that measurement on their mortality risk varies because early onset high blood pressure is more concerning to the model than late onset high blood pressure. These types of interaction effects show up as vertical dispersion in SHAP dependence plots.

For the chronic kidney disease model, a dependence plot again clearly reveals a risk inflection point for systolic blood pressure, but in this dataset the vertical dispersion from interaction effects appears to be partially driven by differences in blood urea nitrogen (Figure~\ref{fig:overview_and_dependence}E). Correctly modeling blood pressure risk while retaining interpretabilty is important since blood pressure control in select chronic kidney disease (CKD) populations may delay progression of kidney disease and reduce the risk of cardiovascular events (\ref{sec:methods_feature_dependence}).

\subsubsection{Local interactions reveal sex-specific life expectancy changes during aging and inflammation effects in chronic kidney disease}
\label{sec:interaction_effects}

Using SHAP interaction values, we can decompose the impact of a feature on a specific sample into a main effect and interaction effects with other features. This allows us to measure global interaction strength as well as decompose the SHAP dependence plots into main effects and interaction effects at a local (i.e., per sample) level (Figures~\ref{fig:overview_and_dependence}B-D; \ref{sec:methods_interaction_effects}). 

\nisha{In the mortality dataset, plotting the SHAP interaction value between age and sex shows a clear change in the relative risk between men and women over a lifetime (Figure~\ref{fig:overview_and_dependence}G). The largest difference in risk between men and women is at age 60. It is plausible that this increased risk is driven by increased cardiovascular mortality in men relative to women near that age \cite{mozaffarian2015heart}. This pattern is not clearly captured without SHAP interaction values because being male always confers greater risk of mortality than being female (Figure~\ref{fig:overview_and_dependence}A).}

\nisha{In the chronic kidney disease model, an interesting interaction is observed between `white blood cells' and `blood urea nitrogen' (Figure~\ref{fig:overview_and_dependence}F). High white blood cell counts are more concerning to the model when they are accompanied by high blood urea nitrogen. This supports the notion that inflammation may interact with high blood urea nitrogen to contribute to faster kidney function decline \cite{bowe2017association,fan2017white}.}

\begin{figure*}
  \centering
  \includegraphics[width=1.0\textwidth]{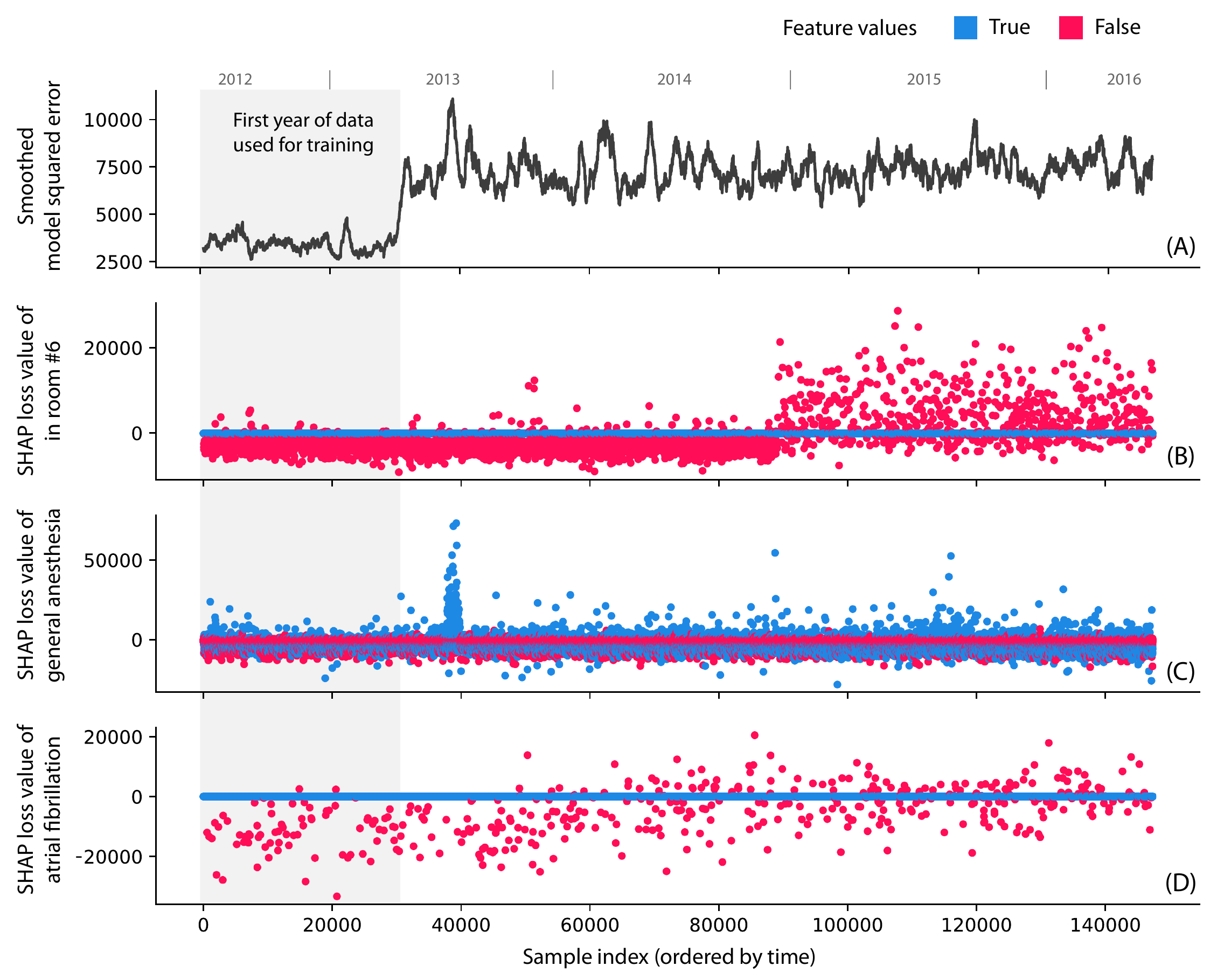}
  \caption{{\bf Monitoring plots reveal problems that would otherwise be invisible in a retrospective hospital machine learning model deployment.} 
  (A) The squared error of a hospital duration model averaged over the nearest 1,000 samples. The increase in error after training is because the test error is (as expected) higher than the training error. (B) The SHAP value of the model loss for the feature indicating if the procedure happens in room 6. The significant change is where we intentionally swapped the labels of room 6 and 13, which is invisible in the overall model loss. (C) The SHAP value of the model loss for the general anesthesia feature; the spike one-third of the way into the data is the result of a previously unrecognized transient data corruption at a hospital. (D) The SHAP value of the model loss for the atrial fibrillation feature. The upward trend of the plot shows feature drift over time (P-value $5.4 \times 10^{-19}$).
  }
  \label{fig:key_monitoring_plots}
\end{figure*}

\subsubsection{Local model monitoring reveals previously invisible problems with deployed machine learning models}
\label{sec:model_monitoring}
Here we show how using TreeExplainer to explain a model's \emph{loss}, instead of a model's prediction, can improve our ability to monitor deployed models (\ref{sec:methods_model_monitoring}). 
Deploying machine learning models in practice is challenging because of the potential for input features to change after deployment.
It is hard to detect when such changes occur, so many bugs in machine learning pipelines go undetected, even in core software at top tech companies \cite{zinkevich2017rules}. 
We demonstrate that local model monitoring helps debug model deployments by decomposing the loss among the model's input features and so identifying problematic features (if any) directly. 
This is a significant improvement over simply speculating about the cause of global model performance fluctuations. 


We simulated a model deployment with the hospital procedure duration dataset using the first year of data for training and the next three years for deployment. We present three examples: (1) is an intentional error, (2) and (3) are previously undiscovered problems.
1) We intentionally swapped the labels of operating rooms 6 and 13 two-thirds of the way through the dataset to mimic a typical feature pipeline bug. The overall loss of the model's predictions gives no indication that a problem has occurred (Figure~\ref{fig:key_monitoring_plots}A), whereas the {\it SHAP monitoring plot} for the room 6 feature clearly shows when the labeling error begins (Figure~\ref{fig:key_monitoring_plots}B). 
2) \bala{Figure~\ref{fig:key_monitoring_plots}C shows a spike in error for the general anesthesia feature shortly after the deployment window begins. This spike corresponds to a subset of procedures affected by a previously undiscovered temporary electronic medical record configuration problem (\ref{sec:methods_model_monitoring}).}
3) \jordan{Figure~\ref{fig:key_monitoring_plots}D shows an example of feature drift over time, not of a processing error. During the training period and early in deployment, using the `atrial fibrillation' feature lowers the loss; however, the feature becomes gradually less useful over time and ends up hurting the model. We found this drift was caused by significant changes in atrial fibrillation ablation procedure duration, driven by technology and staffing changes (Supplementary Figure~\ref{fig:afib_duration_plot}; \ref{sec:methods_model_monitoring}).} 

Current deployment practice is to monitor the overall loss of a model (Figure~\ref{fig:key_monitoring_plots}A) over time, and potentially statistics of input features. 
TreeExplainer enables us to instead directly allocate a model's loss among individual features.


\subsubsection{Local explanation embeddings reveal population subgroups relevant to mortality risk and complementary diagnostic indicators in chronic kidney disease}
\label{sec:embedding}

\nisha{Unsupervised clustering and dimensionality reduction are widely used to discover patterns characterizing subgroups of samples (e.g., study participants), such as disease subtypes \cite{van2010identification, sorlie2003repeated}. 
They present two drawbacks: 1) the distance metric does not account for the discrepancies among the units/meaning of features (e.g., weight vs. age), and 2) there is no way for an unsupervised approach to know which features are relevant for an outcome of interest, and so should be weighted more strongly.
We can address both of these limitations by using {\it local explanation embeddings} to embed each sample into a new ``explanation space.'' If we run clustering in this new space, we will get a {\it supervised clustering} where samples are grouped together based on their \emph{explanations}. Supervised clustering naturally accounts for the differing units of various features, only highlighting changes that are relevant to a particular outcome (\ref{sec:methods_embedding}).}

Running hierarchical supervised clustering using the mortality model results in many groups of people that share a similar mortality risk for similar reasons (Figure~\ref{fig:clustering_and_embedding}A; \ref{sec:methods_embedding}). Analogously, we can also run PCA on local explanation embeddings for chronic kidney disease samples, which uncovers the two primary categories of risk factors that identify unique individuals at risk of end-stage renal disease. This is consistent with the fact that clinically these factors should be measured in parallel
(Figures~\ref{fig:clustering_and_embedding}B-D; \ref{sec:methods_embedding}). This type of insight into the overall structure of kidney risk is not at all apparent when just looking at a standard unsupervised embedding (Supplementary Figure~\ref{fig:kidney_raw_pca}).

\begin{figure*}
  \centering
  \includegraphics[width=1.0\textwidth]{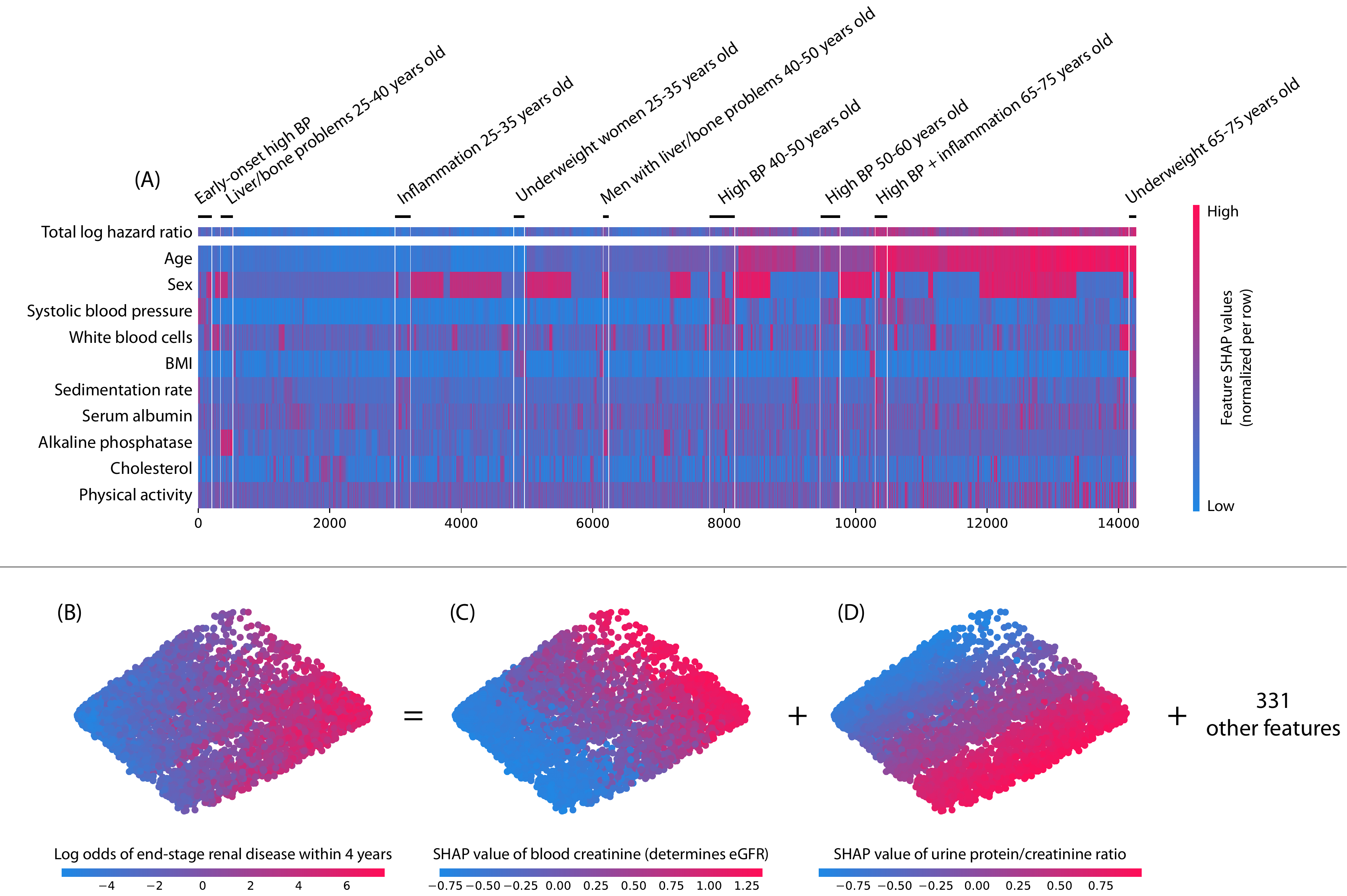}
  \caption{{\bf Local explanation embeddings support both supervised clustering and interpretable dimensionality reduction.} (A) A clustering of mortality study individuals by their local explanation embedding. Columns are patients, and rows are features' normalized SHAP values. Sorting by a hierarchical clustering reveals population subgroups that have distinct mortality risk factors. 
  (B-D) A local explanation embedding of kidney study visits projected onto two principal components. Local feature attribution values can be viewed as an embedding of the samples into a space where each dimension corresponds to a feature and all axes have the units of the model's output. (B) The embedding colored by the predicted log odds of a participant developing end-stage renal disease with 4 years of that visit. (C) The embedding colored by the SHAP value of blood creatinine. (D) The embedding colored by the SHAP value of the urine protein/creatinine ratio. Many other features also align with these top two principal components (Supplementary Figure~\ref{fig:kidney_embedding2}), and an equivalent unsupervised PCA embedding is far less interpretable (Supplementary Figure~\ref{fig:kidney_raw_pca}) }
  \label{fig:clustering_and_embedding}
\end{figure*}

\section{Discussion}

Tree-based machine learning models are used in many domains where interpretability is important. We have sought to significantly improve the interpretability of these models in three main ways: First, we present a new exact method for computing the game-theoretic Shapley values, the only explanations that have several desirable properties. 
Second, we present a new richer type of local explanation that directly captures interaction effects. Finally, we propose many new tools for model interpretation based on combining local explanations. Local explanations have a distinct advantage over global explanations because by only focusing on a single sample they can remain more faithful to the original model. We anticipate that in the future local explanations will become foundational building blocks for many downstream tasks in machine learning. 

\subsection*{Acknowledgements}
We are grateful to Ruqian Chen, Alex Okeson, Cassianne Robinson, Vadim Khotilovich, Nao Hiranuma, Joseph Janizek, Marco Tulio Ribeiro, Jacob Schreiber, and members of the Lee lab for the feedback and assistance they provided during the development and preparation of this research. This work was funded by National Science Foundation [DBI-1759487, DBI-1355899, DGE-1762114, and DGE-1256082]; American Cancer Society [127332-RSG-15-097-01-TBG]; and National Institutes of Health [R35 GM 128638].

The Chronic Renal Insufficiency Cohort (CRIC) study was conducted by the CRIC Investigators and supported by the National Institute of Diabetes and Digestive and Kidney Diseases (NIDDK). The data from the CRIC study reported here were supplied by the NIDDK Central Repositories. This manuscript was not prepared in collaboration with Investigators of the CRIC study and does not necessarily reflect the opinions or views of the CRIC study, the NIDDK Central Repositories, or the NIDDK.

\clearpage
\section*{\hfil Methods \hfil}
\setcounter{section}{0}
\renewcommand*{\thesection}{Methods~\arabic{section}}
\renewcommand*{\thesubsection}{Methods~\arabic{section}.\arabic{subsection}}

\section{Institutional review board statement}

The chronic kidney disease data used in this study was obtained from the Chronic Renal Insufficiency Cohort (CRIC) study. University Washington Human Subjects Division determined that our study does not involve human subjects because we do not have access to identifiable information (IRB ID: STUDY00006766).

The hospital procedure data used for this study was retrieved from three institutional electronic medical record and data warehouse systems after receiving approval from the Institutional Review Board (University of Washington Human Subjects Division, Approval no. 46889). Protected health information was excluded from the dataset that was used for the machine-learning methods.







\section{The three medical datasets used for experiments}
\label{sec:methods_datasets}

\subsection{Mortality dataset}
The mortality data was obtained from the National Health and Nutrition Examination Survey (NHANES~I) conducted by the U.S. Centers for Disease Control (CDC), as well as the NHANES I Epidemiologic Follow-up Study (NHEFS) \cite{cox1997plan}. Raw tape-format data files were obtained from the CDC website and converted to a tabular format by custom scripts. This reformatted version of the public data has been released at \mbox{\url{http://github.com/suinleelab/treexplainer-study}}. NHANES I examined 23,808 individuals in the United States between 1971 and 1974, recording a large number of clinical and laboratory measurements. The NHANES~I Epidemiologic Followup Study researched the status of the original NHANES I participants as of 1992 to identify when they had died, or alternatively when they were last known to be alive. After filtering NHANES~I to subjects that had both followup mortality data as well as common health measurements (such as systolic blood pressure) we obtained 79 features for 14,407 individuals, of which 4,785 individuals had recorded death dates before 1992. This data was used to train several cox proportional hazard ratio model types (Section~\ref{sec:why_trees_accurate}). Because NHANES~I represents a cross section of the United States population, it is a classic dataset that has been often used to understand the association between standard clinical measurements and long term health outcomes \cite{launer1994body,fang2000serum}.

\subsection{Chronic kidney disease dataset}

\nisha{The kidney dataset is from the Chronic Renal Insufficiency Cohort (CRIC) Study which follows individuals with chronic kidney disease recruited from 7 clinical centers \cite{lash2009chronic}. Participants are assessed at an annual visit and study follow-up is ongoing. We joined both visit and person level attributes to create 333 features for 10,745 visits among 3,939 patients. For each visit, we determined if the patient progressed to `end-stage renal disease' within the following four years. (End-stage renal disease is the last stage of chronic kidney disease when kidneys are functioning at 10-15\% of their normal capacity; dialysis or kidney transplantation is necessary to sustain life.)
Predicting this progression outcome results in a binary classification task. Understanding what leads some people with chronic kidney disease to progress to end-stage renal disease while others do not is a priority in clinical kidney care that can help doctors and patients better manage and treat their condition \cite{orlandi2018hematuria,wilson2014urinary,tangri2016multinational,matsushita2012comparison}. In the United States chronic kidney disease affects 14\% of the population, so improving our management and understanding of the disease can have a significant positive impact on public health \cite{ckdrate2018}.}

\subsection{Hospital procedure duration dataset}

\bala{Our hospital’s operating rooms have installed an Anesthesia Information Management System - AIMS (Merge AIMS, Merge Inc., Hartland, WI) that integrates with medical devices and other electronic medical record systems to automatically acquire hemodynamic, ventilation, laboratory, surgery schedule and patient registration data. The automatic capture of data is supplemented by the manual documentation of medications and anesthesia interventions to complete the anesthesia record during a surgical episode. We extracted data from the AIMS database from June 2012 to May 2016. The corresponding medical history data of each patient were also extracted from our electronic medical record data warehouse (Caradigm Inc., Redmond, WA). Patient and procedure specific data available prior to surgery were captured and summarized into 
2,185 features over 147,000 procedures.
These data consist of diagnosis codes, procedure types, location, free text notes (represented as a bag of words), and various other information recorded in the EMR system. We measured the duration of a procedure as the time spent in the room by the patient. This is an important prediction task since one potential application of machine learning in hospitals is to reduce costs by better anticipating procedure duration (and so improve surgery scheduling). When hospital scheduling systems depend on machine learning models it is important to monitor the ongoing performance of the model. In Section~\ref{sec:model_monitoring} we demonstrate how local explanations can significantly improve the process of monitoring this type of model deployment.}

\section{Model accuracy performance experiments}
\label{sec:methods_model_training_details}

Modern gradient boosted decision trees often provide state-of-the-art performance on tabular style datasets where features are individually meaningful, as consistently demonstrated by open data science competitions \cite{friedman2001greedy,chen2016xgboost}. 
All three medical datasets we examine here represent tabular-style data, and gradient boosted trees achieve the highest accuracy across all three datasets (Figure~\ref{fig:performance_and_synth}A).  We used 100 random train/test splits followed by retraining to assess the statistical significance of the separation between methods. In the mortality dataset, gradient boosted trees outperformed the linear lasso with a P-value $<0.01$ and the neural network with a P-value of $0.03$. In the chronic kidney disease dataset, gradient boosted trees trees outperformed the linear lasso with a P-value $<0.01$ and the neural network with a P-value of $0.08$. In the hospital procedure duration dataset, gradient boosted trees trees outperformed both the linear lasso and neural network with a P-value $<0.01$. The details of how we trained each method and obtained the results in Figure~\ref{fig:performance_and_synth}A are presented below.

\subsection{Mortality dataset}

For NHANES~I mortality prediction in Figure~\ref{fig:performance_and_synth}A, we used a cox proportional hazards loss, and a C-statistic \cite{heagerty2000time} to measure performance. For each of the three algorithms we split the $14,407$ samples according to a 64/16/20 split for train/validation/test. The features for the linear and the neural network models were mean imputed and standardized based on statistics computed on the training set. For the gradient boosted tree models, we passed the original data unnormalized and with missing values.

\vspace{1em}
\noindent {\bf Gradient Boosted Trees} - Gradient boosted trees were implemented using an XGBoost \cite{chen2016xgboost} model with the cox proportional hazards objective (we implemented this objective and merged it into XGBoost to support the experiments in this paper). Hyper-parameters were chosen using coordinate descent on the validation set loss. This resulted in a learning rate of $0.001$; $6,765$ trees of max depth $4$ chosen using early stopping; $\ell_2$ regularization of $5.5$; no $\ell_1$ regularization; no column sampling during fitting; and bagging sub-sampling of 50\%.

\vspace{1em}
\noindent {\bf Linear Model} - 
The linear model for the mortality dataset was implemented using the lifelines Python package \cite{lifelines2016}. The $\ell_2$ regularization weight was chosen using the validation set and set to $215.44$.

\vspace{1em}
\noindent {\bf Neural Network} -
The neural network model was implemented using the DeepSurv Python package \cite{katzman2018deepsurv}. Running DeepSurv to convergence for mortality data took hours on a modern GPU server, so hyper-parameter tuning was done by manual coordinate decent. This resulted in $\ell_2$ regularization of $1.0$; the use of batch normalization; a single hidden layer with 20 nodes; a dropout rate of $0.5$; a learning rate of $0.001$ with no decay; and momentum of $0.9$.

\subsection{Chronic kidney disease dataset}

For CRIC kidney disease prediction, we used logistic loss for binary classification, and measured performance using the area under the precision-recall curve in Figure~\ref{fig:performance_and_synth}A. For each of the three algorithms we split the $10,745$ samples according to a 64/16/20 split for train/validation/test. The features for the linear and the neural network models were mean imputed and standardized based on statistics computed on the training set. For the gradient boosted tree models we passed the original data unnormalized and with missing values. 

\vspace{1em}
\noindent {\bf Gradient Boosted Trees} - 
Gradient boosted trees were implemented using an XGBoost \cite{chen2016xgboost} model with the binary logistic objective function. Hyper-parameters were chosen using coordinate descent on the validation set loss. This resulted in a learning rate of $0.003$; $2,300$ trees of max depth $5$ chosen using early stopping; no $\ell_2$ regularization; no $\ell_1$ regularization; a column sampling rate during fitting of 15\%; and bagging sub-sampling of 30\%.

\vspace{1em}
\noindent {\bf Linear Model} - The linear model for the kidney dataset was implemented using scikit-learn \cite{pedregosa2011scikit}. Both $\ell_1$ and $\ell_2$ regularization were tested and $\ell_1$ was selected based on validation set performance, with an optimal penalty of $0.1798$.

\vspace{1em}
\noindent {\bf Neural Network} -
The neural network model was implemented using Keras \cite{chollet2015keras}. We chose to explore various feed-forward network architectures as well as 1D convolution based kernels (that learn a shared non-linear transform across many features). The best performance on the validation data came from a 1D convolution based architecture. After coordinate descent hyper-parameter tuning, we chose 15 1D convolution kernels which then go into a layer of 10 hidden units with rectified linear unit activation functions. Dropout of $0.4$ was used during training between the convolution kernels and the hidden layer, and between the hidden layer and the output. Because stochastic gradient descent can have varying performance from run to run, we chose the best model based on validation loss from 10 different optimization runs.

\subsection{Hospital procedure duration dataset}

For hospital procedure duration prediction, we used a squared error loss and measured performance by the coefficient of determination ($R^2$) in Figure~\ref{fig:performance_and_synth}A. Two different hospital procedure dataset splits were used in this paper. The first dataset split, used for model comparison in Section~\ref{sec:why_trees_accurate} (Figure~\ref{fig:performance_and_synth}A), consisted of 147,000 procedures divided according to a random 80/10/5/5 split for train/validation/test1/test2. The test2 set was only used to get final performance numbers and not during method development. The features for the linear and the neural network models were mean imputed and standardized based on statistics computed on the training set. For the gradient boosted tree models, we passed the original data unnormalized and with missing values. All hyper-parameter tuning was done with the validation dataset.

The second dataset split, used for the monitoring plots in Section~\ref{sec:model_monitoring}, divided procedures from the two hospitals based on time: the first 1 year of data was used for training and the last 3 years were used for test, in a manner intended to simulate actual model deployment. Models for this data were not hyper-parameter tuned, as the goal was not to achieve perfect performance but to demonstrate the value of monitoring methods; instead, they simply used the best hyper-parameter values found on the first (random-splits) dataset.

\vspace{1em}
\noindent {\bf Gradient Boosted Trees} - Gradient boosted trees were implemented using an XGBoost \cite{chen2016xgboost} model with a squared loss objective function. Hyper-parameters were chosen using grid search over the validation set loss. This resulted in a learning rate of $0.1$; $530$ trees of max depth $10$ chosen using early stopping; a column sampling rate during fitting of 50\%; and no bagging sub-sampling.

\vspace{1em}
\noindent {\bf Linear Model} - The linear model for the procedure duration dataset was implemented using the \texttt{LassoCV} class in scikit-learn \cite{pedregosa2011scikit}. $\ell_1$ regularization was tuned based on validation set performance to $0.120$.

\vspace{1em}
\noindent {\bf Neural Network} - The neural network model was implemented using Keras \cite{chollet2015keras}. We limited our architecture search to feed-forward networks with up to three hidden layers, and sizes up to 4,096 nodes per layer. The best performance on validation data came from a single-hidden-layer architecture with a 1,024-node layer followed by a dropout probability of 0.75 before the output.

\section{Interpretability comparison of linear models and tree-based models in the presence of non-linearities}
\label{sec:methods_sim_bias_experiement_details}

Even if linear models appear easier to interpret, their high bias may force the relationships they learn to be farther from the truth than a low-bias tree-based model (Figures~\ref{fig:performance_and_synth}B-D). We illustrated this using \emph{simulated outcomes} from the NHANES~I mortality dataset with varying amounts of non-linearity between certain features and the outcome. The input feature data was used as is. Even if the simulated outcome only depends on age and body mass index (BMI), the linear model learns features in the mortality dataset that are non-linearly dependent with age and BMI to try to approximate non-linear relationships with the outcome. This increases the test accuracy of the linear model slightly, as seen by the small increase in performance from the two-feature linear model (dashed red) to the all-feature linear model (solid red) (Figure~\ref{fig:performance_and_synth}C). However, this comes at the cost of placing much of its weight on features that are not actually used to generate the outcome. In contrast, the gradient boosted tree model correctly uses only age and BMI across all levels of function non-linearity (Figure~\ref{fig:performance_and_synth}D). 
Below we describe the experimental setup behind the results in  Figure~\ref{fig:performance_and_synth}B-D (Section~\ref{sec:why_trees_interpretable}). 


\subsection{Selection of features} 
We generated a synthetic label with a known relationship to two input features from the mortality dataset, but to the extent possible, we intended this synthetic relationship to be realistic; while still retaining the ability to control the amount of non-linearity in the relationship. Starting with the classifiers trained on the full mortality (NHANES~I) dataset, SHAP dependence plots (Section~\ref{sec:feature_dependence}) were used to find one feature that had a strong linear relationship with mortality (age) and one that had a ``U-shaped'' relationship with mortality (BMI). These two features were selected and used as the ``true'' label-generating features in our synthetic model.

\subsection{Label generation}

The synthetic label-generating function was constructed using a logistic function applied to the sum of a linear function of age and a quadratic function of BMI. The functional form allows us to smoothly vary the amount of nonlinearity in the label-generating model. The quadratic function of BMI was parameterized as $(1-p)$ times a linear function, plus $p$ times a quadratic function, where both functions had the same minimum and maximum values (the x-location of the minimum for the quadratic was set to the mean BMI). The contribution in logits was set to range from a minimum of -2.0 to a maximum of 2.0 for age, and -1.0 to 1.0 for BMI, so that even when nonlinear contribution were the strongest ($p=1$) the linear contribution of the main risk factor was still more important (as true in the original data). The output for each data point was then a probability to be predicted by our models (we did not add noise by binarizing these probabilities to 0-1, so as to avoid the need to average our results over many random replicates). Thus the final label-generating model was
$$ y= \sigma\left((1.265(\text{age})+0.0233)+(1-p)(0.388(\text{BMI})-0.325) + (p)(1.714(\text{BMI})^2-1)\right)$$
where $\sigma=\frac{1}{1+e^{-t}}$ is the logistic function. A total of 11 datasets were generated, with $p\in[0.0,0.1,0.2...1.0]$, by applying this labeling function to the true data covariates, so that the matrix of predictors $X$  was real data but the labels $y$ were generated synthetically by a known mechanism.

\subsection{Model training}
We trained both gradient boosted trees and linear logistic regression to predict the synthetic labels. For each of the 11 datasets with varying degrees $p$ of nonlinearity, models were tuned on a validation set and evaluated on test data with a 64/16/20 train/validation/test split. The only logistic regression hyper-parameter was the L$1$ regularization penalty which was optimized over the range $\lambda \in [10^{-4},10^{-3}...10^3,10^{4}]$. The tuned gradient boosting model hyper-parameters were optimized over the tree depths of $[1,2,4,8,10]$ and bagging sub-sampling over the rates $[0.2,0.5,0.8,1.0]$. The learning rate was fixed at $0.01$; the minimum loss reduction for splitting to $1.0$; the minimum child weight for splitting was $10$; and trees we trained for a maximum of $1,000$ rounds with early stopping based on validation set loss.

\subsection{Feature importance}
Per-sample importance values were calculated for each feature in each
of the 11 datasets using SHAP values for both the logistic (using Linear SHAP
values assuming feature independence \cite{lundberg2017unified}) and gradient boosted tree (using TreeExplainer's Tree SHAP algorithm) models. At
each value of $p$, the total weight of irrelevant samples was calculated
by taking the absolute value of all SHAP values for all features other
than age and BMI, and summing these values across all samples and
features.


\section{Previous Global Explanation Methods for Trees}
\label{sec:methods_previous_global_tree}

While local explanation methods for trees have not been extensively studied, interpreting tree-based models by assigning a global importance value to each input feature is a well studied problem and many methods have been proposed \cite{friedman2001elements,friedman2001greedy,sandri2008bias,strobl2007bias,strobl2008conditional,ishwaran2007variable,auret2011empirical,louppe2014understanding,kazemitabar2017variable}. The most basic global approach is to simply count the number of times a feature was used for splitting, but this fails to account for the differing impacts of different splits. A better approach is to attribute the reduction in loss (aka. Gain) provided by each split in each decision tree to the feature that was split on \cite{breiman1984classification, friedman2001elements}. This ``Gain'' measure of feature importance was shown to correctly recover the mutual information between the input features and the outcome label in the limit of an infinite ensemble of totally random fully developed trees \cite{louppe2014understanding}. However, it becomes biased for finite ensembles of greedily built trees, and so approaches have been designed to account for this bias when using Gain for feature selection \cite{irrthum2010inferring, chebrolu2005feature}. Another popular method for determining feature importance is to permute the data column corresponding to a feature and then observe the change in the model's loss \cite{breiman2001random}. If the model's loss increases significantly when a feature is permuted then it indicates the model was heavily depending on that feature. This permutation approach can be further extended to account for statistical dependencies between features by permuting only within specified groups of samples \cite{strobl2008conditional}. All of these approaches are designed to estimate the global importance of a feature over an entire dataset, so they are not directly applicable to local explanations that are specific to each prediction. If we try to use global methods in place of true local explanations we get significantly worse performance on many benchmark metrics (Figure~\ref{fig:cric_tile}).

While not explicitly designed to be a global feature attribution method, TreeExplainer can be used as a global method by averaging many local explanations. If we do this over all samples in a dataset, then we get a global measure of feature importance that does not suffer from the inconsistencies of the classic Gain method (Supplementary Figure~\ref{fig:and_trees_simple}), and unlike the permutation method, does not miss high-order interaction effects. Global feature attribution based on TreeExplainer has a higher power to detect important features in the presence of interactions than current state-of-the-art methods (Supplementary Figure~\ref{fig:feature_selection}). This has important implications for the popular task of feature selection based on tree-ensembles. 

\section{Previous local explanation methods for trees}
\label{sec:methods_previous_local_tree}

As described in Section~\ref{sec:inconsistent_local}, we are aware of only two previous tree-specific local explanation methods: reporting the decision path; and an unpublished heuristic difference in expectations method (proposed by Saabas) \cite{treeinterpreter}. Since reporting the decision path is not useful for large tree ensembles we instead focus on the heuristic Saabas method.
The Saabas difference in expectations approach explains a prediction by following the decision path and attributing changes in the expected output of the model to each feature along the path. This is efficient since the expected value of every node in the tree can be estimated by averaging the model output over all the training samples that pass through that node.
Let $f$ be a decision tree model, $x$ the instance we are going to explain, $f(x)$ the output of the model for the current instance, and $f_x(S) \approx E[f(x) \mid x_S]$ the estimated expectation of the model output conditioned on the set $S$ of feature values (\ref{sec:methods_tree_explainer}, Algorithm~\ref{alg:exp_value}), then we can define the {\it Saabas value} for the $i$'th feature as

\begin{equation}
\phi^{s}_i(f,x) = \sum_{j \in D^i_x} f_x(A_j \cup j) - f_x(A_j), 
\label{eq:saabas}
\end{equation}

\noindent where $D^i_x$ is the set of nodes on the decision path from $x$ that split on feature $i$, and $A_j$ is the set of all features split on by ancestors of $j$. Equation~\ref{eq:saabas} results in a set of feature attribution values that sum up to the difference between the expected output of the model and the output for the current prediction being explained (Supplementary Figure~\ref{fig:number_line}). When explaining an ensemble model made up of a sum of many decision trees, the Saabas values for the ensemble model are defined as the sum of the Saabas values for each tree.

\section{Model agnostic local explanation methods}
\label{sec:methods_previous_agnostic}

Many different local explanation methods have been proposed in the literature \cite{baehrens2010explain,vstrumbelj2014explaining,ribeiro2016should,datta2016algorithmic,lundberg2017unified,ribeiro2018anchors}. The most well known is simply taking the gradient of the model's output with respect to its inputs at the current sample. This is common in the deep learning literature, as is the related approach of multiplying the gradient times the value of the input features. Relying entirely on the gradient of the model at a single point though can often be misleading \cite{shrikumar2016not}. To provide a better allocation of credit for deep learning models, various methods have been proposed that either modify the standard gradient back propagation rule to instead propagate attributions \cite{springenberg2014striving,zeiler2014visualizing,bach2015pixel,shrikumar2016not,kindermans2017learning,ancona2018towards}, or integrate the gradient along a path based on fair allocation rules from economics \cite{sundararajan2017axiomatic}. 

In contrast to deep learning methods, model-agnostic methods make no assumptions about the internal structure of the model. These methods rely only on observing the relationship between changes in the model inputs and model outputs. This can be done by training a global mimic model to approximate the original model, then locally explaining the mimic model by either taking its gradient \cite{baehrens2010explain}, or fitting a local linear model as in MAPLE \cite{plumb2018model}. Alternatively, the mimic model can be fit to the original model locally for each prediction. For a local linear mimic model the coefficients can be used as an explanation, as in the popular LIME method \cite{ribeiro2016should}. For a local decision rule mimic model the rules can be used as the explanation as in Anchors \cite{ribeiro2018anchors}. Another class of approaches do not explicitly fit a mimic model, but instead perturb sets of features to measure their importance, then use methods from game theory to fairly allocate the importance of these sets among the input features, this class includes IME \cite{vstrumbelj2014explaining} and QII \cite{datta2016algorithmic}. 

Perhaps surprisingly, despite the seeming variety of different local explanation methods, two back propagation-style deep learning methods \cite{bach2015pixel,shrikumar2016not}, local linear mimic models \cite{ribeiro2016should}, and several game theoretic methods \cite{lipovetsky2001analysis,vstrumbelj2014explaining,datta2016algorithmic} were recently unified into a single class of {\it additive feature attribution methods} in our prior study \cite{lundberg2017unified}. This class is of particular interest since results from cooperative game theory imply there is a unique optimal explanation approach in the class (the Shapley values) that satisfies several desirable properties \cite{shapley1953value, roth1988shapley}. Unfortunately computing the Shapley values is NP-hard in general \cite{matsui2001np}, with a runtime cost exponential in the number of input features. When faced with an NP-hard optimization problem it is typical to build approximate methods, which exactly IME \cite{vstrumbelj2014explaining} or QII \cite{datta2016algorithmic} do. However, here we take an alternative approach and restrict our focus specifically to tree-based machine learning models. By doing this we are able to show constructively that solving for the exact Shapley values in trees is not NP-hard, and can be solved by TreeExplainer in low-order polynomial time (\ref{sec:methods_tree_explainer}).

\section{Convergence experiments for model agnostic Shapley value approximations}
\label{sec:methods_agnostic_convergence}


In Supplementary Figures~\ref{fig:tree_shap_performance}C-D we generated random datasets of increasing size and then explained (over)fit XGBoost models with 1,000 trees. The runtime and standard deviation of the local explanations are reported for Kernel SHAP \cite{lundberg2017unified}, IME \cite{vstrumbelj2014explaining}, and TreeExplainer; except that for Kernel SHAP and IME the reported times are only a lower bound. 
Both the IME and Kernel SHAP model-agnostic methods must evaluate the original model a specific number of times for each explanation, so the time spent evaluating the original model represents a lower bound on the runtime of the methods (note that the QII method \cite{datta2016algorithmic} is not included in our comparisons since for local feature attribution it is identical to IME). In Supplementary Figure~\ref{fig:tree_shap_performance}C we report this lower bound, but it is important to note that in practice there is also additional overhead associated with the actual execution of the methods that depends on how efficiently they are implemented. We also only used a single background reference sample for the model-agnostic approaches. This allows them to converge faster at the expense of using less accurate estimates of conditional expectations. Increasing the number of background samples would only further reduce the computational performance of these methods.
Each method is run ten times, then the standard deviation for each feature is divided by the mean of each feature to get a normalized standard deviation (Supplementary Figure~\ref{fig:tree_shap_performance}D).  In order to maintain a constant level of normalized standard deviation, Kernel SHAP and IME are allowed a linearly increasing number of samples as the number of features in a dataset, $M$, grows.  
In Supplementary Figure~\ref{fig:tree_shap_performance}C, TreeExplainer is so much faster than the model-agnostic methods that it appears to remain unchanged as we scale $M$, though in reality there is a small growth in its runtime. In Supplementary Figure~\ref{fig:tree_shap_performance}D there is truly no variability since the TreeExplainer method is exact and not stochastic.

In Supplementary Figures~\ref{fig:tree_shap_performance}E-F, the different explainers are compared in terms of estimation error (absolute deviation from the ground truth) on the chronic kidney disease dataset. We chose this dataset model for our comparisons because it is much smaller than the hospital procedure duration dataset, and has a more common loss function than the mortality dataset model (logistic loss vs. a cox proportional hazards loss). Ground truth is obtained via TreeExplainer's exact Independent Tree SHAP algorithm with the reference set fixed to be the mean of the data. The plots are obtained by increasing the number of samples allowed to each explainer and reporting the max and mean estimation error. For IME, we tune the minimum samples per feature, a hyper-parameter that is utilized to estimate which features' attributions have larger variability. After the minimum samples per feature has been achieved, the rest of the samples are allocated so as to optimally reduce the variance of the sum of the estimated values (by giving more samples to features with high sampling variance). As expected this is beneficial to the max evaluation error, but can potentially lead to bias (as for IME (min 10) in Supplementary Figure~\ref{fig:tree_shap_performance}E). While often helpful, $\ell_1$ regularization was not useful to improve Kernel SHAP for this dataset, so we report the results from unregularized regression.

To measure the cost of computing low variance estimates of explanations for the chronic kidney disease dataset we defined ``low variance'' as 1\% of the tenth largest feature impact (out of 333 features), then measured how many samples it took on average to reach a standard deviation below that level (where standard deviation is measured across repeated runs of the explanation method). This was done for both the maximum standard deviation across all features (Supplementary Figure~\ref{fig:tree_shap_performance}E), and the mean standard deviation (Supplementary Figure~\ref{fig:tree_shap_performance}F). Calculating low variance estimates for the experiments presented in this paper on the chronic kidney disease dataset would have taken almost 2 CPU days for basic explanations, and over 3 CPU years for interaction values (Section~\ref{sec:shap_interaction_values}).


\section{Unifying previous heuristics with Shapley values}
\label{sec:methods_shap_unity}

Here we review the uniqueness guarantees of Shapley values from game theory as they apply to local explanations of predictions from machine learning models \cite{shapley1953value}. We then detail how to reinterpret previous local explanation heuristics for trees, and so connect them with Shapley values.

As applied here, Shapley values are computed
by introducing each feature, one at at time, into a conditional expectation function of the model's output, $f_x(S) \approx E[f(x) \mid x_S]$ (\ref{sec:methods_tree_explainer}, Algorithm~\ref{alg:exp_value}), and attributing the change produced at each step to the feature that was introduced; then averaging this process over all possible feature orderings (Supplementary Figure~\ref{fig:number_line}). Shapley values represent the only possible method in the broad class of {\it additive feature attribution methods} \cite{lundberg2017unified} that will simultaneously satisfy three important properties: {\it local accuracy}, {\it consistency}, and {\it missingness}. 




Local accuracy (known as {\it additivity} in game theory) states that when approximating the original model $f$ for a specific input $x$, the explanation's attribution values should sum up to the output $f(x)$:

\begin{property}[Local accuracy / Additivity]
\begin{equation}
f(x) = \phi_0(f) + \sum_{i = 1}^M \phi_i(f, x)
\end{equation}
The sum of feature attributions $\phi_i(f, x)$ matches the original model output $f(x)$, where $\phi_0(f) = E[f(z)] = f_x(\emptyset)$.
\label{prop:local_accuracy}
\end{property}

Consistency (known as {\it monotonicity} in game theory) states that if a model changes so that some feature's contribution increases or stays the same regardless of the other inputs, that input's attribution should not decrease:

\begin{property}[Consistency / Monotonicity]
For any two models $f$ and $f'$, if
\begin{equation}
f'_x(S) - f'_x(S \setminus i) \ge f_x(S) - f_x(S \setminus i)
\end{equation}
for all subsets of features $S \in \mathcal{F}$, then $\phi_i(f',x) \ge \phi_i(f,x)$.
\label{prop:consistency}
\end{property}

Missingness (similar to {\it null effects} in game theory) requires features with no effect on the set function $f_x$ to have no assigned impact. All local previous methods we are aware of satisfy missingness.

\begin{property}[Missingness]
If
\begin{equation}
f_x(S \cup i) = f_x(S)
\end{equation}
for all subsets of features $S \in \mathcal{F}$, then $\phi_i(f,x) = 0$.
\label{prop:missingness}
\end{property}

The only way to simultaneously satisfy these properties is to use the classic Shapley values:

\begin{theorem}
\label{thrm:shapley}
Only one possible feature attribution method based on $f_x$ satisfies Properties 1, 2 and 3:
\begin{equation}
\phi_i(f,x) = \sum_{R \in \mathcal{R}} \frac{1}{M!} \left [ f_x(P^R_i \cup i) - f_x(P^R_i) \right ]
\label{eq:shapley}
\end{equation}
where $\mathcal{R}$ is the set of all feature orderings, $P^R_i$ is the set of all features that come before feature $i$ in ordering $R$, and $M$ is the number of input features for the model.
\end{theorem}


The equivalent of Theorem \ref{thrm:shapley} has been previously presented in \cite{lundberg2017unified} and follows from cooperative game theory results \cite{young1985monotonic}, where the values $\phi_i$ are known as the Shapley values \cite{shapley1953value}. Shapley values are defined independent of the set function used to measure the importance of a set of features. Since here we are using $f_x$, a conditional expectation function of the model's output, we are computing the more specific {\it SHapley Additive exPlanation (SHAP) values} \cite{lundberg2017unified}.


Surprisingly, there is a close parallel between the Saabas values (Equation~\ref{eq:saabas}) and the SHAP values (Equation~\ref{eq:shapley}). While SHAP values average the importance of introducing a feature over all possible feature orderings, Saabas values only consider the single ordering defined by a tree's decision path (Supplementary Figure~\ref{fig:number_line}). This mean that Saabas values satisfy the local accuracy property since they always sum up to the difference between the expected value of the model $E[f(x)]$ and the current output $f(x)$. But since they do not average over all orderings they do not match the SHAP values, and so must violate consistency. Examples of such inconsistencies can be found for even very small trees, where changing the model to depend more on one feature can actually cause the Saabas attribution values for that feature to decrease (Supplementary Figure~\ref{fig:and_trees_simple}).
The difference this makes can be seen by examining trees representing multi-way AND functions. If we let all the input features be independent and identically distributed then no feature in an AND function should have any more credit than another, yet for Saabas values, splits near the root are given much less credit than splits near the leaves (Supplementary Figures~\ref{fig:tree_shap_performance}A). This means that while mathematically all the features play an equal role in the model, the features near the root of the tree get little or no credit. This is particularly troubling since features near the root are likely the most important as they were chosen first by greedy splitting.


There is a close parallel between Saabas values and the classic ``Gain'' method for global feature attribution (sometimes known as Gini importance) \cite{friedman2001elements}. Just as Gain attributes the change in loss after each split to the feature that was split on, so Saabas attributes the change in conditional expectation after each split to the feature that was split on. Both methods only consider a single order for introducing features into the model, the order defined by paths in the tree. Choosing to use only a single ordering leads to inconsistent allocation of credit (Supplementary Figure~\ref{fig:tree_shap_performance}A~and~\ref{fig:and_trees_simple}). Shapley values guarantee a consistent allocation of credit by averaging the change after each split over all possible orderings. It has been previously shown through a connection with mutual information (which is consistent) that the Gain method becomes consistent in the limit of an infinite ensemble of totally random fully developed trees \cite{louppe2014understanding}. This suggests that the Saabas method may also become consistent in the limit of an infinite ensemble of totally random fully developed trees. This is indeed the case, and we show in Theorem~\ref{thrm:saabas_limit} that for binary features the Saabas values converge to the SHAP values in the limit of an infinite ensemble of totally random fully developed trees. 

\begin{theorem}
\label{thrm:saabas_limit}
In the limit of an infinite ensemble of totally random fully developed trees on binary features the Saabas values equal the SHAP values \cite{lundberg2017unified} (which are Shapley values of a conditional expectation function of the model's output).
\end{theorem}

\begin{proof}
Assume all features are binary, then the decision path of a single input instance $x$ for a single tree will be a random ordering of all the input features. The Saabas values for that tree will be equivalent to a single permutation from the formula for Shapley values (Equation~\ref{eq:shapley}). Since there is an infinite ensemble of trees, all possible feature orderings will be represented in equal proportions. Given a finite output domain, there will furthermore be all possible feature orderings represented for each possible leaf value.  Taking the average of the Saabas values over the ensemble of trees then becomes the same as the averaging function in the definition of the Shapley values (Equation~\ref{eq:shapley}).
\end{proof}

\section{TreeExplainer algorithms}
\label{sec:methods_tree_explainer}

Here, we describe the algorithms behind TreeExplainer in three stages. First, we describe an easy to understand (but slow) version of the main Tree SHAP algorithm, then we present the complex polynomial time version of Tree SHAP, and finally we describe the Independent Tree SHAP algorithm used for explaining non-linear model output transformations (such as the model's loss). While solving for the Shapley values is in general NP-hard \cite{matsui2001np}, these algorithms show that by restricting our attention to trees, we can find exact algorithms that run in low-order polynomial time.

\subsection{Tree SHAP}

Tree SHAP, the main algorithm behind TreeExplainer, can exactly compute the Shapley values, and so guarantee consistent explanations (Property~\ref{prop:consistency}). Tree SHAP exactly computes Equation~\ref{eq:shapley} in low order polynomial time, where the conditional expectation function, $f_x$, is defined using tree traversal (Algorithm~\ref{alg:exp_value}). Letting $T$ be the number of trees, $D$ the maximum depth of any tree, and $L$ the number of leaves, Tree SHAP has worst case complexity of $O(TLD^2)$. This represents an exponential complexity improvement over previous exact Shapley methods, which would have a complexity of $O(TLM2^M)$, where $M$ is the number of input features. By directly computing the Shapley values we are able to guarantee that the explanations will always be consistent and locally accurate.

\subsubsection{Estimating SHAP values directly in $O(TLM2^M)$ time} 
\label{sec:shap_direct}

If we ignore computational complexity then we can compute the SHAP values for a tree by estimating $E[f(x) \mid x_S]$ and then using Equation \ref{eq:shapley} (\ref{sec:methods_shap_unity}). For a tree model, $E[f(x) \mid x_S]$ can be estimated recursively using Algorithm \ref{alg:exp_value}, where $tree$ contains the information of the tree. $v$ is a vector of node values; for internal nodes, we assign the value $internal$. 
The vectors $a$ and $b$ represent the left and right node indexes for each internal node. The vector $t$ contains the thresholds for each internal node, and $d$ is a vector of indexes of the features used for splitting in internal nodes. The vector $r$ represents the cover of each node (i.e., how many data samples fall in that sub-tree). 

Algorithm \ref{alg:exp_value} finds $E[f(x) \mid x_S]$ by recursively following the decision path for $x$ if the split feature is in $S$, and taking the weighted average of both branches if the split feature is not in $S$. The computational complexity of Algorithm~\ref{alg:exp_value} is proportional to the number of leaves in the tree, which when used on all $T$ trees in an ensemble and plugged into Equation~\ref{eq:shapley} leads to a complexity of $O(TLM2^M)$ for computing the SHAP values of all $M$ features.

\begin{algorithm}
\caption{Estimating $E[f(x) \mid x_S]$ \label{alg:exp_value}}
\begin{algorithmic}[1]
\Procedure{EXPVALUE}{$x$, $S$, $tree = \{v, a, b, t, r, d\}$}
\Procedure{G}{$j$} \Comment{Define the $G$ procedure which we will call on line 10}
  \If{$v_j \ne internal$} \Comment{Check if node $j$ is a leaf}
    \State \Return{$v_j$} \Comment{Return the leaf's value}
  \Else
    \If{$d_j \in S$} \Comment{Check if we are conditioning on this feature}
      \State \Return \Call{G}{$a_j$} {\bf if} $x_{d_j} \le t_j$ {\bf else} \Call{G}{$b_j$} \Comment{Use the child on the decision path}
    \Else
    \State \Return $[$\Call{G}{$a_j$}$ \cdot r_{a_j}$ + \Call{G}{$b_j$}$ \cdot  r_{b_j}] /r_j$ \Comment{Weight children by their coverage}
    \EndIf
  \EndIf
\EndProcedure
\State \Return \Call{G}{$1$} \Comment{Start at the root node}
\EndProcedure
\end{algorithmic}
\end{algorithm}

\subsubsection{Estimating SHAP values in $O(TLD^2)$ time}
\label{sec:shap_fast}

Now we calculate the same values as above, but in polynomial time instead of exponential time. Specifically, we propose an algorithm that runs in $O(TLD^2)$ time and $O(D^2 + M)$ memory, where for balanced trees the depth becomes $D = \log L$. Recall $T$ is the number of trees, $L$ is the maximum number of leaves in any tree, and $M$ is the number of features.


The intuition of the polynomial time algorithm is to recursively keep track of what proportion of all possible subsets flow down into each of the leaves of the tree.
This is similar to running Algorithm \ref{alg:exp_value} simultaneously for all $2^M$ subsets $S$ in Equation~\ref{eq:shapley}. Note that a single subset $S$ can land in multiple leaves. It may seem reasonable to simply keep track of how many subsets (weighted by the cover splitting of Algorithm~\ref{alg:exp_value} on line 9) pass down each branch of the tree. However, this combines subsets of different sizes and so prevents the proper weighting of these subsets, since the weights in Equation~\ref{eq:shapley} depend on $|S|$. To address this we keep track of each possible subset size during the recursion, not just single a count of all subsets. The {\it EXTEND} method in Algorithm~\ref{alg:tree_shap} grows all these subset sizes according to a given fraction of ones and zeros, while the {\it UNWIND} method reverses this process and is commutative with {\it EXTEND}. The {\it EXTEND} method is used as we descend the tree. The {\it UNWIND} method is used to undo previous extensions when we split on the same feature twice, and to undo each extension of the path inside a leaf to compute weights for each feature in the path. Note that {\it EXTEND} keeps track of not just the proportion of subsets during the recursion, but also the weight applied to those subsets by Equation~\ref{eq:shapley}. Since the weight applied to a subset in Equation~\ref{eq:shapley} is different when it includes the feature $i$, we need to {\it UNWIND} each feature separately once we land in a leaf, so as to compute the correct weight of that leaf for the SHAP values of each feature. The ability to {\it UNWIND} only in the leaves depends on the commutative nature of {\it UNWIND} and {\it EXTEND}.

\begin{algorithm}
\caption{Tree SHAP \label{alg:tree_shap}}
\begin{algorithmic}[1]
\Procedure{TREESHAP}{$x$, $tree = \{v, a, b, t, r, d\}$}
\State $\phi = \textrm{array of $len(x)$ zeros}$
\Procedure{RECURSE}{$j$, $m$, $p_z$, $p_o$, $p_i$}
  \State $m =~$\Call{EXTEND}{$m$, $p_z$, $p_o$, $p_i$} \Comment{Extend subset path with a fraction of zeros and ones}
  \If{$v_j \ne internal$} \Comment{Check if we are at a leaf node}
    \For{$i \gets 2 \textrm{ to } len(m)$} \Comment{Calculate the contributions from every feature in our path}
      \State $w = sum(\Call{UNWIND}{m, i}.w)$ \Comment{Undo the weight extension for this feature}
      \State $\phi_{m_i} = \phi_{m_i} + w(m_i.o - m_i.z) v_j$ \Comment{Contribution from subsets matching this leaf}
    \EndFor
  \Else
    \State $h,c = (a_j,b_j)$ {\bf if} $x_{d_j} \le t_j$ {\bf else} $(b_j,a_j)$ \Comment{Determine hot and cold children}
    \State $i_z = i_o = 1$
    \State $k = \Call{FINDFIRST}{m.d, d_j}$
    \If{$k \ne \textrm{nothing}$} \Comment{Undo previous extension if we have already seen this feature}
      \State $i_z,i_o = (m_k.z,m_k.o)$
      \State $m = \Call{UNWIND}{m, k}$
    \EndIf
    \State \Call{RECURSE}{$h$, $m$, $i_z r_h/r_j$, $i_o$, $d_j$} \Comment{Send both zero and one weights to the hot child}
    \State \Call{RECURSE}{$c$, $m$, $i_z r_c/r_j$, $0$, $d_j$} \Comment{Send just zero weights to the cold child}
  \EndIf
\EndProcedure
\Procedure{EXTEND}{$m$, $p_z$, $p_o$, $p_i$}
\State $l, m = len(m), copy(m)$
\State $m_{l+1}.(d,z,o,w) = (p_i, p_z, p_o, (1 $ {\bf if} $l = 0$ {\bf else} $0))$ \Comment{Init subsets of size $l$}
\For{$i \gets l \textrm{ to } 1$} \Comment{Grow subsets using $p_z$ and $p_o$ }
  \State $m_{i+1}.w = m_{i+1}.w + p_o \cdot m_i.w \cdot (i/l)$ \Comment{Subsets that grow by one}
  \State $m_i.w =  p_z \cdot m_i.w \cdot (l-i)/l$ \Comment{Subsets that stay the same size}
\EndFor
\State \Return m \Comment{Return the new extended subset path}
\EndProcedure
\Procedure{UNWIND}{$m$, $i$} \Comment{The inverse of the $i$th call to EXTEND$(m, ...)$}
\State $l,n,m = len(m), m_l.w, copy(m_{1...l-1})$
\For{$j \gets l-1 \textrm{ to } 1$} \Comment{Shrink subsets using $m_i.z$ and $m_i.o$ }
  \If{$m_i.o \ne 0$}
    \State $t = m_j.w$
    \State $m_j.w = n \cdot l/(j \cdot m_i.o)$
    \State $n = t - m_j.w \cdot m_i.z \cdot (l-j)/l$
  \Else
    \State $m_j.w = (m_j.w \cdot l)/(m_i.z (l-j))$
  \EndIf
\EndFor
\For{$j \gets i \textrm{ to } l-1$}
  \State $m_j.(d,z,o) = m_{j+1}.(d,z,o)$
\EndFor
\State \Return m
\EndProcedure
\State \Call{RECURSE}{$1$, $[~]$, $1$, $1$, $0$} \Comment{Start at first node with all zero and one extensions}
\State \Return $\phi$
\EndProcedure
\end{algorithmic}
\end{algorithm}

In Algorithm~\ref{alg:tree_shap}, $m$ is the path of unique features we have split on so far, and contains four attributes: i) $d$, the feature index, ii) $z$, the fraction of ``zero'' paths (where this feature is not in the set $S$) that flow through this branch, iii) $o$, the fraction of ``one'' paths (where this feature is in the set $S$) that flow through this branch, and iv) $w$, which is used to hold the proportion of sets of a given cardinality that are present weighted by their Shapley weight (Equation~\ref{eq:shapley}). Note that the weighting captured by $w$ does not need to account for features not yet seen on the decision path so the effective size of $M$ in Equation~\ref{eq:shapley} is growing as we descend the tree. We use the dot notation to access member values, and for the whole vector $m.d$ represents a vector of all the feature indexes. The values $p_z$, $p_o$, and $p_i$ represent the fraction of zeros and ones that are going to extend the subsets, and the index of the feature used to make the last split. We use the same notation as in Algorithm~\ref{alg:exp_value} for the tree and input vector $x$. The child followed by the tree when given the input $x$ is called the ``hot'' child. Note that the correctness of Algorithm~\ref{alg:tree_shap} (as implemented in the open source code) has been validated by comparing its results to the brute force approach based on Algorithm~\ref{alg:exp_value} for thousands of random models and datasets where $M < 15$.

{\bf Complexity analysis:} Algorithm \ref{alg:tree_shap} reduces the computational complexity of exact SHAP value computation from exponential to low order polynomial for trees and sums of trees (since the SHAP values of a sum of two functions is the sum of the original functions' SHAP values). The loops on lines 6, 12, 21, 27, and 34 are all bounded by the length of the subset path $m$, which is bounded by $D$, the maximum depth of a tree. This means the complexity of {\it UNWIND} and {\it EXTEND} is bounded by $O(D)$. Each call to {\it RECURSE} incurs either $O(D)$ complexity for internal nodes, or $O(D^2)$ for leaf nodes, since {\it UNWIND} is nested inside a loop bounded by $D$. This leads to a complexity of $O(LD^2)$ for the whole tree because the work done at the leaves dominates the complexity of the internal nodes. For an entire ensemble of $T$ trees this bound becomes $O(TLD^2)$. If we assume the trees are balanced then $D = \log L$ and the bound becomes $O(TL\log^2 L)$. $\square$ 

\subsection{Independent Tree SHAP: Estimating SHAP values under independence in $O(TRL)$ time}
\label{sec:independent_tree_shap_method}

The Tree SHAP algorithm provides fast exact solutions for trees and sums of trees (because of the linearity of Shapley values \cite{shapley1953value}), but there are times when it is helpful to explain not the direct output of the trees, but also a non-linear transform of the tree's output. A compelling example of this is explaining a model's loss function, which is very useful for model monitoring and debugging (Section~\ref{sec:model_monitoring}). 

Unfortunately, there is no simple way to adjust the Shapley values of a function to exactly account for a non-linear transformation of the model output. Instead, we combine a previously proposed compositional approximation (Deep SHAP) \cite{lundberg2017unified} with ideas from Tree SHAP to create a fast method specific to trees, {\it Independent Tree SHAP}. The compositional approach requires iterating over each background sample from the dataset used to compute the expectation. This means that conditional expectations can no longer be computed using Algorithm~\ref{alg:exp_value}, since only the path corresponding to the current background reference sample will be defined. To solve this we enforce independence between input features then develop Independent Tree SHAP as a single-reference version of Tree SHAP (Algorithm~\ref{alg:ind_tree_shap}). 

Independent Tree SHAP enforces an independence assumption between the conditional set $S$ and the set of remaining features ($x_S \bot x_{\bar S}$). Utilizing this independence assumption, Shapley values with respect to $R$ individual background samples can be averaged together to get the attributions for the full distribution.  Accordingly, Algorithm \ref{alg:ind_tree_shap} is performed by traversing hybrid paths made up of a single foreground and background sample in a tree.  At each internal node, {\it RECURSE} traverses down the tree, maintaining local state to keep track of the set of upstream features and whether each split went down the path followed by the foreground or background sample.  Then, at each leaf, two contributions are computed -- one positive and one negative.  Each leaf's positive and negative contribution depends on the feature being explained.  However, calculating the Shapley values by iterating over all features at each leaf would result in a quadratic time algorithm.  Instead, {\it RECURSE} passes these contributions up to the parent node and determines whether to assign the positive or negative contribution to the feature that was split upon based on the directions the foreground and background samples traversed.  Then the internal node aggregates the two positive contributions into a single positive contribution and two negative contributions into a single negative contribution and passes it up to its parent node.

Note that both the positive and negative contribution at each leaf is a function of two variables: 1) $U$: the number of features that matched the foreground sample along the path and 2) $V$: the total number of unique features encountered along the path.  This means that for different leaves, a different total number of features $V$ will be considered. This allows the algorithm to consider only $O(L)$ terms, rather than an exponential number of terms.  Despite having different $U$'s at each leaf, Independent Tree SHAP exactly computes the traditional Shapley value formula (which considers a fixed total number of features $\geq V$ for any given path) because the terms in the summation group together nicely.

\begin{algorithm}
\caption{Independent Tree SHAP \label{alg:ind_tree_shap}}
\begin{algorithmic}[1]
\Procedure{INDTREESHAP}{$x$, $refset$, $tree = \{v, a, b, t, r, d\}$} 
\State $\phi = \textrm{array of $len(x)$ zeros}$
\Procedure{CALCWEIGHT}{$U$, $V$} \Comment{Shapley value weight for a set size and number of features}
    \State \Return $\frac{U!(V-U-1)!}{V!}$
\EndProcedure
    \Procedure{RECURSE}{$j$, $U$, $V$, $xlist$, $clist$} 
  \If{$v_j\neq internal$} \Comment{Calculate possible contributions at leaf}
    \State $pos = neg = 0$
    \If{$U==0$} \Return ($pos$, $neg$) \EndIf
    \If{$U \neq 0$} $pos = calcweight(V,U-1)*v_j$ \EndIf
    \If{$U \neq V$} $neg = -calcweight(V,U)*v_j$ \EndIf
    \State \Return ($pos$, $neg$)
  \EndIf
  \State $k = None$ \Comment{Represents the next node}
  \If{$(x_{d_j} > t_j)$ and $(c_{d_j} > t_j)$} $k = b_j$  \Comment{Both x and $c$ go right} \EndIf
  \If{$!(x_{d_j} > t_j)$ and $!(c_{d_j} > t_j)$} $k = a_j$ \Comment{Both x and $c$ go left}  \EndIf
  \If{$xlist_{d_j}>0$} \Comment{Feature was previously x}
    \If{$x_{d_j} > t_j$}  $k = b_j$
    \Else $\text{ }k = a_j$ \EndIf
  \EndIf
  \If{$clist_{d_j}>0$} \Comment{Feature was previously $c$}
    \If{$c_{d_j} > t_j$} $k = b_j$ 
    \Else $\text{ }k = a_j$ \EndIf
  \EndIf
  \If{$k \neq None$}  \Comment{Recurse down a single path if next node is set}
    \State \Return RECURSE($k$, $U$, $V$, $xlist$, $clist$)
  \EndIf
  \If{$(x_{d_j} > t_j)$ and $!(c_{d_j} > t_j)$} 
  \Comment{Recurse x right and $c$ left}
    \State $xlist_{d_j} = xlist_{d_j}+1$
    \State ($posx$,$negx$) = RECURSE($b_j$, $U+1$, $V+1$, $xlist$, $clist$)
    \State $xlist_{d_j} = xlist_{d_j}-1$
    \State $clist_{d_j} = clist_{d_j}+1$
    \State ($posc$,$negc$) = RECURSE($a_j$, $U$, $V+1$, $xlist$, $clist$)
    \State $clist_{d_j} = clist_{d_j}-1$
  \EndIf
  \If{!$(x_{d_j} > t_j)$ and $(c_{d_j} > t_j)$} \Comment{Recurse x left and $c$ right}
    \State $xlist_{d_j} = xlist_{d_j}+1$
    \State ($posx$,$negx$) = RECURSE($a_j$, $U+1$, $V+1$, $xlist$, $clist$)
    \State $xlist_{d_j} = xlist_{d_j}-1$
    \State $clist_{d_j} = clist_{d_j}+1$
    \State ($posc$,$negc$) = RECURSE($b_j$, $U$, $V+1$, $xlist$, $clist$)
    \State $clist_{d_j} = clist_{d_j}-1$
  \EndIf
  \State $\phi_{d_j} = \phi_{d_j}+posx+negc$ \Comment{Save contributions for $d_j$}
  \State \Return ($posx+posc$, $negx+negc$) \Comment{Pass up both contributions}
\EndProcedure
\For{$c$ \textbf{in} $refset$}
  \State \Call{RECURSE}{$0$, $0$, $0$, \textrm{array of $len(x)$ zeros}, \textrm{array of $len(x)$ zeros}}
\EndFor
\State \Return $\phi/\text{len}(refset)$
\EndProcedure
\end{algorithmic}
\end{algorithm}

\textbf{Complexity Analysis: } If we assume {\it CALCWEIGHT} takes constant time (which it will if the factorial function is  implemented based on lookup tables), then Algorithm \ref{alg:ind_tree_shap} performs a constant amount of computation at each node.  This implies the complexity for a single foreground and background sample is $O(L)$, since the number of nodes in a tree is of the same order as the number of leaves.  Repeating this algorithm for each tree and for each background sample gives us $O(TRL)$. $\square$

Note that for the experiments in this paper we used $R = 200$ background samples to produce low variance estimates.

\section{Benchmark evaluation metrics}
\label{sec:benchmark_metrics}

We used 21 evaluation metrics to measure the performance of different explanation methods. These metrics were chosen to capture practical runtime considerations, desirable properties such as local accuracy and consistency, and a range of different ways to measure feature importance. We considered multiple previous approaches and based these metrics off what we considered the best aspects of prior evaluations \cite{ancona2018towards,hooker2018evaluating,lundberg2017unified,shrikumar2016not}. Importantly, we have included three different ways to hide features from the model. One based on mean masking, one based on random sampling under the assumption of feature independence, and one based on imputation. The imputation based metrics attempt to avoid evaluating the model on unrealistic data, but it should be noted that this comes at the cost of encouraging explanation methods to assign importance to correlated features. After extensive consideration, we did not include metrics based on retraining the original model since, while informative, these can produce misleading results in certain situations.

All metrics used to compute comprehensive evaluations of the Shapley value estimation methods we consider are described below (Figure~\ref{fig:cric_tile}, Supplementary Figures~\ref{fig:corrgroups60_tile}~and~\ref{fig:independentlinear60_tile}). Python implementations of these metrics are available online \url{https://github.com/suinleelab/treeexplainer-study}. Performance plots for all benchmark results are also available in Supplementary Data 1.

\subsection{Runtime}

Runtime is reported as the time to explain 1,000 predictions. For the sake of efficiency the runtime for each explanation method was measured using 100 random predictions, and then scaled by 10 to represent the time to explain 1,000 predictions. Both the initialization time of each method and the per-prediction time was measured, and only the per-prediction time was scaled.

\subsection{Local accuracy}

Local accuracy strictly holds only when the sum of the attribution values exactly sum up from some constant base value to the output of the model for each prediction (Property~\ref{prop:local_accuracy}). This means $E_x[(f(x) - \sum_i \phi_i)^2] = 0$. But to also capture how close methods come to achieving local accuracy when they fail, we compute the normalized standard deviation of the difference from the model's output over 100 samples

\begin{equation}
\sigma = \frac{\sqrt{E_x[(f(x) - \sum_i \phi_i)^2]}}{\sqrt{E_x[f(x)^2]}}
\end{equation}

\noindent then define nine cutoff levels of $\sigma$ for reporting a positive score between $0$ and $1$:
\begin{align}
   \sigma < 10^{-6} \implies 1.00 \\
    10^{-6} \le \sigma < 0.01 \implies 0.90 \\
    0.01 \le \sigma < 0.05 \implies 0.75 \\
    0.05 \le \sigma < 0.10 \implies 0.60 \\
    0.10 \le \sigma < 0.20 \implies 0.40 \\
    0.20 \le \sigma < 0.30 \implies 0.30 \\
    0.30 \le \sigma < 0.50 \implies 0.20 \\
    0.50 \le \sigma < 0.70 \implies 0.10
\end{align}

\subsection{Consistency guarantees}

Consistency guarantees are a theoretical property of an explanation method that ensure pairs of cases will never be inconsistent (Property~\ref{prop:consistency}). We broke agreement with this property into three different categories: an exact guarantee, a guarantee that holds in the case of infinite sampling, and no guarantee. Note that while inconsistency could be tested computationally, it would require enumerating a search space exponential in the number of input features, which is why we chose to directly report the theoretical guarantees provided by different methods.

\subsection{Keep positive (mask)}

The Keep Positive (mask) metric measures the ability of an explanation method to find the features that increased the output of the model the most. For a single input the most positive input features are kept at their original values, while all the other input features are masked with their mean value. This is done for eleven different fractions of features ordered by how positive an impact they have as estimated by the explanation method we are evaluating (those that have a negative impact are always removed and never kept). Plotting the fraction of features kept vs. the model output produces a curve that measures how well the local explanation method has identified features that increase the model's output for this prediction. Higher valued curves represent better explanation methods. We average this curve over explanations of 100 test samples for 10 different models trained on different train/test splits.  To summarize how well an explanation method performed we take the area under this curve. Masking features and observing their impact on a model's output is a common method for assessing local explanation methods \cite{ancona2018towards,lundberg2017unified,shrikumar2016not}. An example plot of this metric for a random forest model of the chronic kidney disease model is available in Supplementary Figure~\ref{fig:plot_cric_random_forest_keep_positive_mask}.

\subsection{Keep positive (resample)}

The Keep Positive (resample) metric is similar to the Keep Positive (mask) metric, but instead of replacing hidden features with their mean value, this resample version of the metric replaces them with values from a random training sample. This replacement with values from the training dataset is repeated 100 times and the model output's are averaged to integrate over the background distribution. If the input features are independent then this estimates the expectation of the model output conditioned on the observed features. The mask version of this metric described above can also be viewed as approximating the conditional expectation of the model's output, but only if the model is linear. The resample metric does not make the assumption of model linearity. An example plot of this metric for a random forest model of the chronic kidney disease model is available in Supplementary Figure~\ref{fig:plot_cric_random_forest_keep_positive_resample}.

\subsection{Keep positive (impute)}

The Keep Positive (impute) metric is similar to the Keep Positive (mask) and (resample) metrics, but instead of replacing hidden features with their mean value or a sample from the training set, this impute version of the metric replaces them with values imputed based on the data's correlation matrix. The imputations match the maximum likelihood estimate under the assumption that the inputs features follow a multivariate normal distribution. Unlike the mask and resample versions of this metric, the impute version accounts for correlations among the input features. By imputing we prevent the evaluation of the model on invalid inputs that violate the correlations observed in the data (for example having an input where 'hematocrit' is normal but 'anemia' is true). An example plot of this metric for a random forest model of the chronic kidney disease model is available in Supplementary Figure~\ref{fig:plot_cric_random_forest_keep_positive_impute}.

\subsection{Keep negative (mask)}

The Keep Negative (mask) metric measures the ability of an explanation method to find the features that decreased the output of the model the most. It works just like the Keep Positive (mask) metric described above, but keeps the most negative impacting features as computed by the explanation method (Supplementary Figure~\ref{fig:plot_cric_random_forest_keep_negative_mask}).

\subsection{Keep negative (resample)}

The Keep Negative (resample) metric measures the ability of an explanation method to find the features that decreased the output of the model the most. It works just like the Keep Positive (resample) metric described above, but keeps the most negative impacting features as computed by the explanation method (Supplementary Figure~\ref{fig:plot_cric_random_forest_keep_negative_resample}).

\subsection{Keep negative (impute)}

The Keep Negative (impute) metric measures the ability of an explanation method to find the features that decreased the output of the model the most. It works just like the Keep Positive (impute) metric described above, but keeps the most negative impacting features as computed by the explanation method (Supplementary Figure~\ref{fig:plot_cric_random_forest_keep_negative_impute}).

\subsection{Keep absolute (mask)}

The Keep Absolute (mask) metric measures the ability of an explanation method to find the features most important for the model's accuracy. It works just like the Keep Positive (mask) metric described above, but keeps the most important features as measured by the absolute value of the score given by the explanation method (Supplementary Figure~\ref{fig:plot_cric_random_forest_keep_absolute_mask__roc_auc}). Since removing features by the absolute value of their effect on the model is not designed to push the model's output either higher or lower, we measure not the change in the model's output, but rather the change in the model's accuracy. Good explanations will enable the model to achieve high accuracy with only a few important features.

\subsection{Keep absolute (resample)}

The Keep Absolute (resample) metric measures the ability of an explanation method to find the features most important for the model's accuracy. It works just like the Keep Absolute (mask) metric described above, but uses resampling instead of mean masking (Supplementary Figure~\ref{fig:plot_cric_random_forest_keep_absolute_resample__roc_auc}).

\subsection{Keep absolute (impute)}

The Keep Absolute (impute) metric measures the ability of an explanation method to find the features most important for the model's accuracy. It works just like the Keep Absolute (mask) metric described above, but uses imputing instead of mean masking (Supplementary Figure~\ref{fig:plot_cric_random_forest_keep_absolute_impute__roc_auc}).

\subsection{Remove positive (mask)}

The Remove Positive (mask) metric measures the ability of an explanation method to find the features that increased the output of the model the most. It works just like the Keep Positive (mask) metric described above, but instead of keeping the most positive features, it instead removes them, which should lower the model output (Supplementary Figure~\ref{fig:plot_cric_random_forest_remove_positive_mask}).

\subsection{Remove positive (resample)}

The Remove Positive (resample) metric measures the ability of an explanation method to find the features that increased the output of the model the most. It works just like the Keep Positive (resample) metric described above, but instead of keeping the most positive features, it instead removes them, which should lower the model output (Supplementary Figure~\ref{fig:plot_cric_random_forest_remove_positive_resample}).

\subsection{Remove positive (impute)}

The Remove Positive (impute) metric measures the ability of an explanation method to find the features that increased the output of the model the most. It works just like the Keep Positive (impute) metric described above, but instead of keeping the most positive features, it instead removes them, which should lower the model output (Supplementary Figure~\ref{fig:plot_cric_random_forest_remove_positive_impute}).

\subsection{Remove negative (mask)}

The Remove Negative (mask) metric measures the ability of an explanation method to find the features that decreased the output of the model the most. It works just like the Keep Negative (mask) metric described above, but instead of keeping the most negative features, it instead removes them, which should raise the model output (Supplementary Figure~\ref{fig:plot_cric_random_forest_remove_negative_mask}).

\subsection{Remove negative (resample)}

The Remove Negative (resample) metric measures the ability of an explanation method to find the features that decreased the output of the model the most. It works just like the Keep Negative (resample) metric described above, but instead of keeping the most negative features, it instead removes them, which should raise the model output (Supplementary Figure~\ref{fig:plot_cric_random_forest_remove_negative_resample}).

\subsection{Remove negative (impute)}

The Remove Negative (impute) metric measures the ability of an explanation method to find the features that decreased the output of the model the most. It works just like the Keep Negative (impute) metric described above, but instead of keeping the most negative features, it instead removes them, which should raise the model output (Supplementary Figure~\ref{fig:plot_cric_random_forest_remove_negative_impute}).

\subsection{Remove absolute (mask)}

The Remove Absolute (mask) metric measures the ability of an explanation method to find the features most important for the model's accuracy. It works just like the Keep Absolute (mask) metric described above, but instead of keeping the most important features, it instead removes them, which should lower the model's performance (Supplementary Figure~\ref{fig:plot_cric_random_forest_remove_absolute_mask__roc_auc}).

\subsection{Remove absolute (resample)}

The Remove Absolute (resample) metric measures the ability of an explanation method to find the features most important for the model's accuracy. It works just like the Keep Absolute (resample) metric described above, but instead of keeping the most important features, it instead removes them, which should lower the model's performance (Supplementary Figure~\ref{fig:plot_cric_random_forest_remove_absolute_resample__roc_auc}).

\subsection{Remove absolute (impute)}

The Remove Absolute (impute) metric measures the ability of an explanation method to find the features most important for the model's accuracy. It works just like the Keep Absolute (impute) metric described above, but instead of keeping the most important features, it instead removes them, which should lower the model's performance (Supplementary Figure~\ref{fig:plot_cric_random_forest_remove_absolute_impute__roc_auc}).

\section{User study experiments}
\label{sec:methods_user_study}

Here we explore how consistent different attribution methods are with human intuition. While most complex models cannot be easily explained by humans, very simple decision trees can be explained by people. This means we can run user study experiments that measure how people assign credit to input features for simple models that people understand, then compare people's allocations to allocations given by different local explanation methods. If an explanation method gives the same result as humans, then we say that method is consistent with human intuition.

To test human intuition, we use four simple models and ask 33 english speaking individuals in the U.S. on Amazon Mechanical Turk to explain three different samples for each model. Each model is a simple depth-two decision tree that only depends on two binary features. To make the model more approachable we called the model a "sickness score", and used three binary input features: {\it fever}, {\it cough}, and {\it headache}. The models we used were: {\it AND}, {\it OR}, {\it XOR}, and {\it SUM} (study participants were not told these names). The headache feature was never used by any of the models. The fever and cough features always each contributed a linear effect of $+2$ when they were on, but for models other than SUM there were also non-linear effects. For AND $+6$ was given when both features were true. For OR $+6$ was given when either feature was true. For XOR $+6$ was given when either feature was true but not both. For each model we explained three samples: 1) fever false, cough false, headache true; 2) fever false, cough true, headache true; and 3) fever true, cough true, headache true.

Users were asked to allocate blame for a sickness score output value among each of the three input features. No constraints were placed on the input fields used to capture user input, except that the fields were not left blank when there was a non-zero sickness score. This experiment resulted in twelve different consensus credit allocations (Supplementary Figures~\ref{fig:human_and_survey}-\ref{fig:human_sum_survey}). We then took these twelve consensus credit allocations and used them as ground truth for twelve metrics. Each metric is the sum of the absolute differences between the human consensus feature attribution and attribution given by an explanation method (Supplementary Figures~\ref{fig:plot_human_decision_tree_human_and_11}-\ref{fig:plot_human_decision_tree_human_xor_11}). Note that we used a healthy population with independent input features as the background reference dataset (the background is used for computing conditional expectations by the explanation methods).

The results show that all the local explanation methods based on Shapley values agree with the human consensus feature attribution values across all twelve cases. However, the heuristic Saabas method differs significantly from the human consensus in several of the nonlinear cases (Supplementary Figure~\ref{fig:human_tile}). Not only do Shapley values have attractive theoretical guarantees, and strong quantitative performance (Figure~\ref{fig:cric_tile}), but these experiments show they also match human intuition on small example models (Supplementary Figure~\ref{fig:human_tile}). Python implementations of these study scenarios are available online \url{https://github.com/suinleelab/treeexplainer-study}. Performance plots for all user study results are also available in Supplementary Data 1.

\section{SHAP interaction values}
\label{sec:methods_shap_interaction_values}

Here we describe the new richer explanation model we proposed to capture local interaction effects; it is based on the Shapley interaction index from game theory. The Shapley interaction index is a more recent concept than the classic Shapley values, and follows from generalizations of the original Shapley value properties \cite{fujimoto2006axiomatic}. It can allocate credit not just among each player of a game, but among all pairs of players. While standard feature attribution results in a vector of values, one for each feature, attributions based on the Shapley interaction index result in a matrix of feature attributions. The main effects are on the diagonal and the interaction effects on the off-diagonal. If we use the same definition of $f_x$ that we used to get Sabaas values and SHAP values, but with the Shapley interaction index, we get {\it SHAP interaction values} \cite{fujimoto2006axiomatic}, defined as:
\begin{equation}
\label{eq:shapley_interactions}
\Phi_{i,j}(f,x) = \sum_{S \subseteq \mathcal{M} \setminus \{i,j\}} \frac{|S|!(M - |S| - 2)!}{2(M - 1)!} \nabla_{ij}(f, x, S),
\end{equation}
\noindent when $i \ne j$, and
\begin{align}
\label{eq:shapley_interaction_grad}
\nabla_{ij}(f, x, S) &= f_x(S \cup \{i,j\}) - f_x(S \cup \{i\}) - f_x(S \cup \{j\}) + f_x(S) \\
&= f_x(S \cup \{i,j\}) - f_x(S \cup \{j\}) - [f_x(S \cup \{i\}) - f_x(S)]. \label{eq:shapley_interaction_grad2}
\end{align}
\noindent where $\mathcal{M}$ is the set of all $M$ input features. In Equation~\ref{eq:shapley_interactions} the SHAP interaction value between feature $i$ and feature $j$ is split equally between each feature so $\Phi_{i,j}(f, x) = \Phi_{j,i}(f, x)$ and the total interaction effect is $\Phi_{i,j}(f, x) + \Phi_{j,i}(f, x)$. The main effects for a prediction can then be defined as the difference between the SHAP value and the off-diagonal SHAP interaction values for a feature:
\begin{align}
\label{eq:shapley_interaction_main}
\Phi_{i,i}(f, x) &= \phi_i(f, x)  - \sum_{j \ne i} \Phi_{i,j}(f, x)
\end{align}
\noindent We then set $\Phi_{0,0}(f, x) = f_x(\emptyset)$ so $\Phi(f, x)$ sums to the output of the model:
\begin{align}
\label{eq:shapley_interaction_main2}
\sum_{i=0}^M \sum_{j=0}^M \Phi_{i,j}(f, x) &= f(x)
\end{align}


While SHAP interaction values could be computed directly from Equation~\ref{eq:shapley_interactions}, we can leverage Algorithm~\ref{alg:tree_shap} to drastically reduce their computational cost for tree models. As highlighted in Equation~\ref{eq:shapley_interaction_grad2}, SHAP interaction values can be interpreted as the difference between the SHAP values for feature $i$ when feature $j$ is present and the SHAP values for feature $i$ when feature $j$ is absent. This allows us to use Algorithm~\ref{alg:tree_shap} twice, once while ignoring feature $j$ as fixed to present, and once with feature $j$ absent. This leads to a run time of $O(TMLD^2)$ when using Tree SHAP, since we repeat the process for each feature. A full open-source implementation is available online \url{https://github.com/suinleelab/treeexplainer-study}.

SHAP interaction values have uniqueness guarantees similar to SHAP values \cite{fujimoto2006axiomatic}, and allow the separate consideration of main and interaction effects for individual model predictions. This separation can uncover important interactions captured by tree ensembles that might otherwise be missed (Section~\ref{sec:interaction_effects}). While previous work has used global measures of feature interactions \cite{lunetta2004screening,jiang2009random}, to the best of our knowledge SHAP interaction values represent the first local approach to feature interactions beyond simply listing decision paths.

\section{Model summarization experiments}
\label{sec:methods_model_summary}

Here, we describe in more detail the model summarization results introduced in Section~\ref{sec:model_summary}. One of the most basic ways to understand a model is to display the global importance of each feature, often as a bar chart (Figure~\ref{fig:overview_and_dependence}A left). For tree-based models such as gradient boosted trees a basic operation supported by any implementation is providing the total ``Gain'' over all splits for a feature as a global measure of feature importance \cite{friedman2001elements}. Computing SHAP values across a whole dataset, we can improve on this basic task of displaying global feature importance in two ways: 1) By averaging the SHAP values across a dataset, we can get a single global measure of feature importance that retains the theoretical guarantees of SHAP values. This avoids the troubling inconsistency problems of the classic Gain method (Supplementary Figure~\ref{fig:and_trees_simple}), and also provides better global feature selection power than either Gain or permutation testing (Supplementary Figure~\ref{fig:feature_selection}). This is particularly important since tree-based models are often used in practice for feature selection \cite{genuer2010variable,pan2009feature}. 2) A limitation of traditional global explanations for trees is that reporting a single number as the measure of a feature's importance conflates two important and distinct concepts: the magnitude of an effect, and the prevalence of an effect. By plotting many local explanations in a beeswarm-style {\it SHAP summary plot} we can see both the magnitude and prevalence of a feature's effect (Figure~\ref{fig:overview_and_dependence}A right), and by adding color, we can also display the effect's direction.

The value of summary plots based on many local explanations is illustrated in Figure~\ref{fig:overview_and_dependence}A for the NHANES I mortality dataset, where an XGBoost cox proportional hazards model was trained using hyper-parameters optimized on a validation dataset (\ref{sec:methods_model_training_details}), its predictions were explained using TreeExplainer, and then compiled into a summary plot. On the left of Figure~\ref{fig:overview_and_dependence}A is a familiar bar-chart, not based on a typical heuristic global measure of feature importance, but on the average magnitude of the SHAP values. On the right of Figure~\ref{fig:overview_and_dependence}A is a set of beeswarm plots where each dot corresponds to an individual person in the study. Each person has one dot for each feature, where the position of the dot on the x-axis corresponds to the impact that feature has on the model's prediction for that person (as measured by the prediction's SHAP value for that feature). When multiple dots land at the same x position they pile up to show density.

Unsurprisingly, the dominating factor for risk of death in the U.S. in the 1970s (which is when the NHANES I data was collected) is age. By examining the top row of the summary plot we can see that a high value for the age feature (red) corresponds to a large increase in the log hazard ratio (i.e., a large positive SHAP value), while a low value for age (blue) corresponds to a large decrease in the log hazard ratio (i.e., a large negative SHAP value). The next most important feature for mortality prediction is sex, with men having about a 0.6 increase in log-hazards relative to women, which corresponds to about 7 years of change in the age feature. Interestingly, the impact of being a man vs. woman on the model's output is not constant across individuals, as can be seen by the spread of the blue and red dots. The differences of effect within the same sex are due to interactions with other features in the model that modulate the importance of sex for different individuals (we explore this in more detail in \ref{sec:methods_feature_dependence}).

An important pattern in the mortality data revealed by the summary plot, but not by classic global feature importance, is that features with a low global importance can still be some of the most important features for a specific person (Figure~\ref{fig:overview_and_dependence}A). Blood protein, for example, has a low global impact on mortality, as indicated by its small global importance (Figure~\ref{fig:overview_and_dependence}A left). However, for some individuals, high blood protein has a very large impact on their mortality risk, as indicated by the long tail of red dots stretching to the right in the summary plot (Figure~\ref{fig:overview_and_dependence}A right). This trend of rare high magnitude effects is present across many of the features, and always stretches to the right. This reflects the fact that there are many ways to die abnormally early when medical measurements are out of range, but there not many ways to live abnormally longer (since there are no long tails stretching to the left). Summary plots combine many local explanations to provide a more detailed global picture of the model than a traditional list of global feature importance values. In many cases this more detailed view can provide useful insights, as demonstrated in our medical datasets  (Figure~\ref{fig:overview_and_dependence}A, Supplementary Figures~\ref{fig:kidney_summary}~and~\ref{fig:hospital_summary}).

\section{Feature dependence experiments}
\label{sec:methods_feature_dependence}

Here we describe in more detail the feature dependence results introduced in Section~\ref{sec:feature_dependence}. Just as summary plots provide richer information than traditional measures of global feature importance (Figure~\ref{fig:overview_and_dependence}A), dependence plots based on SHAP values can provide richer information than traditional partial dependence plots by combining the local importance of a single feature across many samples (Figure~\ref{fig:overview_and_dependence}B-G). Plotting a feature's value on the x-axis vs. the feature's SHAP value on the y-axis produces a {\it SHAP dependence plot} that shows how much that feature impacted the prediction of every sample in the dataset.

For the mortality model a SHAP dependence plot reproduces the standard risk inflection point known for systolic blood pressure between 120 mmHg and 140 mmHg \cite{sprint2015randomized} (Figure~\ref{fig:overview_and_dependence}B). This highlights the value of a flexible model that is able to capture non-linear 
patterns, and also shows how interaction effects can have a significant impact on a feature. Standard partial dependence plots capture the general trends, but do not provide any information about the heterogeneity present within people that have the same measured systolic blood pressure (Supplemental Figure \ref{fig:nhanes_pdp_sbp}). Each dot in a SHAP dependence plot represents a person, and the vertical dispersion of the dots is driven by interaction effects with other features in the model. Many different individuals have a recorded blood pressure of 180 mmHg in the mortality dataset, but the impact that measurement has on their log-hazard ratio varies from 0.2 to 0.6 because of other factors that differ among these individuals. We can color by another feature to better understand what drives this vertical dispersion. Coloring by age in Figure~\ref{fig:overview_and_dependence}B explains most of the dispersion, meaning that early onset high blood pressure is more concerning to the model than late onset high blood pressure.

\nisha{For the chronic kidney disease (CKD) model a SHAP dependence plot again clearly reveals the previously documented non-linear inflection point for systolic blood pressure risk, but in this dataset the vertical dispersion from interaction effects appears to be partially driven by differences in blood urea nitrogen (Figure~\ref{fig:overview_and_dependence}E). Correctly modeling blood pressure risk is important, since blood pressure control in select CKD populations may delay progression of kidney disease and reduce the risk of cardiovascular events. Lower blood pressure has been found to slow progression of CKD and decrease overall cardiovascular mortality in some studies \cite{walker1992renal,rosansky1990association,shulman1989prognostic,perry1995early,sarnak2005effect,bakris2003effects}. For example, long-term follow-up of the Modification of Diet in Renal Disease (MDRD) study suggested that lower systolic blood pressure led to improved kidney outcomes in patients with CKD \cite{sarnak2005effect}.  The SPRINT trial, which randomized patients to treatment to systolic blood pressure <120 vs. <140 mmHg found that treatment to lower systolic blood pressure was associated with lower risk of cardiovascular disease; though no difference was seen in rates of CKD progression between the treatment groups \cite{sprint2015randomized,cheung2017effects}.}

\section{Interaction effect experiments}
\label{sec:methods_interaction_effects}

Here we describe in more detail the interaction effect results introduced in Section~\ref{sec:interaction_effects}. As mentioned in Section~\ref{sec:interaction_effects}, using SHAP interaction values we can decompose the impact of a feature on a specific sample into a main effect and interaction effects with other features. 
SHAP interaction values allow pairwise interaction effects to be measured at an individual sample level. By combining many such samples we can then observe patterns of interaction effects across a dataset.

In the mortality dataset, we can compute the SHAP interaction values for every sample and then decompose the systolic blood pressure dependence plot into two components. One component contains the main effect of systolic blood pressure and interaction effects with features that are not age, and the other component is the (symmetric) SHAP interaction value of systolic blood pressure and age. The main effect plus the interaction effects equals the original SHAP value for a sample, so we can add the y-values of Figure~\ref{fig:overview_and_dependence}C and Figure~\ref{fig:overview_and_dependence}D to reproduce Figure~\ref{fig:overview_and_dependence}B. Interaction effects are visible in dependence plots through vertical dispersion of the samples, and coloring can often highlight patterns likely to explain this dispersion, but it is necessary to compute the SHAP interaction values to confirm the causes of vertical dispersion. In the systolic blood pressure dependence plot from the mortality model the vertical dispersion is primarily driven by an interaction with age (Figure~\ref{fig:overview_and_dependence}D), as suggested by the original coloring (Figure~\ref{fig:overview_and_dependence}B).

\nisha{Plotting interaction values can reveal interesting relationships picked up by complex tree-ensemble models that would otherwise be hidden, such as the interaction effect between age and sex described in Section~\ref{sec:interaction_effects}.}
%
\nisha{In the chronic kidney disease model an interesting interaction effect is observed between `white blood cells' and `blood urea nitrogen' (Figure~\ref{fig:overview_and_dependence}F). This means that high white blood cell counts are more concerning to the model when they are accompanied by high blood urea nitrogen. Recent evidence has suggested that inflammation may be an important contributor to loss of kidney function \cite{bowe2017association,fan2017white}. While there are numerous markers of inflammation, white blood cell count is one of the most commonly measured clinical tests available, and this interaction effect supports the notion that inflammation may interact with high blood urea nitrogen to contribute to faster kidney function decline.}

\section{Model monitoring experiments}
\label{sec:methods_model_monitoring}

Here we describe in more detail the model monitoring results introduced in Section~\ref{sec:model_monitoring}. As noted in Section~\ref{sec:model_monitoring}, deploying machine learning models in practice is challenging because they depend on a large range of input features, any of which could change after deployment and lead to degraded performance. Finding problems in deployed models is difficult because the result of bugs is typically not a software crash, but rather a change in an already stochastic measure of prediction performance. It is hard to determine when a change in a model's performance is due to a feature problem, an expected generalization error, or random noise. Because of this, many bugs in machine learning pipelines can go undetected, even in core software at top tech companies \cite{zinkevich2017rules}.

A natural first step when debugging model deployments is to identify which features are causing problems. Computing the SHAP values of a model's loss function directly supports this by decomposing the loss among the model's input features. This has two important advantages over traditional model loss monitoring: First, it assigns blame directly to the problematic features so that instead of looking at global fluctuations of model performance, one can see the impact each feature has on the performance. Second, by focusing on individual features we have higher power to identify problems that would otherwise be hidden in all the other fluctuations of the overall model loss. Improving our ability to monitor deployed models is an important part of enabling the safe use of machine learning in medicine.

As mentioned in Section~\ref{sec:model_monitoring}, to simulate a model deployment we used the hospital procedure duration prediction dataset. It contains four years of data from two large hospitals. We used the first year of data for training and ran the model on the next three years of data to simulate a deployment. This is a simple batch prediction task, and so is far less prone to errors than actual real-time deployments, yet even here we observed dataset issues that had not been previously detected during data cleaning (Figure~\ref{fig:key_monitoring_plots}).

Figure~\ref{fig:key_monitoring_plots}A shows the smoothed loss of the procedure duration model over time, and represents the type of monitoring used widely in industry today. There is a clear increase in the model's error once we switch to the test set, which was expected; then there are short spikes in the error that are hard to differentiate from random fluctuations. To test the value of using monitoring plots based on local explanations we intentionally swapped the labels of operating rooms 6 and 13 two-thirds of the way through the dataset. This is meant to represent the type of coding change that can often happen during active development of a real-time machine learning pipeline. If we look at the overall loss of the model's predictions two-thirds of the way through we see no indication that a problem has occurred (Figure~\ref{fig:key_monitoring_plots}A), which means this type of monitoring would not be able to catch the issue. In contrast, the {\it SHAP monitoring plot} for the room 6 feature clearly shows when the room labeling error begins (Figure~\ref{fig:key_monitoring_plots}B). The y-axis of the SHAP monitoring plot is the impact of the room 6 feature on the loss. About two-thirds of the way through the data we see a clear shift from negative values (meaning using the feature helps accuracy) to positive values (meaning using the feature hurts accuracy). The impact on accuracy is substantial, but because procedures occurring in room 6 and 13 are just a small fraction of the overall set of procedures, the increased error is invisible when looking at the overall loss (Figure~\ref{fig:key_monitoring_plots}A).

In addition to finding our intentional error, SHAP monitoring plots also revealed problems that were already present in the dataset. Two of these are shown Figure~\ref{fig:key_monitoring_plots}C and D.

\bala{In Figure~\ref{fig:key_monitoring_plots}C we can see a spike in error for the general anesthesia feature shortly after the deployment window begins. This spike represents a transient configuration issue in our hospital revealed by a SHAP monitoring plot. It corresponds to a subset of procedures from a single hospital where the anesthesia type data field was left blank. After going back to the hospital system we found that this was due to a temporary electronic medical record configuration issue whereby the link between the general anesthesia induction note and the anesthesia type got broken. This prevented the anesthesia type data field from being auto-populated with  “general anesthesia” when the induction note was completed. This is exactly the type of configuration issue that impacts machine learning models deployed in practice, and so needs to be detected during model monitoring.}

\jordan{In Figure~\ref{fig:key_monitoring_plots}D we see an example, not of a processing error, but of feature drift over time. The atrial fibrillation feature denotes if a patient is undergoing an atrial fibrillation ablation procedure. During the training period, and for some time into deployment, using the atrial fibrillation feature lowers the loss. But then over the course of time the feature becomes less and less useful until it begins hurting the model by the end of the deployment window. Based on the SHAP monitoring plot for atrial fibrillation we went back to the hospital to determine possible causes. During an atrial fibrillation ablation, the goal is to electrically isolate the pulmonary veins of the heart using long catheters placed from the groin area. Traditionally, the procedure was completed with a radiofrequency ablation catheter delivering point-by-point lesions to burn the left atrium around the pulmonary veins. During the deployment window, the hospital began to use the second generation cryoballoon catheter (Arctic Front Advance, Medtronic Inc., Minneapolis, MN), which freezes the tissue and has been demonstrated to have a shorter procedure duration compared to radiofrequency ablation \cite{Kuck2016}. At the same time, there were improvements in radiofrequency ablation catheter technology including the use of contact force sensing which allowed the operator to determine how strongly the catheter was touching the left atrial wall. This technology ensures that ablation lesions are delivered with enough force to create significant lesion size. With noncontact force catheters, the operator may think the catheter is touching the atrial wall but it may, in actuality, simply be floating nearby. Contact force sensing is also associated with shorter procedure times \cite{Marijon2014,Kimura2014}. Cryoballoon versus radiofrequency ablation is chosen based on patient
characteristics and physician preference. Lastly, during this time there were staffing changes including the use of specialized electrophysiology technologists which decreased procedural preparation time. All of these changes
led to a significant decrease in atrial fibrillation ablation procedural and in-room time during the test period (Supplementary Figure~\ref{fig:afib_duration_plot}), which translated into high model error attributable to the atrial fibrillation feature.  We quantified the significance of the SHAP monitoring plot trend using an independent t-test between atrial fibrillation ablation procedures that appear in the first 30,000 samples of the test time period vs. those that appear in the last 60,000 samples of the test time period. This lead to a P-value of $5.4 \times 10^{-19}$. The complexity of the reasons behind the feature drift in  Figure~\ref{fig:key_monitoring_plots}D illustrate why it is so difficult to anticipate how assumptions depended on by a model might break in practice. Using SHAP values to monitor the loss of a model allow us to retroactively identify how model assumptions may be changing individually for each input feature, even if we cannot a priori anticipate which features are likely to cause problems.}

Explaining the loss of a model is not only useful for monitoring over time, it can also be used in dependence plots to understand how a specific feature helps improve model performance. For hospital procedure duration prediction we can plot the dependence between the time of the day feature and it's impact on the model's output. The time of day feature indicates when a procedure started, and is 
particularly effective at reducing the model's loss just after 7:30am and 8:30am in the morning (Supplementary Figure~\ref{fig:hospital_time_of_day_loss_dependence}). These two times are when long-running elective surgeries are scheduled in this hospital system, so the dependence plot reveals that the model is using the time of the day to detect routine long-running surgeries.

Current practice is to monitor the overall loss of the model, and also potentially monitor the statistics of the features for changes over time. This is problematic because just monitoring the overall loss of the model can hide many important problems, while monitoring the changes in the statistics of features is essentially an unsupervised anomaly detection problem that is prone to both false positives and false negatives. Swapping the names of operating rooms that are used equally often would be invisible to such an unsupervised method. By directly attributing the loss to the features we can highlight precisely which features are impacting the loss and by how much. When changes show up in the monitoring plot, those changes are in the units of the model's loss and so we can quantify how much it is impacting our performance. These monitoring plots represent a compelling way that many local explanations can be combined to provide richer and more actionable insights into model behavior than the current state of the art.

\section{Local explanation embedding experiments}
\label{sec:methods_embedding}

Here we describe in more detail the local explanation embedding results introduced in Section~\ref{sec:embedding}. \nisha{
There are two challenges with standard unsupervised clustering methods: 1) 
The distance metric does not account for the discrepancies among the units and meaning of features (e.g., units of years vs. units of cholesterol), and simple standardization is no guarantee the resulting numbers are comparable.
2) Even after a distance metric is defined, 
there is no way for an unsupervised approach to know which features are relevant for an outcome of interest, and so should be weighted more strongly.
Some applications might seek to cluster patients by groups relating to kidney disease, and another by groups relating to diabetes. But given the same feature set, unsupervised clustering will give the same results in both cases.
}

\nisha{As mentioned in Section~\ref{sec:embedding}, we can address both of the above problems in traditional unsupervised clustering by using local explanation embeddings to embed each sample into a new ``explanation space.'' If we then run clustering in this new space, we will get a {\it supervised clustering} where samples are grouped together that have the same model output for the same reason. Since SHAP values have the same units as the model's output they are all comparable within a model, even if the original features were not comparable. Supervised clustering naturally accounts for the differing scales of different features, only highlighting changes that are relevant to a particular outcome.}
\nisha{Running hierarchical supervised clustering using the mortality model results in groups of people that share a similar mortality risk for similar reasons (Figure~\ref{fig:clustering_and_embedding}A). The heatmap in Figure~\ref{fig:clustering_and_embedding}A uses the leaf order of a hierarchical agglomerative clustering based on the SHAP values of each sample. On the left are young people, in the middle are middle aged people, and on the right are older people. Within these broad categories many smaller groups are of particular interest, such as people with early onset high-blood pressure, young people with inflammation markers, and underweight older people. Each of these recapitulate known high risk groups of people, which would not have been captured by a simple unsupervised clustering approach (Supplementary Figure~\ref{fig:mortality_clustering_unsupervised}).}

\nisha{In addition to making explicit clusters, we can also use dimensionality reduction to directly visualize the explanation space embedding produced by the SHAP values. This gives a continuous representation of the primary directions of model output variation in the dataset. For the kidney disease dataset, the top two principal components of the explanation embedding highlight two distinct risk factors for disease progression (Figures~\ref{fig:clustering_and_embedding}B-D). The first principal component aligns with blood creatinine levels, which are used to compute the estimated glomerular filtration rate (eGFR) of the kidneys. High levels of creatinine are a marker of lower eGFR and are the primary test for detection of kidney disease. The second principal component aligns with higher urine protein concentration in the urine. Quantified as the urine protein to urine creatinine ratio (PCR), this is a marker of kidney damage and is used in conjunction with eGFR to quantify levels of kidney disease.
If we color the explanation space embedding by the risk of kidney disease progression, we see a roughly continuous increase in risk from left to right (Figure~\ref{fig:clustering_and_embedding}B). This is largely explained by a combination of the two orthogonal risk directions described above: One direction follows the blood creatinine level feature (which determine the eGFR) (Figure~\ref{fig:clustering_and_embedding}C) and the other direction follows the urine protein feature (Figure~\ref{fig:clustering_and_embedding}D). Several of the other top features in the chronic kidney disease model also align with these two orthogonal embedding directions (Supplementary Figure~\ref{fig:clustering_and_embedding}B-D). It is well established that eGFR and urine PCR are the strongest predictors of progression to end-stage renal disease among patients with chronic kidney disease \cite{matsushita2015estimated,chronic2010association}. Physiologically, eGFR and PCR are likely complimentary, yet distinct kidney disease markers. Figure~\ref{fig:clustering_and_embedding}B-D shows that eGFR and PCR each identify unique individuals at risk of end-stage renal disease; thus confirming that clinically they should be measured in parallel. This type of insight into the overall structure of kidney risk is not at all apparent when just looking at a standard unsupervised embedding (Supplementary Figure~\ref{fig:kidney_raw_pca}).}

\clearpage

\section*{Supplementary Figures}

\renewcommand{\figurename}{Supplementary Figure}
\setcounter{figure}{0}


\begin{figure*}
  \centering
  \includegraphics[width=0.9\textwidth]{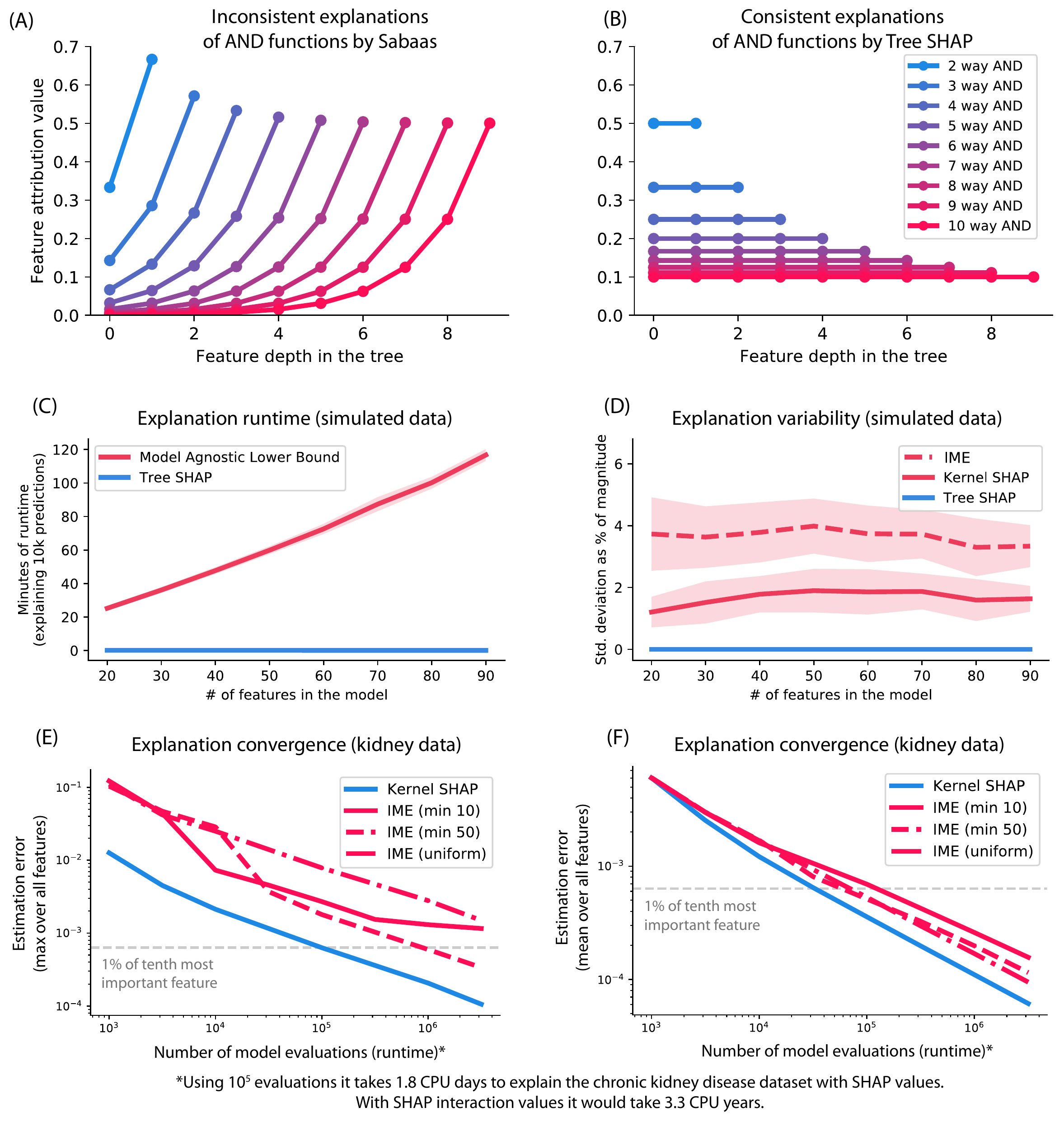}
  \caption{
  {\bf (A-B) TreeExplainer avoids the consistency problems of previous tree-specific approaches.} For tree models that represent a multi-way AND function the Saabas method gives little credit to features near the root, while Tree SHAP evenly distributes credit among all the features involved in the AND function.  {\bf (C-F) TreeExplainer represents a dramatic performance improvement over model agnostic approaches.} All model agnostic approaches rely on sampling, so their runtime is lower bounded by the run-time to evaluate the model, and they always have some sampling variability. TreeExplainer's Tree SHAP algorithm runs thousands of times faster, and has no sampling variability since it is exact. We consider both the Kernel SHAP \cite{lundberg2017unified} and IME \cite{vstrumbelj2014explaining} model agnostic estimators for Shapley values. (C-D) Using simulated data we can train XGBoost models over datasets of different sizes and observe that the number of samples required to maintain estimates with constant variance grows linearly with the number of features in the model. The reported runtime is a lower bound since it only accounts for the time to execute the model, not execute the explanation method itself.
  (E-F) As we increase the number of model evaluations used by model agnostic approaches we converge towards the exact solution. However, achieving low variance estimates requires days or years of CPU time even on the smallest of our medical datasets. We also only used a single background reference sample for computing conditional expectations in the model agnostic methods (instead of an entire dataset) so these represent lower bounds on the runtime. The estimation error on the y-axis represents the difference from the exact solution with a single background reference sample. Details of the experimental setup for C-F are described in \ref{sec:methods_agnostic_convergence}.} 
  \label{fig:tree_shap_performance}
\end{figure*}

\begin{figure*}[!ht]
  \centering
  \includegraphics[width=1.0\textwidth]{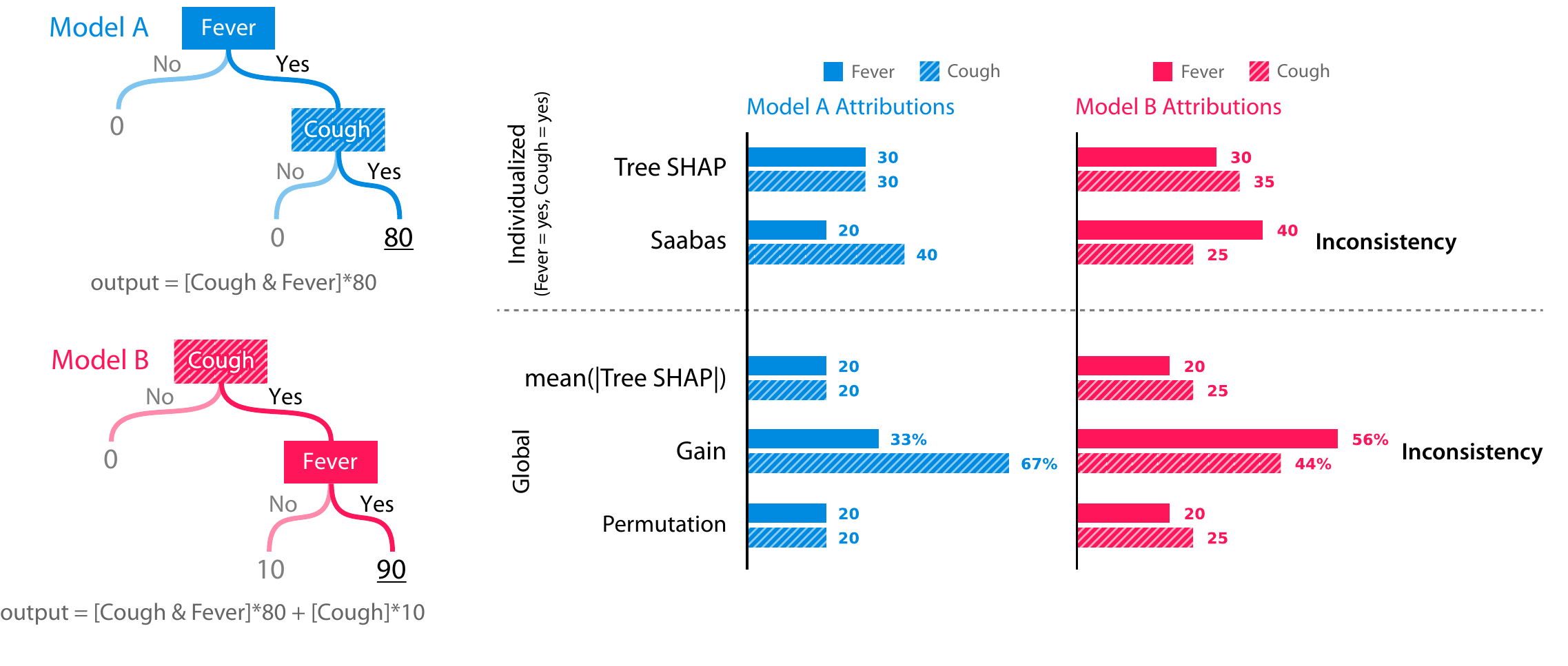}
  \caption[Simple example of an inconsistency]{{\bf Two simple tree models that demonstrate inconsistencies in the Saabas and gain feature attribution methods.} The Cough feature has a larger impact in Model B than Model A, but is attributed less importance in Model B. Similarly, the Cough feature has a larger impact than Fever in Model B, yet is attributed less importance. The individualized attributions explain a single prediction of the model (when both Cough and Fever are Yes) by allocating the difference between the expected value of the model's output (20 for Model A, 25 for Model B) and the current output (80 for Model A, 90 for Model B). Inconsistency prevents the reliable comparison of feature attribution values. The global attributions represent the overall importance of a feature in the model. ``Gain'' is the most common way of measuring feature importance in trees and is the sum of the reductions in loss that come from all splits using that feature. ``Permutation`` is the change in the model's accuracy when a single feature is permuted.
  }
  \label{fig:and_trees_simple}
\end{figure*}

\begin{figure*}[!ht]
  \centering
  \includegraphics[width=0.9\textwidth]{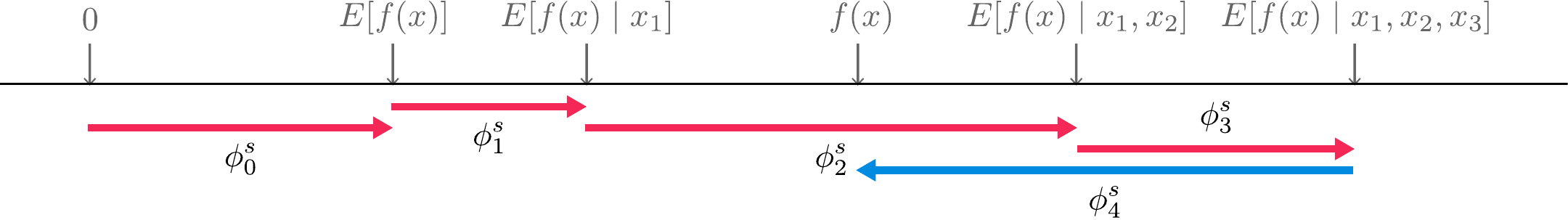}
  \caption{The Sabaas values, $\phi^s_i$, attribute feature importance by measuring differences in conditional expectations along the order defined by the decision path. This is very similar to SHAP (\underline{SH}apley \underline{A}dditive ex\underline{P}lanation) values, $\phi_i$, except SHAP values result from averaging over all possible orderings. This is important since for non-linear functions the order in which features are introduced matters. Proofs from game theory show that averaging over all orderings is the only possible consistent approach where $\sum_{i=0}^M \phi_i = f(x)$ (\ref{sec:methods_tree_explainer}).}
  \label{fig:number_line}
\end{figure*}


\begin{figure*}
  \centering
  \makebox[\textwidth][c]{\includegraphics[width=1.2\textwidth]{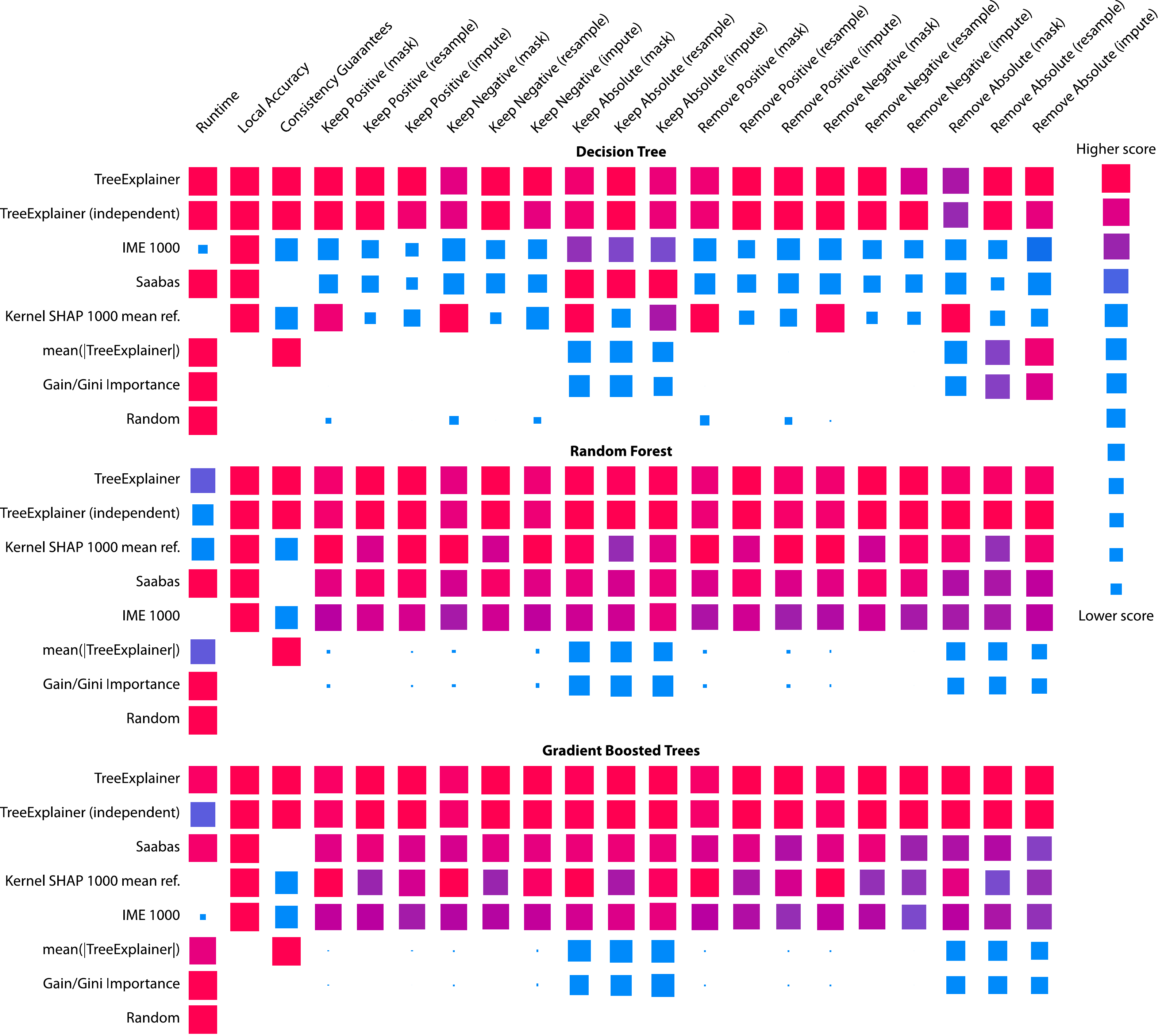}}
  \caption{{\bf Explanation method performance across thirteen different evaluation metrics and three regression models in a simulated dataset with 60 features divided into tightly correlated groups.} Each tile represents the performance of a feature attribution method on a given metric for a given model. Within each model the columns of tiles are scaled between the minimum and maximum value, and methods are sorted by their overall performance. Some of these metrics have been proposed before and others are new quantitative measures of explanation performance that we are introducing (see Methods).
}
  \label{fig:corrgroups60_tile}
\end{figure*}

\begin{figure*}
  \centering
  \makebox[\textwidth][c]{\includegraphics[width=1.2\textwidth]{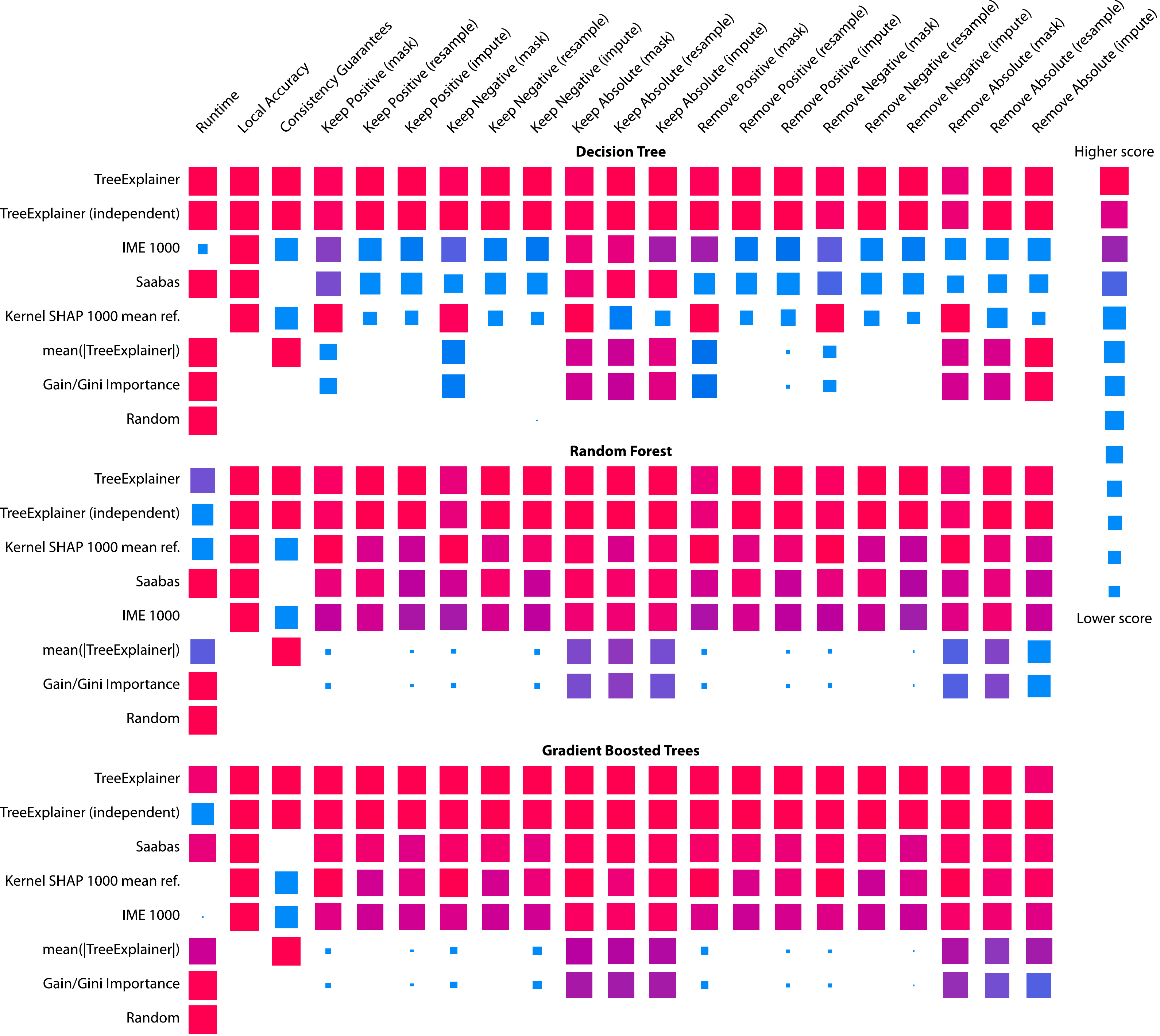}}
  \caption{{\bf Explanation method performance across thirteen different evaluation metrics and three regression models in a simulated dataset with 60 independent features.} Each tile represents the performance of a feature attribution method on a given metric for a given model. Within each model the columns of tiles are scaled between the minimum and maximum value, and methods are sorted by their overall performance. Some of these metrics have been proposed before and others are new quantitative measures of explanation performance that we are introducing (see Methods).}
  \label{fig:independentlinear60_tile}
\end{figure*}

\begin{figure*}
  \centering
  \includegraphics[width=1.0\textwidth]{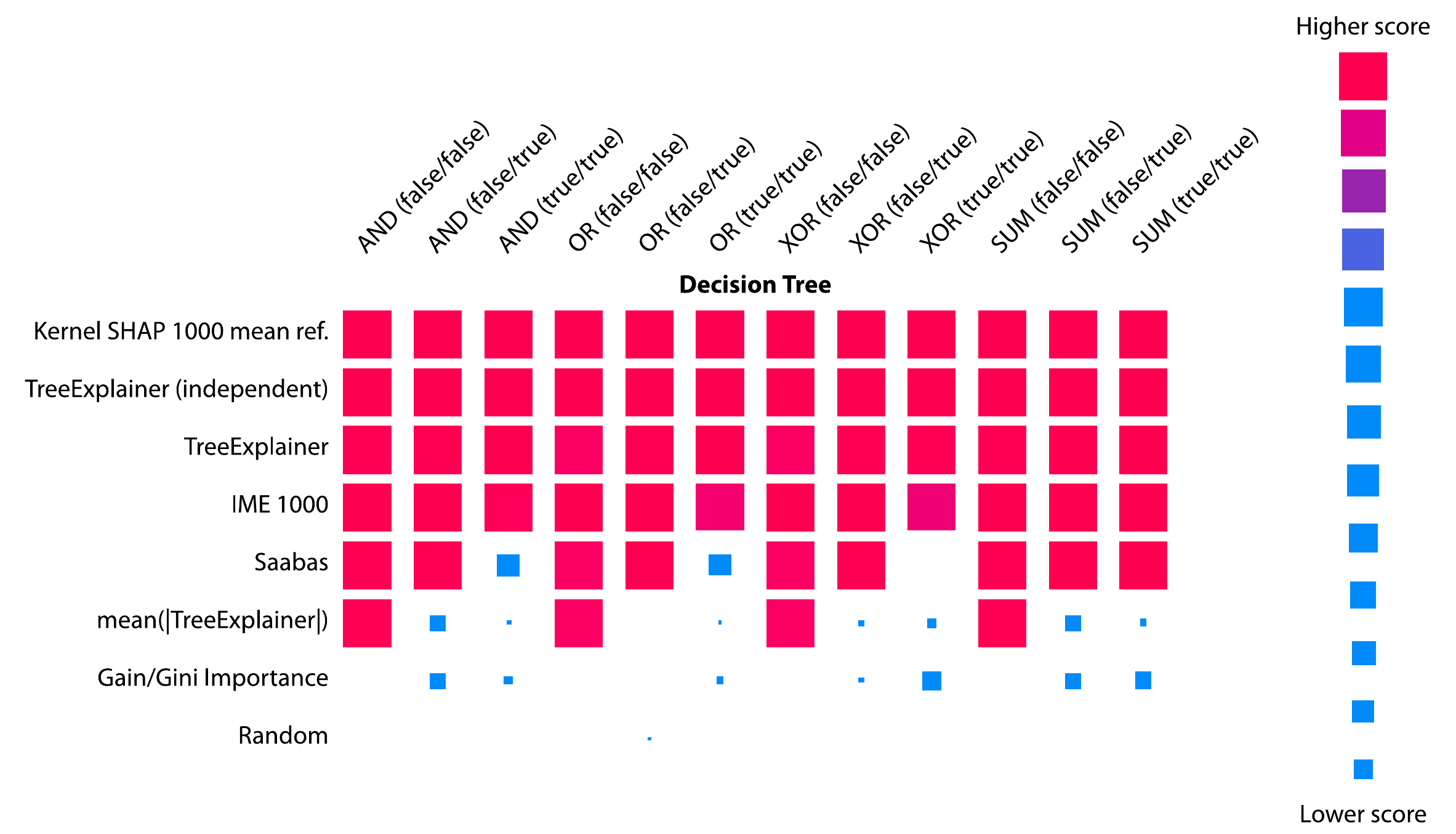}
  \caption{{\bf Shapley value based methods agree with human intuition.} We measured the most intuitive way to assign credit among input features by asking users to assign credit for predictions from four simple models (three predictions per model). The consensus allocation observed from a user study was then used as the ground truth and compared with the allocations given by different explanation methods. The sum of absolute differences was used to quantify agreement. All the Shapley value based methods had nearly perfect agreement across all the scenarios. The raw allocations for the cases where Saabas fails are shown in Supplementary Figures~\ref{fig:plot_human_decision_tree_human_and_11}-\ref{fig:plot_human_decision_tree_human_xor_11} (\ref{sec:methods_user_study}). Note that since these small (human understandable) models have only three features, model agnostic Shapley methods are accurate and so comparable with TreeExplainer.}
  \label{fig:human_tile}
\end{figure*}


\begin{figure*}
  \centering
  \includegraphics[width=1.0\textwidth]{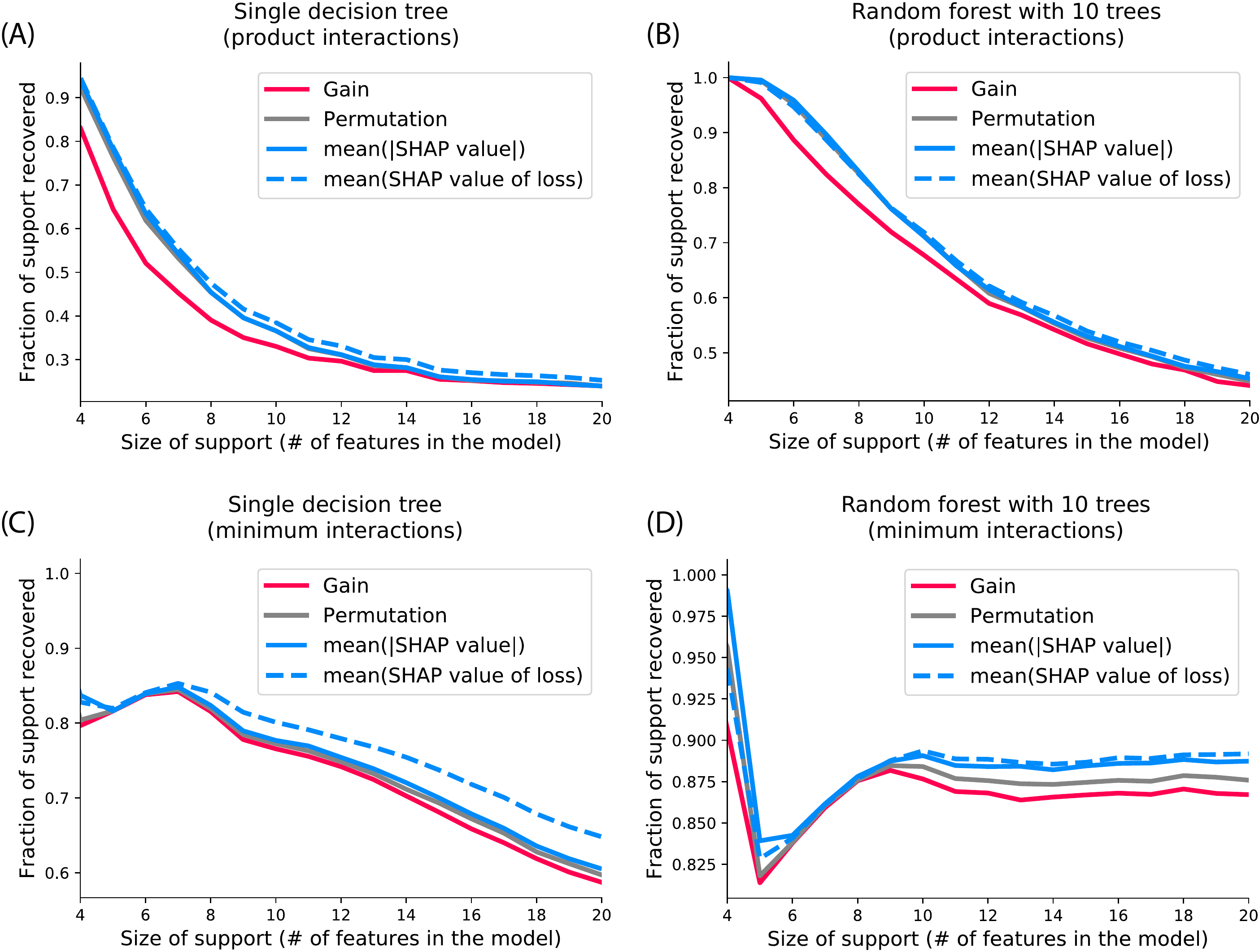}
  \caption{{\bf SHAP values can be combined to provide better feature selection power than traditional approaches.} Feature selection methods based on trees or ensembles of trees typically either use the total gain (reduction in train error when splitting) or a permutation test (scramble each feature and observe the change in error) to do feature selection. 
  To compare feature selection power we reproduce the same simulated independent features setting as \cite{kazemitabar2017variable} (but with strong 3rd order interactions) and compare the mean SHAP value of the loss vs. mean SHAP value magnitude vs. gain vs. permutation for feature selection. We also repeat the experiment replacing the product style interactions from \cite{kazemitabar2017variable} with minimum functions. For both a single tree (A,C) and an ensemble of ten trees (B,D) the SHAP values provide a better ranking of features. Perhaps because, unlike gain, they guarantee consistency as defined by Property~\ref{prop:consistency}, and unlike permutations, they account for high-order interaction effects. The x-axis of the plots represents the number of features used in the true model (out of 200 total features), while the y-axis represents the fraction of those true features recovered in the set of top ranked features of the same size. Results are averages over performance on 1000 simulated datasets. Both SHAP value based methods outperform gain and permutation testing in every figure, with all paired t-test P-values being $<10^{-7}$.}
  \label{fig:feature_selection}
\end{figure*}

\begin{figure*}
  \centering
  \includegraphics[width=0.6\textwidth]{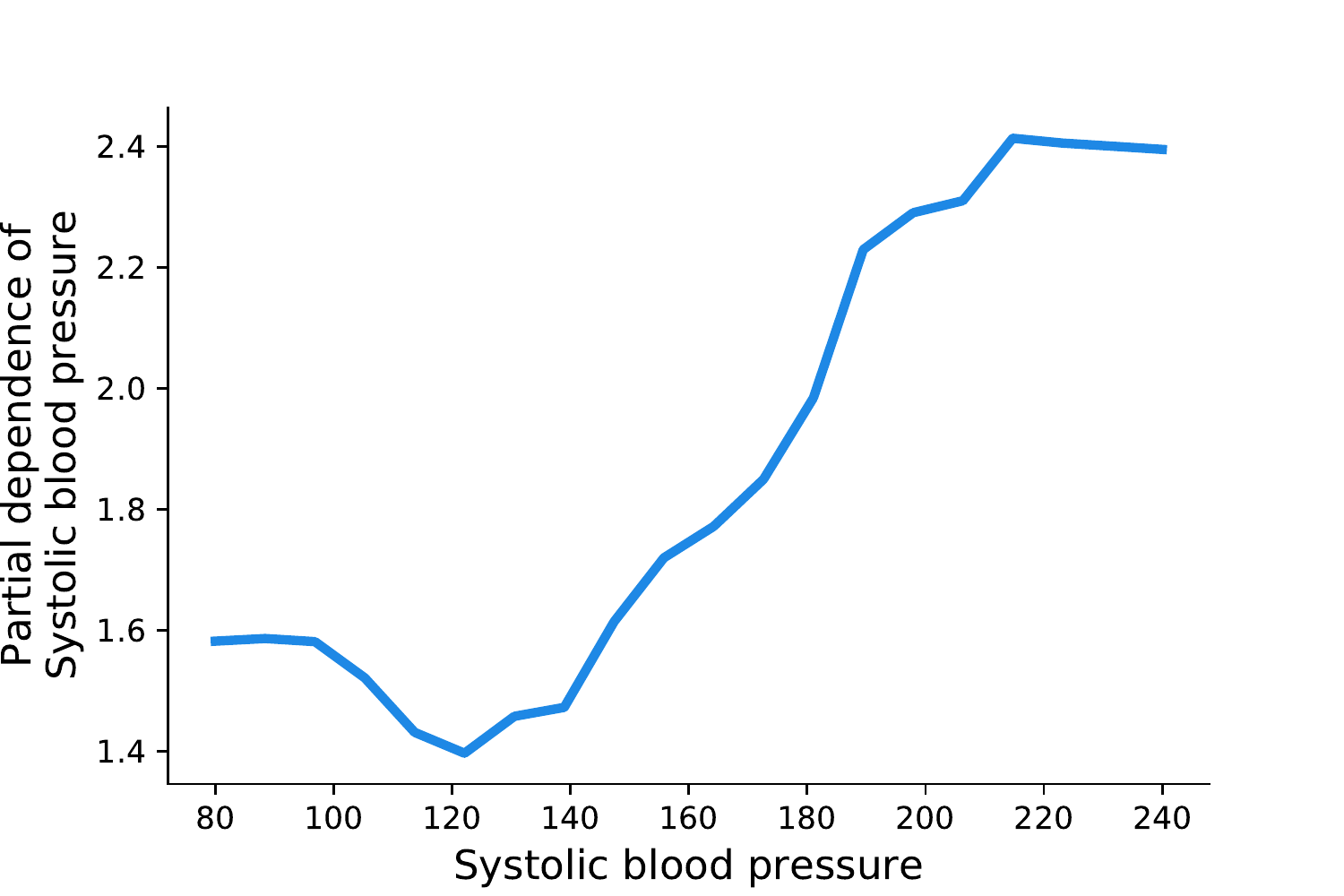}
  \caption{{\bf Partial dependence plot of systolic blood pressure in the mortality model.} Unlike the corresponding SHAP dependence plot in Figure~\ref{fig:overview_and_dependence}B, the partial dependence plot gives no indication of the heterogeneity between individuals caused by interaction effects in the model.}
  \label{fig:nhanes_pdp_sbp}
\end{figure*}

\begin{figure*}
  \centering
  \includegraphics[width=1.0\textwidth]{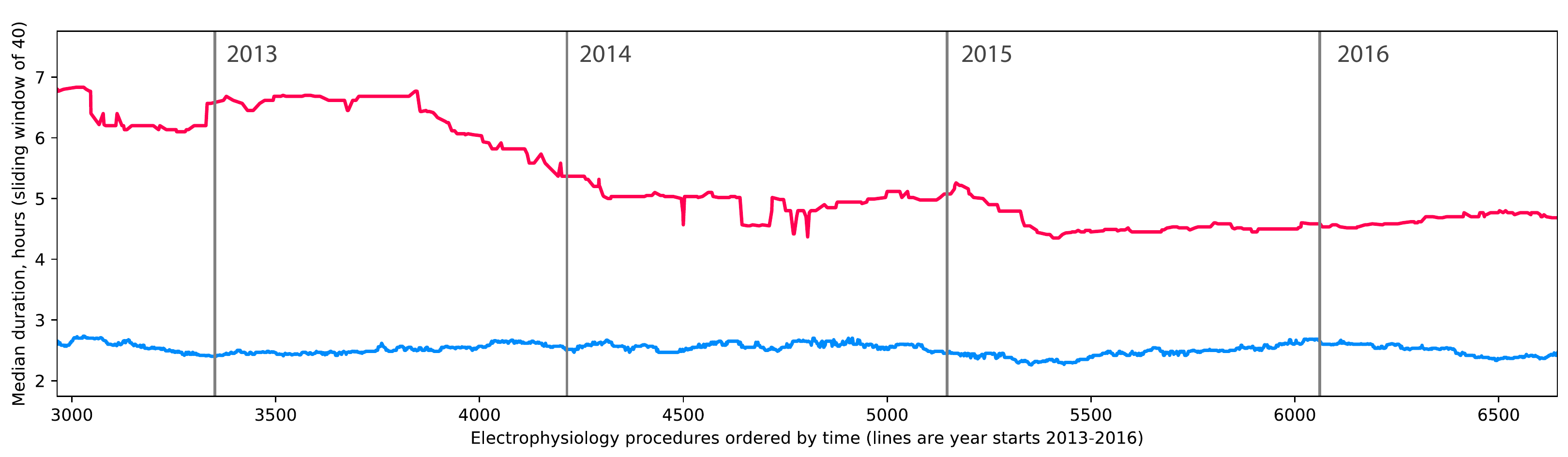}
  \caption{{\bf The average duration of ablation procedures for atrial fibrillation dropped significantly during 2014.} This data was obtained directly from the electrophysiology lab to diagnose why the atrial fibrillation feature was found (using a SHAP monitoring plot) to degrade model performance (Figure~\ref{fig:key_monitoring_plots}D).  The reason is that around 2014 (which was after the simulated model deployment) the duration of these procedures dropped significantly.}
  \label{fig:afib_duration_plot}
\end{figure*}

\begin{figure*}
  \centering
  \includegraphics[width=1.0\textwidth]{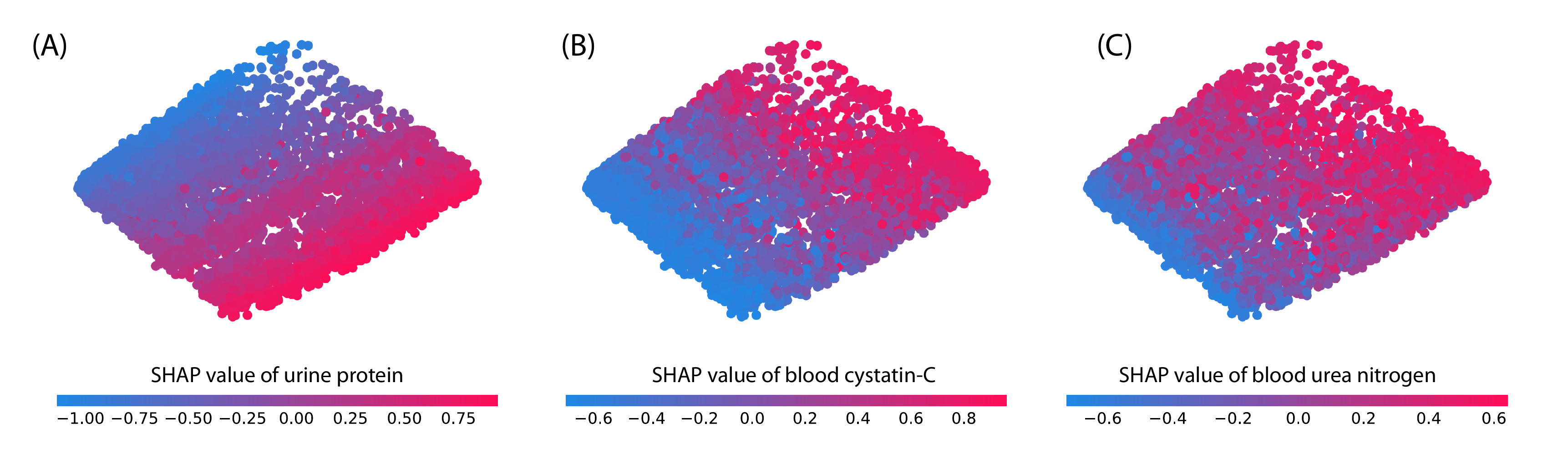}
  \caption{{\bf A local explanation embedding of kidney visits projected onto its first two principle components.} This shows the next three top features beyond those shown in Figures~\ref{fig:clustering_and_embedding}B-D. The fact that these features also align with the top principal components shows how many of the important features in the data set are capturing information along two largely orthogonal dimensions.}
  \label{fig:kidney_embedding2}
\end{figure*}

\begin{figure*}
  \centering
  \includegraphics[width=1.0\textwidth]{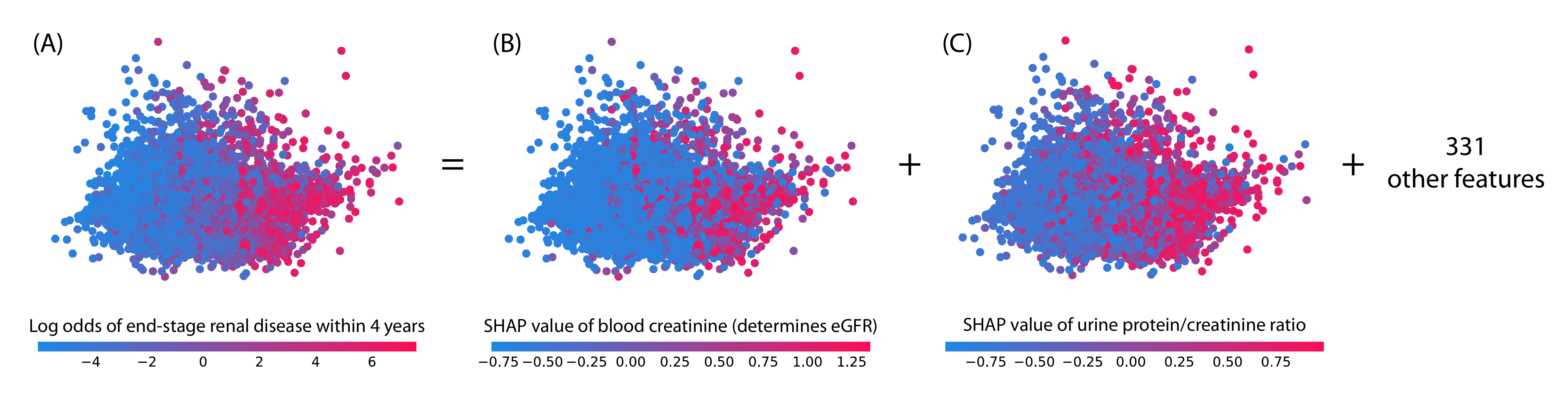}
  \caption{{\bf Principle component embedding of the chronic kidney disease dataset.} Unlike an embedding based in the local explanation space (Figures~\ref{fig:clustering_and_embedding}B-D), an unsupervised embedding of the data does not necessarily align with the outcome of interest in a dataset.}
  \label{fig:kidney_raw_pca}
\end{figure*}

\begin{figure*}
  \centering
  \includegraphics[width=1.0\textwidth]{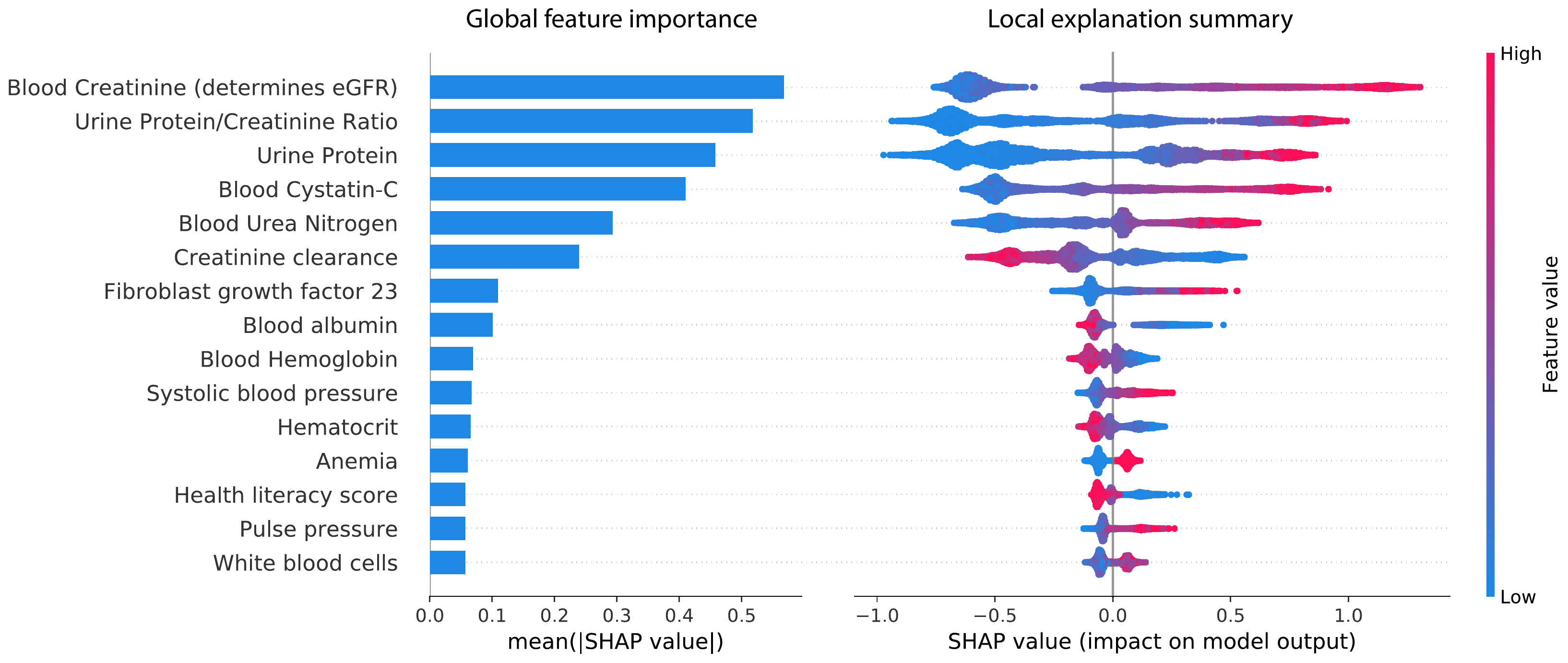}
  \caption{{\bf Bar chart (left) and summary plot (right) for a gradient boosted decision tree model trained on the chronic kidney disease data.} For the summary plot red dots indicate a high value of that feature for that individual, while blue dots represent a low feature value. The x-axis is the SHAP value of a feature for each individual's prediction, representing the change in the log-hazard ratio caused by observing that feature. High blood creatinine increases the risk of end-stage kidney disease. Conversely, low creatinine clearance increases the risk of end-stage kidney disease.}
  \label{fig:kidney_summary}
\end{figure*}

\begin{figure*}
  \centering
  \includegraphics[width=1.0\textwidth]{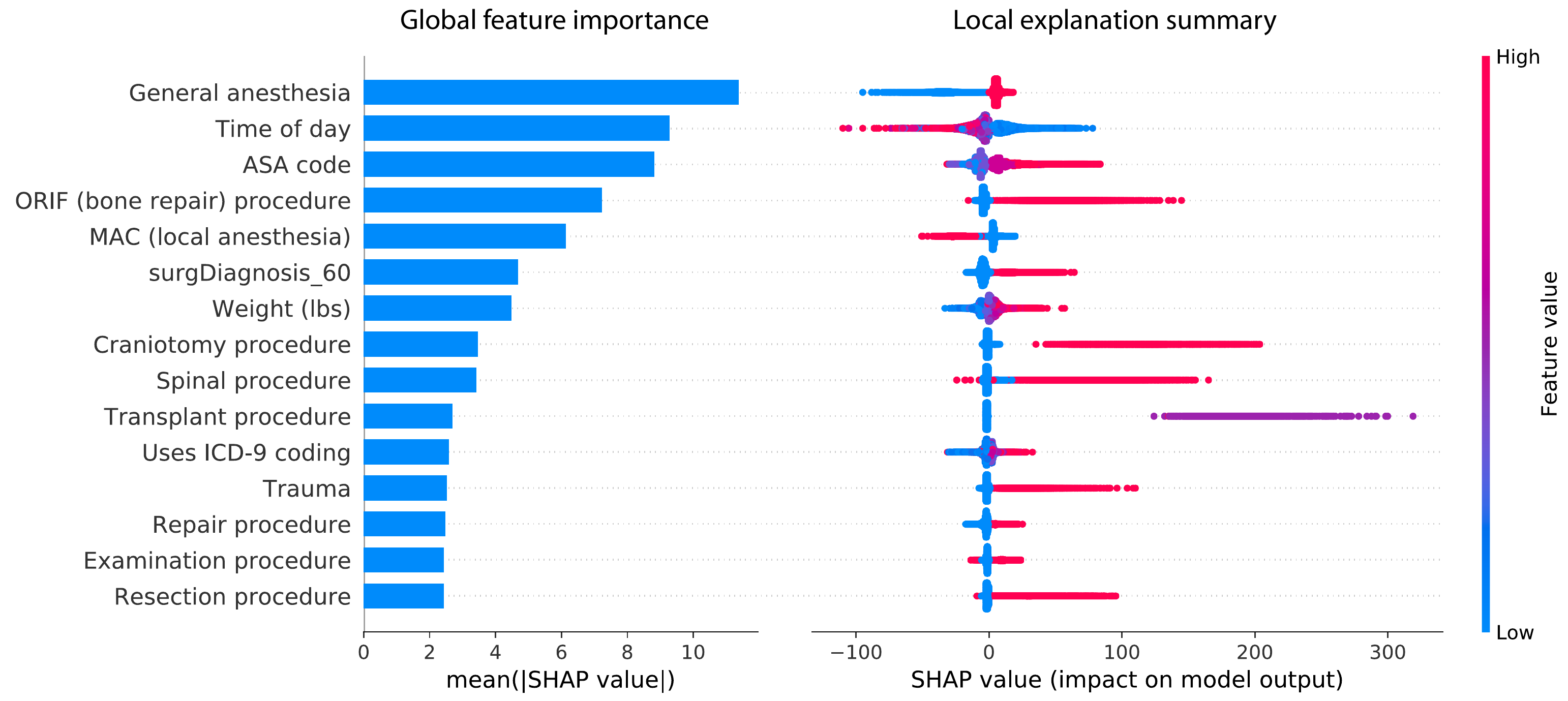}
  \caption{{\bf Bar chart (left) and summary plot (right) for a gradient boosted decision tree model trained on the hospital procedure duration data.} For the summary plot red dots indicate a high value of that feature for that individual, while blue dots represent a low feature value. The x-axis is the SHAP value of a feature for each individual's prediction, representing the change in the log-hazard ratio caused by observing that feature. Many of the features are bag-of-words counts that have only a few non-zero values. Because the model is very nonlinear, the impact of a flag being on (such as the trauma flag) can have very different effects for different procedures (as shown for trauma by the horizontal dispersion of the red dots).}
  \label{fig:hospital_summary}
\end{figure*}

\begin{figure*}
  \centering
  \includegraphics[width=1.0\textwidth]{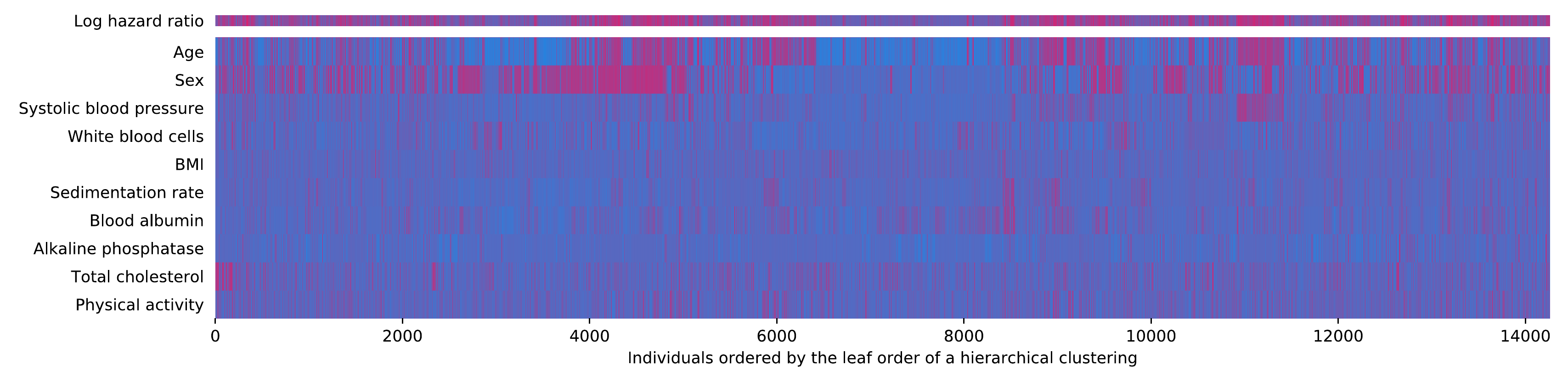}
  \caption{{\bf A clustering of 14,264 mortality study individuals by their normalized data values.} Standard complete linkage hierarchical clustering reveals fewer groups consistency with model risk than supervised clustering (Figure~\ref{fig:clustering_and_embedding}A    ). This is because unsupervised clustering has no bias towards clusters that share common risk characteristics. Row-normalized feature SHAP values are used for coloring, as in Figure~\ref{fig:clustering_and_embedding}A.}
  \label{fig:mortality_clustering_unsupervised}
\end{figure*}






\begin{figure*}
  \centering
  \includegraphics[width=1.0\textwidth]{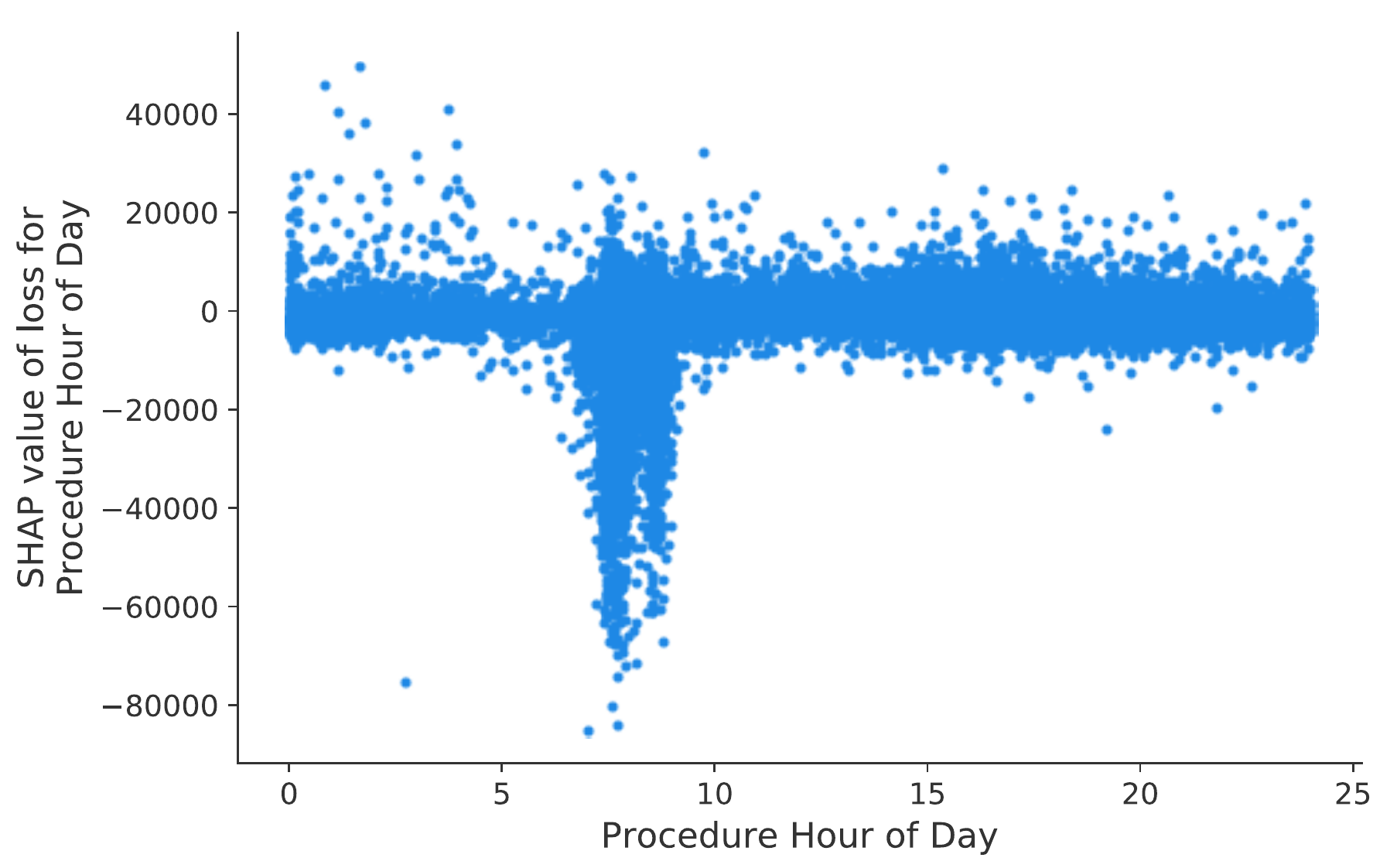}
  \caption{{\bf A dependence plot for time of day vs the SHAP value of time of day with respect to the model loss.} In our hospital system long running elective surgeries are scheduled at 7:20-7:30 AM on Monday/Tuesday/Thursday/Friday and at 8:20-8:30 AM on Wednesday. This plot shows that the primary way model is using time of day to reduce the model's loss is by detecting these surgery scheduling times. Each dot is a procedure, the x-axis is the time that procedure was scheduled to begin, and the y-axis is the impact knowing the time of day had on the model's loss for predicting that procedure's duration (lower is better).}
  \label{fig:hospital_time_of_day_loss_dependence}
\end{figure*}

\begin{figure*}
  \centering
  \includegraphics[width=1.0\textwidth]{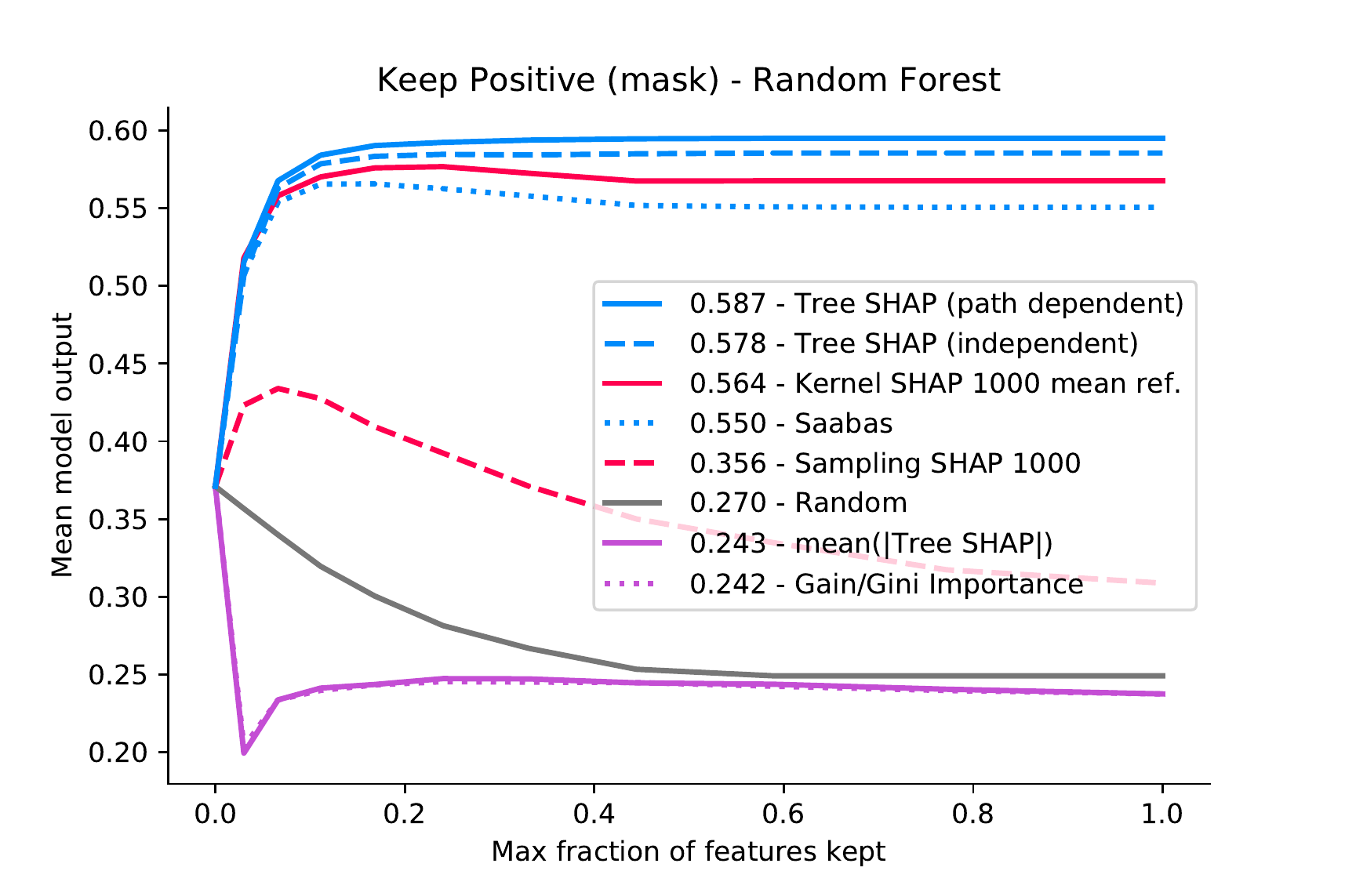}
  \caption{{\bf Keep positive (mask) metric for a random forest trained on the chronic kidney disease dataset.} Sorting the attribution values of an explanation method provides an ordering of the features for each prediction made by the model. Here we keep a fraction of the features ordered by how much they increase the model's output. Features that are not kept are masked with their mean value. If the ordering is good, as we include more and more features we push the model's output higher. Note that features with a negative contribution are always masked. The x-axis is the maximum fraction of features kept and the y-axis is the mean increase of the model over 100 predictions (averaged over 10 model's trained on different train/test splits of the dataset). Note that the Tree SHAP and Sampling SHAP algorithms correspond to TreeExplainer and IME \cite{vstrumbelj2014explaining}, respectively.} 
  \label{fig:plot_cric_random_forest_keep_positive_mask}
\end{figure*}

\begin{figure*}
  \centering
  \includegraphics[width=1.0\textwidth]{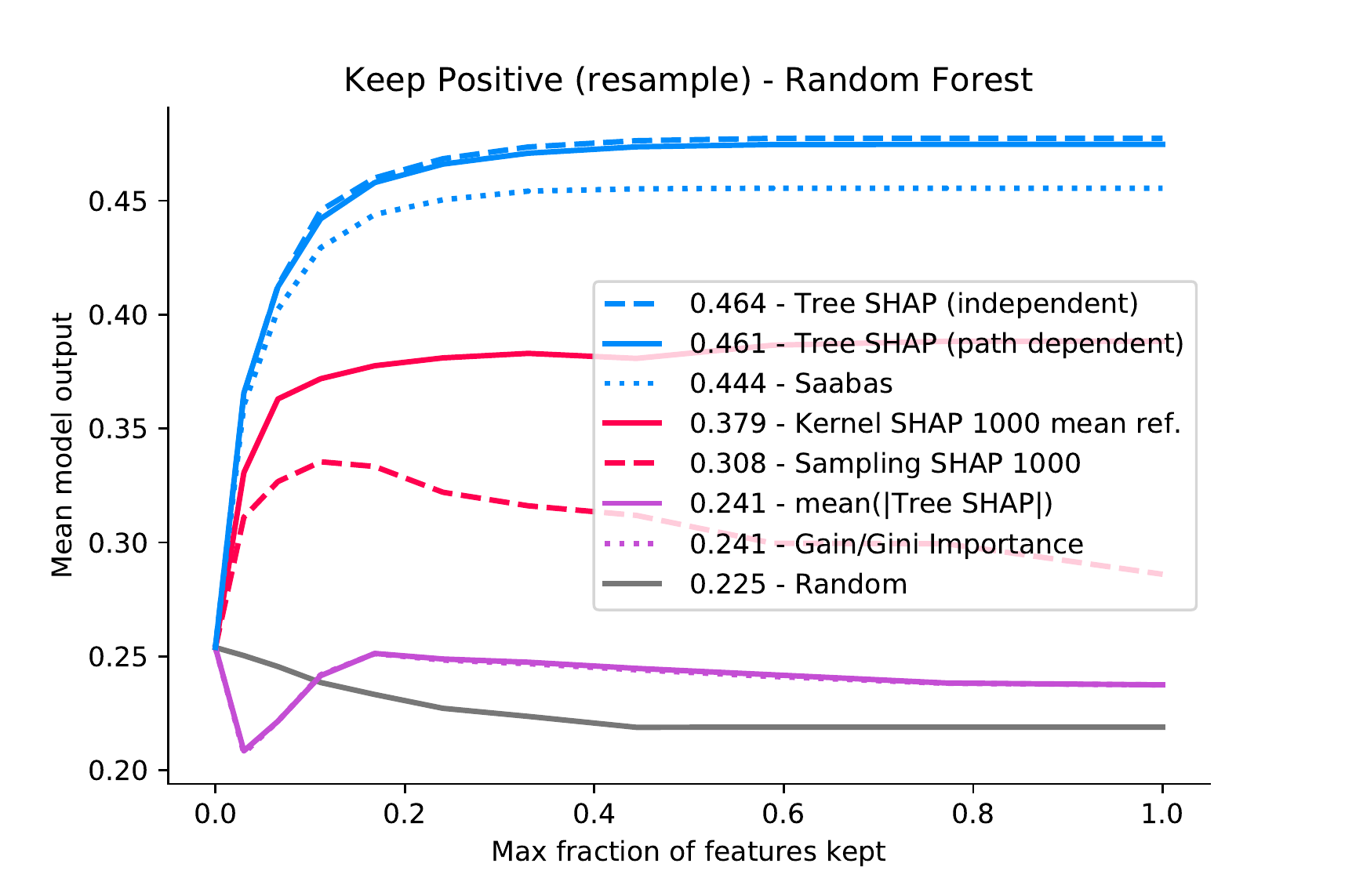}
  \caption{{\bf Keep positive (resample) metric for a random forest trained on the chronic kidney disease dataset.} This is just like Supplementary Figure~\ref{fig:plot_cric_random_forest_keep_positive_mask} except that instead of masking the hidden features with their mean value, we instead replace them with a random sample from the training dataset. This resampling process is averaged over 100 times to integrate over the distribution of background samples. If the input features are independent of one another then this effectively computes the conditional expectation of the model output conditioned on only the observed features. Note that the Tree SHAP and Sampling SHAP algorithms correspond to TreeExplainer and IME \cite{vstrumbelj2014explaining}, respectively.}
  \label{fig:plot_cric_random_forest_keep_positive_resample}
\end{figure*}

\begin{figure*}
  \centering
  \includegraphics[width=1.0\textwidth]{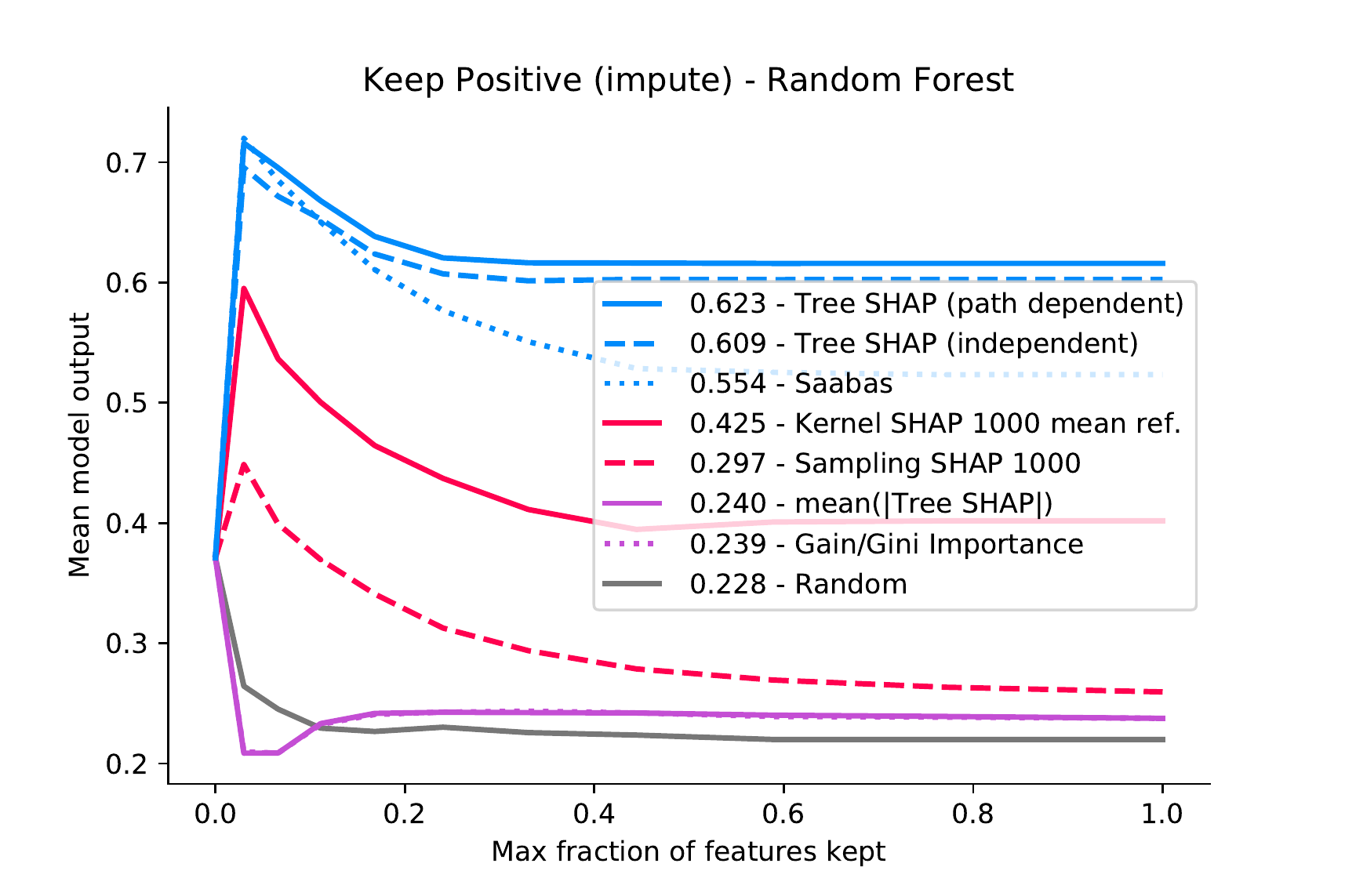}
  \caption{{\bf Keep positive (impute) metric for a random forest trained on the chronic kidney disease dataset.} This is just like Supplementary Figures~\ref{fig:plot_cric_random_forest_keep_positive_mask}~and~\ref{fig:plot_cric_random_forest_keep_positive_resample} except that instead of mean masking or resampling the hidden features, we instead impute them using the covariance matrix of the training data (this is maximum likelihood imputation if we assume the input features are multivariate normal). This imputation process seeks to avoid evaluating the model on unrealistic input data. In contrast with mean imputation or resampling we assume the input feature are independent and so many provide unrealistic combinations of input features (such as pregnant men). Note that the Tree SHAP and Sampling SHAP algorithms correspond to TreeExplainer and IME \cite{vstrumbelj2014explaining}, respectively.}
  \label{fig:plot_cric_random_forest_keep_positive_impute}
\end{figure*}

\begin{figure*}
  \centering
  \includegraphics[width=1.0\textwidth]{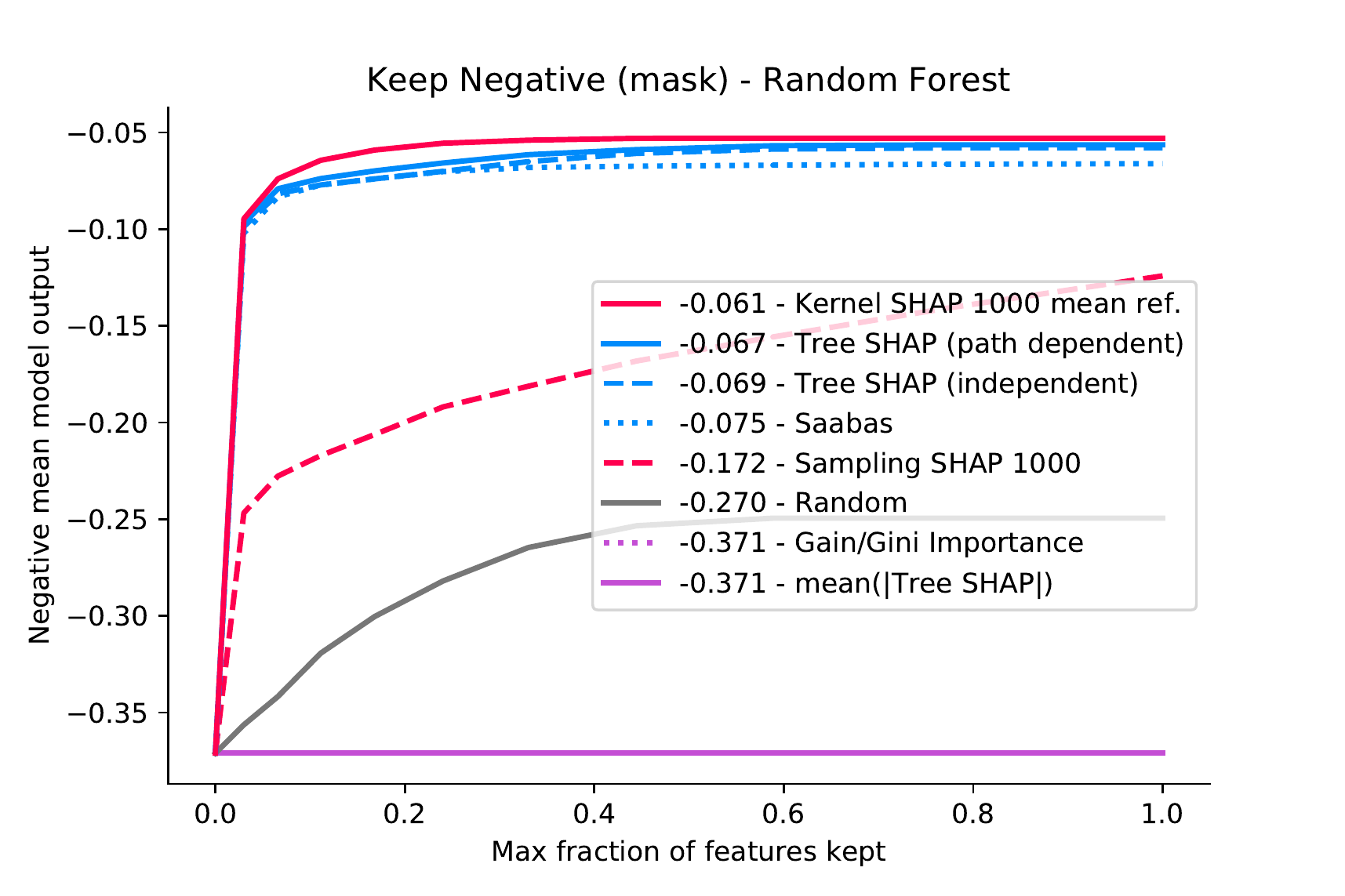}
  \caption{{\bf Keep negative (mask) metric for a random forest trained on the chronic kidney disease dataset.} This is just like Supplementary Figure~\ref{fig:plot_cric_random_forest_keep_positive_mask} except that we keep the most negative features instead of the most positive. Note that the Tree SHAP and Sampling SHAP algorithms correspond to TreeExplainer and IME \cite{vstrumbelj2014explaining}, respectively.}
  \label{fig:plot_cric_random_forest_keep_negative_mask}
\end{figure*}

\begin{figure*}
  \centering
  \includegraphics[width=1.0\textwidth]{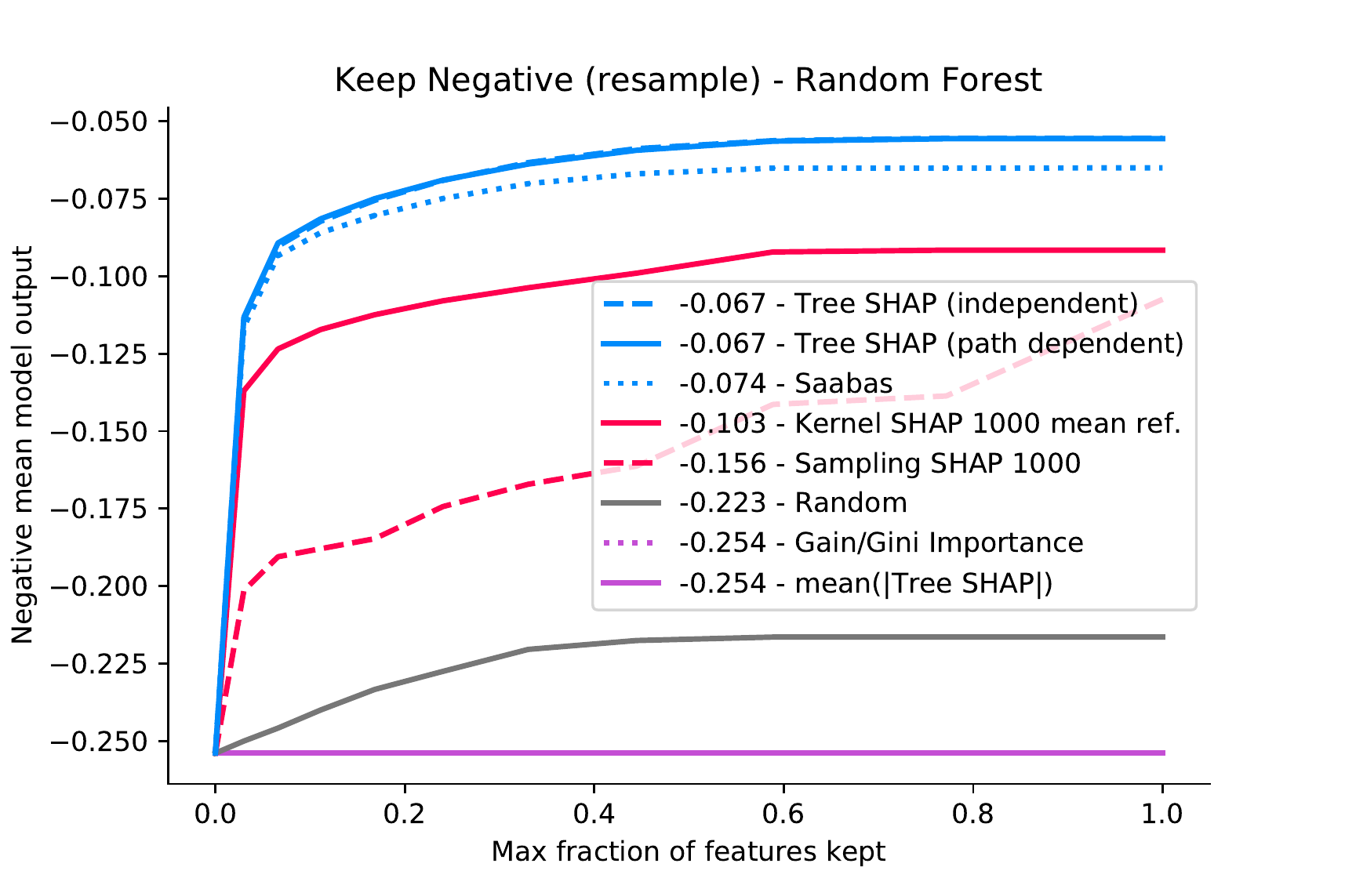}
  \caption{{\bf Keep negative (resample) metric for a random forest trained on the chronic kidney disease dataset.} This is just like Supplementary Figure~\ref{fig:plot_cric_random_forest_keep_positive_resample} except that we keep the most negative features instead of the most positive. Note that the Tree SHAP and Sampling SHAP algorithms correspond to TreeExplainer and IME \cite{vstrumbelj2014explaining}, respectively.}
  \label{fig:plot_cric_random_forest_keep_negative_resample}
\end{figure*}

\begin{figure*}
  \centering
  \includegraphics[width=1.0\textwidth]{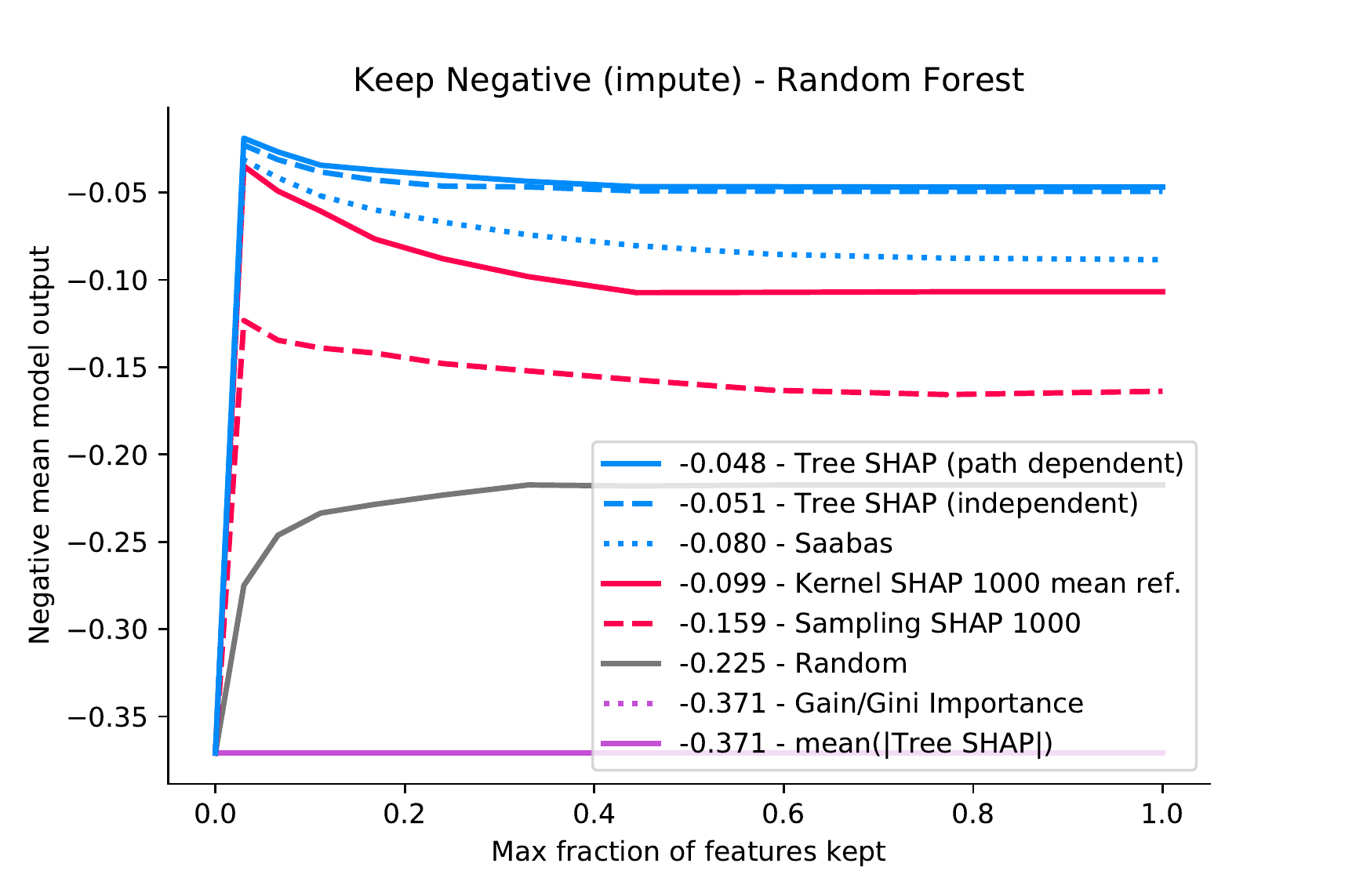}
  \caption{{\bf Keep negative (impute) metric for a random forest trained on the chronic kidney disease dataset.} This is just like Supplementary Figure~\ref{fig:plot_cric_random_forest_keep_positive_impute} except that we keep the most negative features instead of the most positive. Note that the Tree SHAP and Sampling SHAP algorithms correspond to TreeExplainer and IME \cite{vstrumbelj2014explaining}, respectively.}
  \label{fig:plot_cric_random_forest_keep_negative_impute}
\end{figure*}

\begin{figure*}
  \centering
  \includegraphics[width=1.0\textwidth]{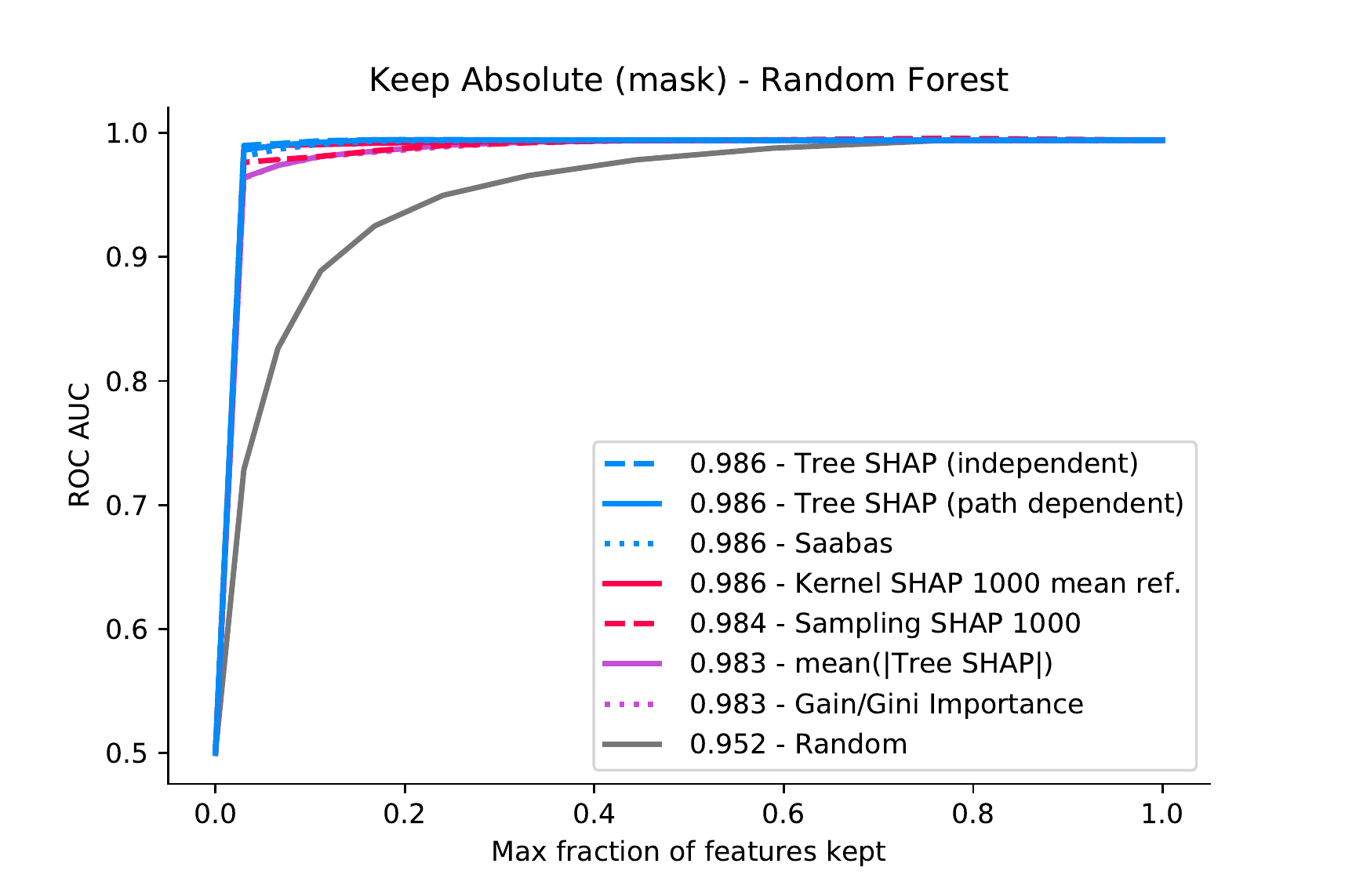}
  \caption{{\bf Keep absolute (mask) metric for a random forest trained on the chronic kidney disease dataset.} This is just like Supplementary Figures~\ref{fig:plot_cric_random_forest_keep_positive_mask}~and~\ref{fig:plot_cric_random_forest_keep_negative_mask} except that we keep the most important features by absolute value instead of the most positive or negative. Since this no longer specifically pushes the model output higher or lower, we instead measure the accuracy of the model. Good attribution methods will identify important features that when kept will result in better model accuracy, measured in this case by the area under the receiver operating characteristic (ROC) curve. Note that the Tree SHAP and Sampling SHAP algorithms correspond to TreeExplainer and IME \cite{vstrumbelj2014explaining}, respectively.}
  \label{fig:plot_cric_random_forest_keep_absolute_mask__roc_auc}
\end{figure*}

\begin{figure*}
  \centering
  \includegraphics[width=1.0\textwidth]{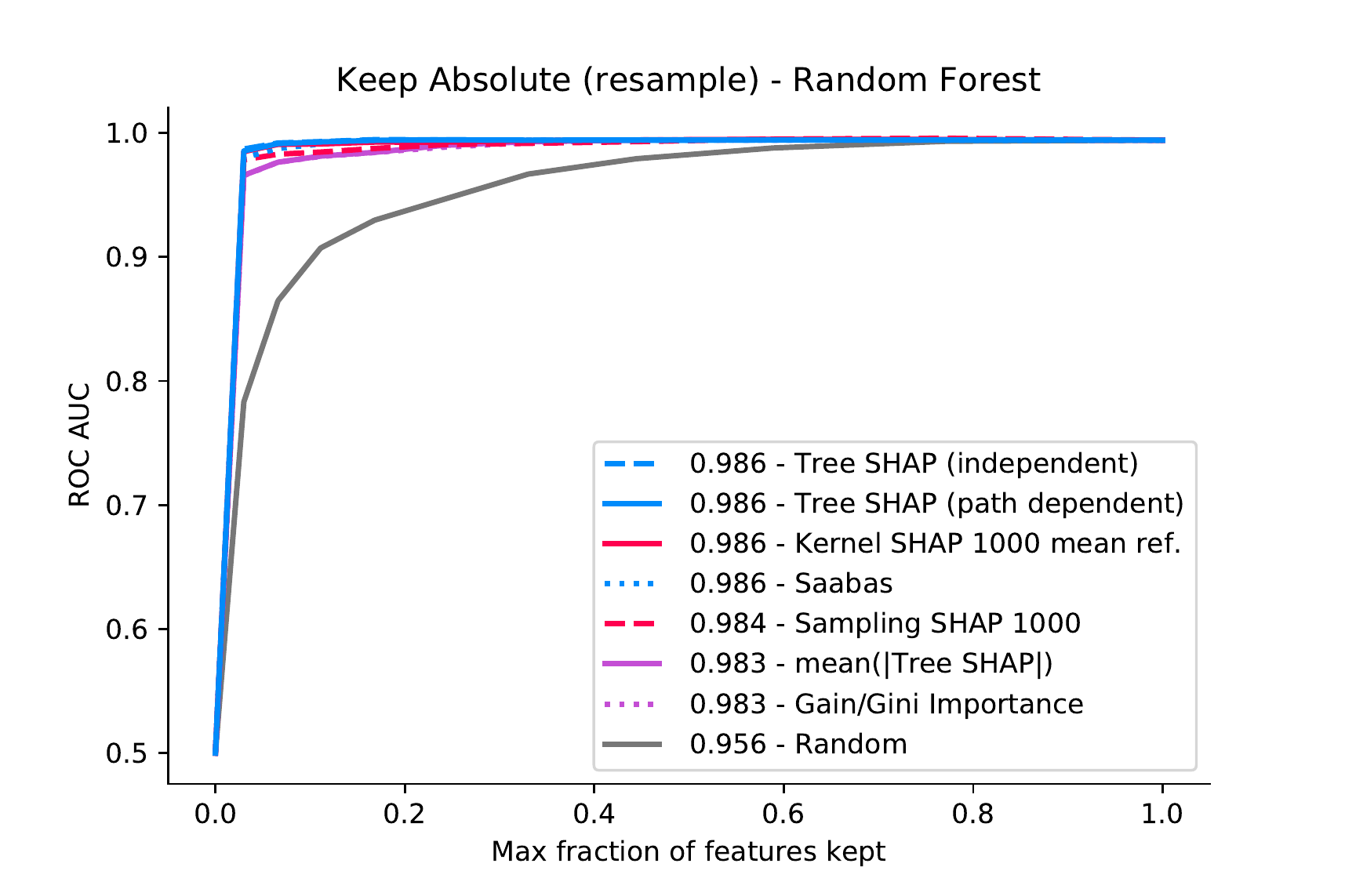}
  \caption{{\bf Keep absolute (resample) metric for a random forest trained on the chronic kidney disease dataset.} This is just like Supplementary Figures~\ref{fig:plot_cric_random_forest_keep_positive_resample}~and~\ref{fig:plot_cric_random_forest_keep_negative_resample} except that we keep the most important features by absolute value instead of the most positive or negative (as in Figure~\ref{fig:plot_cric_random_forest_keep_absolute_mask__roc_auc}). Note that the Tree SHAP and Sampling SHAP algorithms correspond to TreeExplainer and IME \cite{vstrumbelj2014explaining}, respectively.}
  \label{fig:plot_cric_random_forest_keep_absolute_resample__roc_auc}
\end{figure*}

\begin{figure*}
  \centering
  \includegraphics[width=1.0\textwidth]{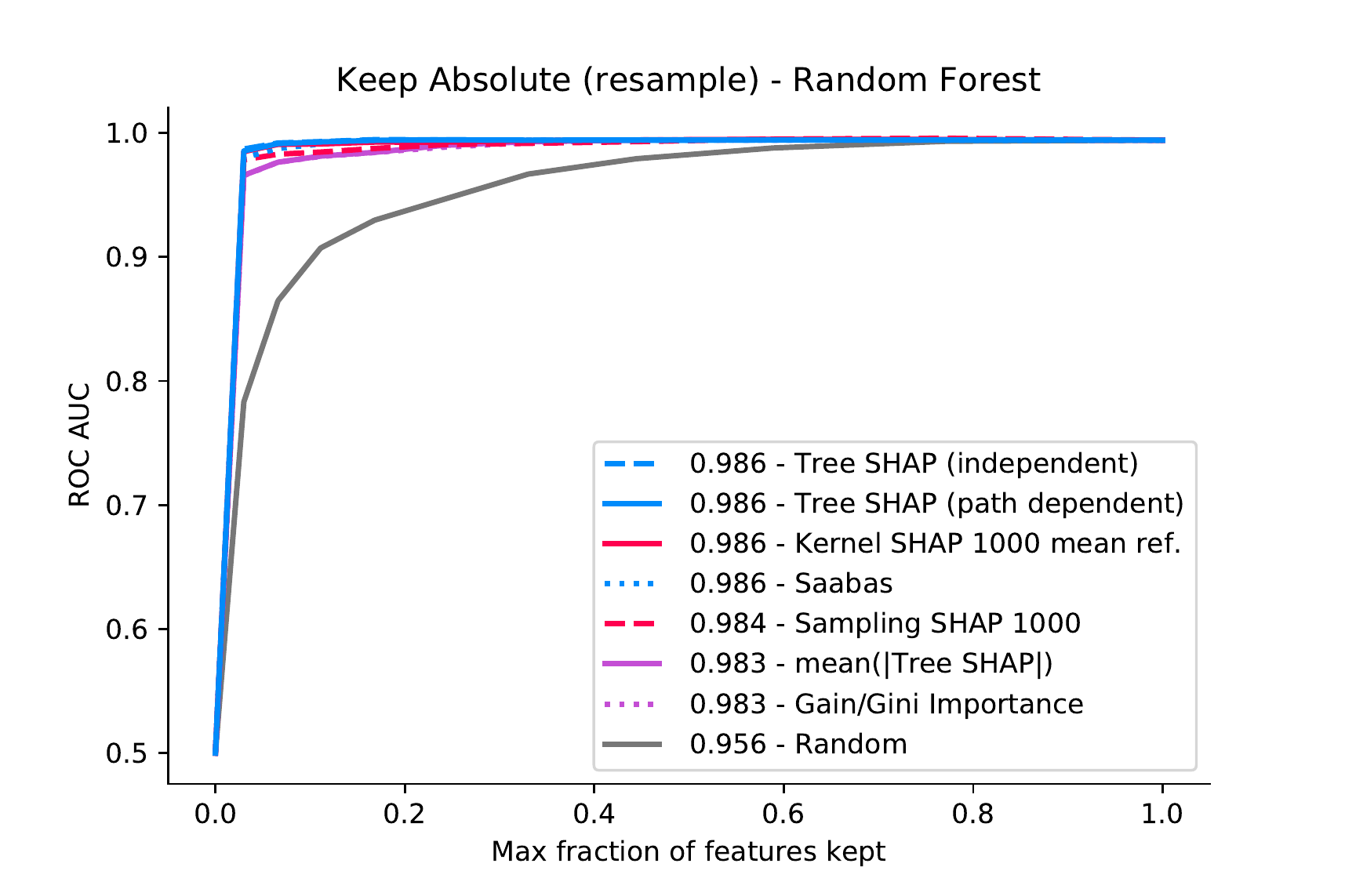}
  \caption{{\bf Keep absolute (impute) metric for a random forest trained on the chronic kidney disease dataset.} This is just like Supplementary Figures~\ref{fig:plot_cric_random_forest_keep_positive_impute}~and~\ref{fig:plot_cric_random_forest_keep_negative_impute} except that we keep the most important features by absolute value instead of the most positive or negative (as in Figure~\ref{fig:plot_cric_random_forest_keep_absolute_mask__roc_auc}). Note that the Tree SHAP and Sampling SHAP algorithms correspond to TreeExplainer and IME \cite{vstrumbelj2014explaining}, respectively.}
  \label{fig:plot_cric_random_forest_keep_absolute_impute__roc_auc}
\end{figure*}

\begin{figure*}
  \centering
  \includegraphics[width=1.0\textwidth]{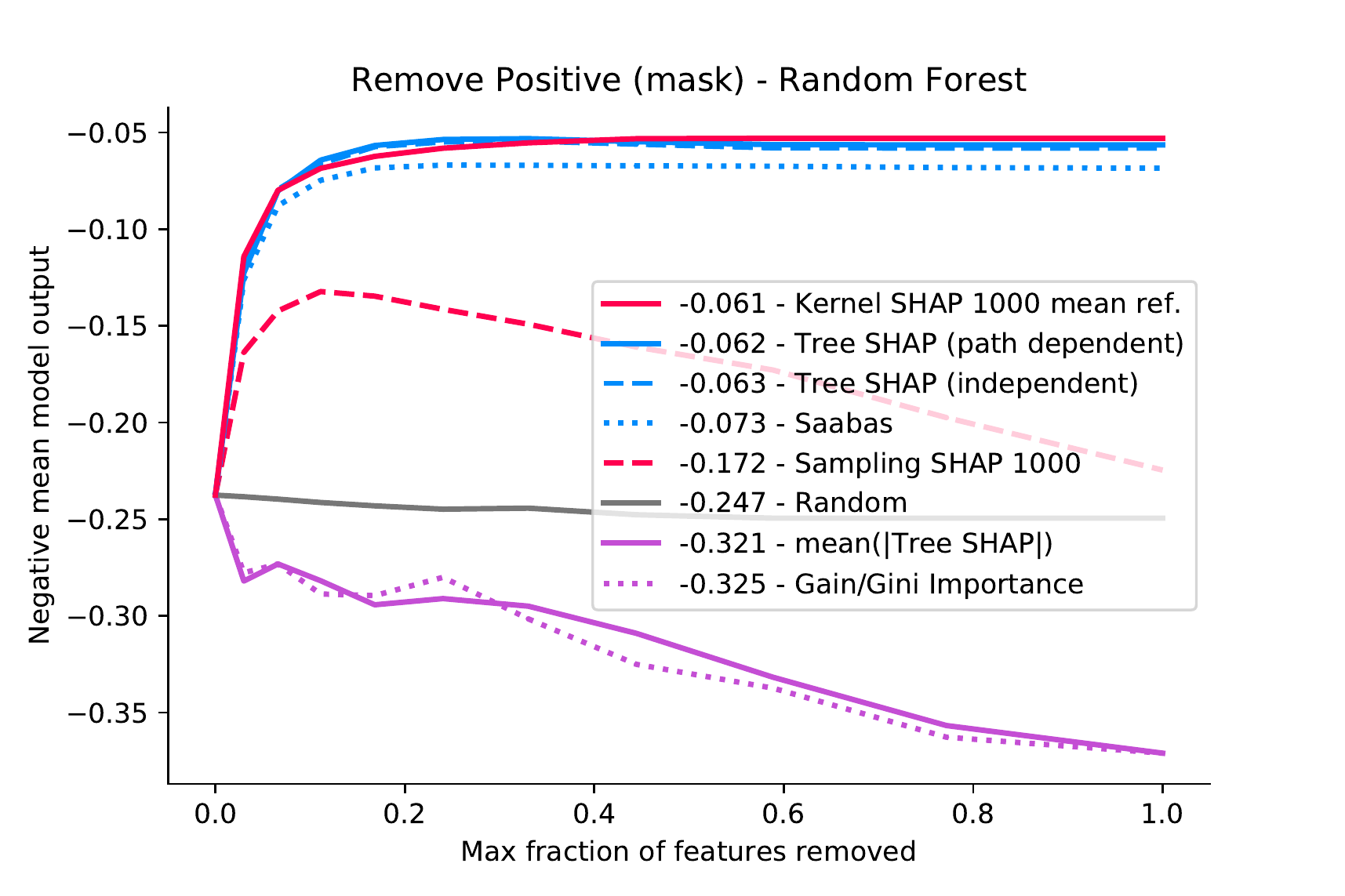}
  \caption{{\bf Remove positive (mask) metric for a random forest trained on the chronic kidney disease dataset.} This is just like Supplementary Figure~\ref{fig:plot_cric_random_forest_keep_positive_mask} except that we remove the most positive features instead of keeping them. Note that the Tree SHAP and Sampling SHAP algorithms correspond to TreeExplainer and IME \cite{vstrumbelj2014explaining}, respectively.}
  \label{fig:plot_cric_random_forest_remove_positive_mask}
\end{figure*}

\begin{figure*}
  \centering
  \includegraphics[width=1.0\textwidth]{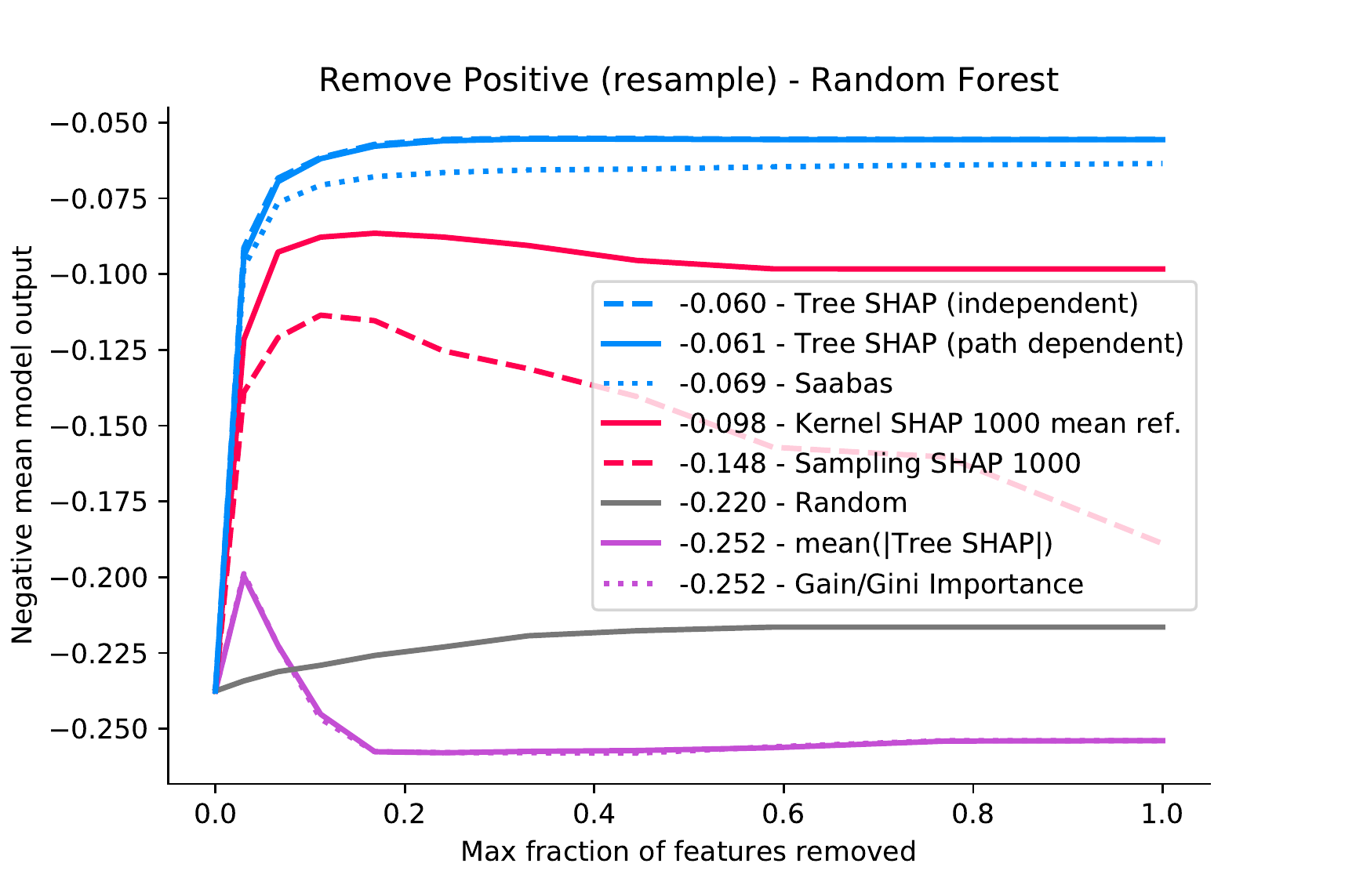}
  \caption{{\bf Remove positive (resample) metric for a random forest trained on the chronic kidney disease dataset.} This is just like Supplementary Figure~\ref{fig:plot_cric_random_forest_keep_positive_resample} except that we remove the most positive features instead of keeping them. Note that the Tree SHAP and Sampling SHAP algorithms correspond to TreeExplainer and IME \cite{vstrumbelj2014explaining}, respectively.}
  \label{fig:plot_cric_random_forest_remove_positive_resample}
\end{figure*}

\begin{figure*}
  \centering
  \includegraphics[width=1.0\textwidth]{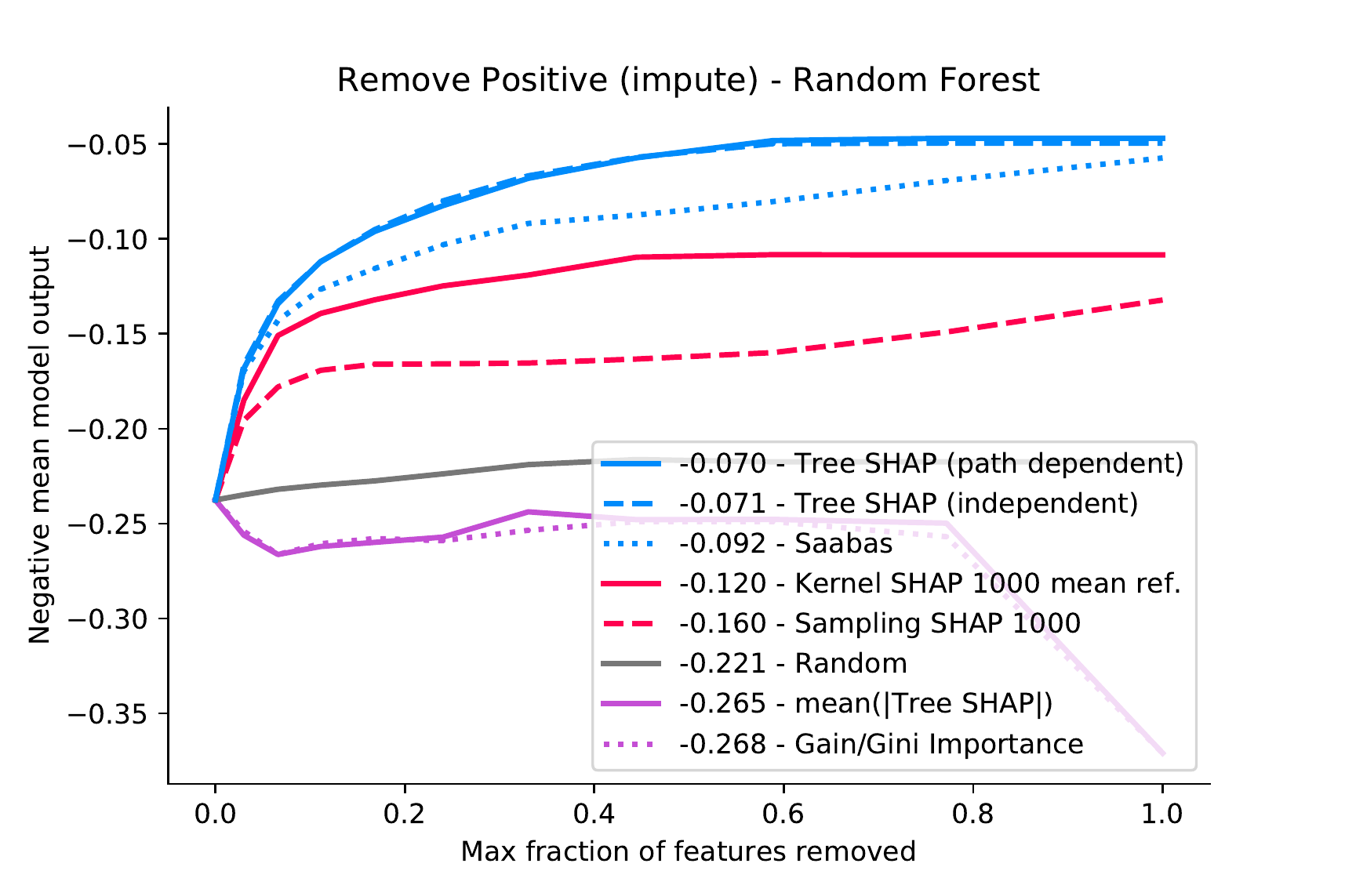}
  \caption{{\bf Remove positive (impute) metric for a random forest trained on the chronic kidney disease dataset.} This is just like Supplementary Figure~\ref{fig:plot_cric_random_forest_keep_positive_impute} except that we remove the most positive features instead of keeping them. Note that the Tree SHAP and Sampling SHAP algorithms correspond to TreeExplainer and IME \cite{vstrumbelj2014explaining}, respectively.}
  \label{fig:plot_cric_random_forest_remove_positive_impute}
\end{figure*}

\begin{figure*}
  \centering
  \includegraphics[width=1.0\textwidth]{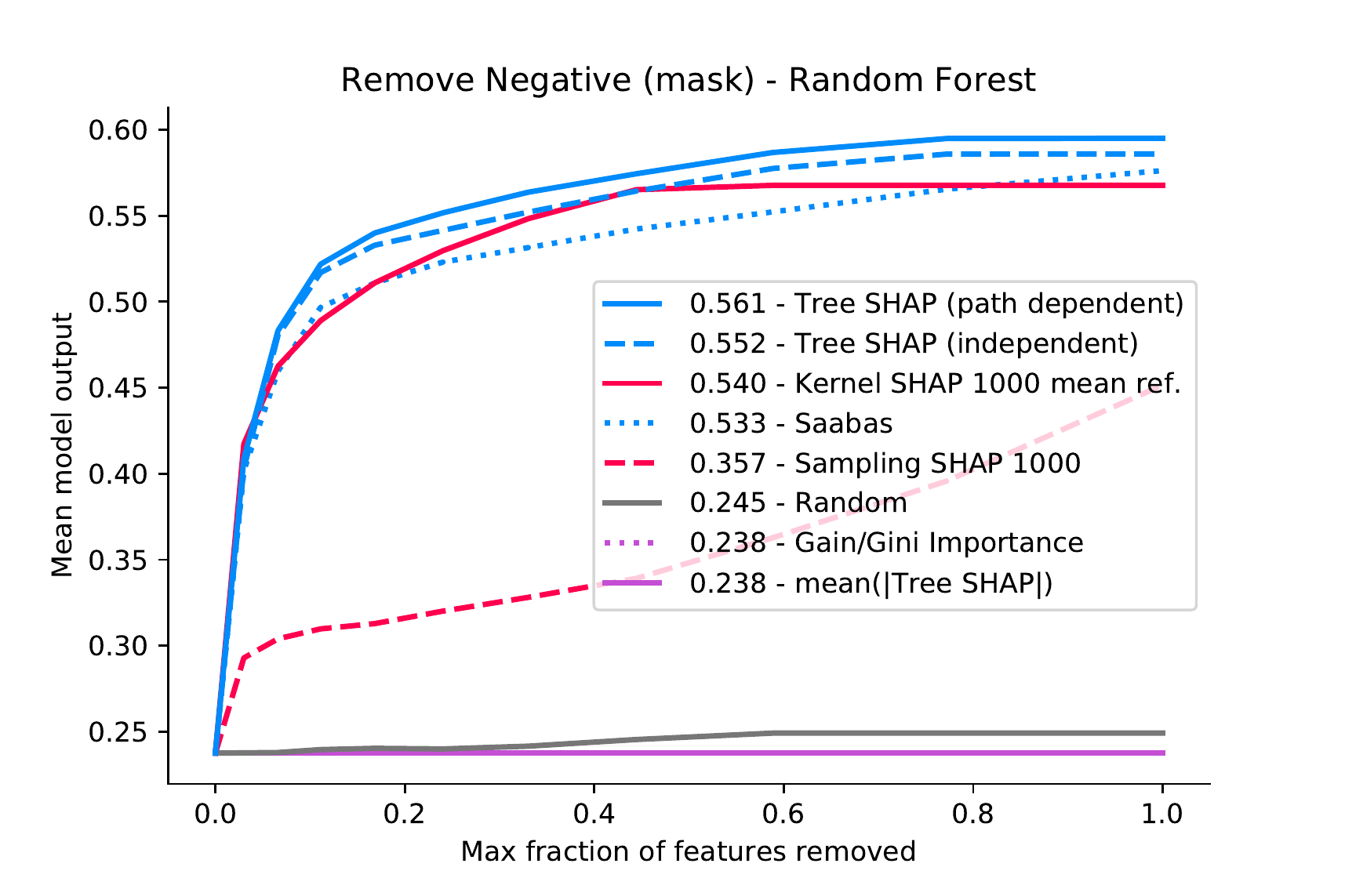}
  \caption{{\bf Remove negative (mask) metric for a random forest trained on the chronic kidney disease dataset.} This is just like Supplementary Figure~\ref{fig:plot_cric_random_forest_keep_negative_mask} except that we remove the most negative features instead of keeping them. Note that the Tree SHAP and Sampling SHAP algorithms correspond to TreeExplainer and IME \cite{vstrumbelj2014explaining}, respectively.}
  \label{fig:plot_cric_random_forest_remove_negative_mask}
\end{figure*}

\begin{figure*}
  \centering
  \includegraphics[width=1.0\textwidth]{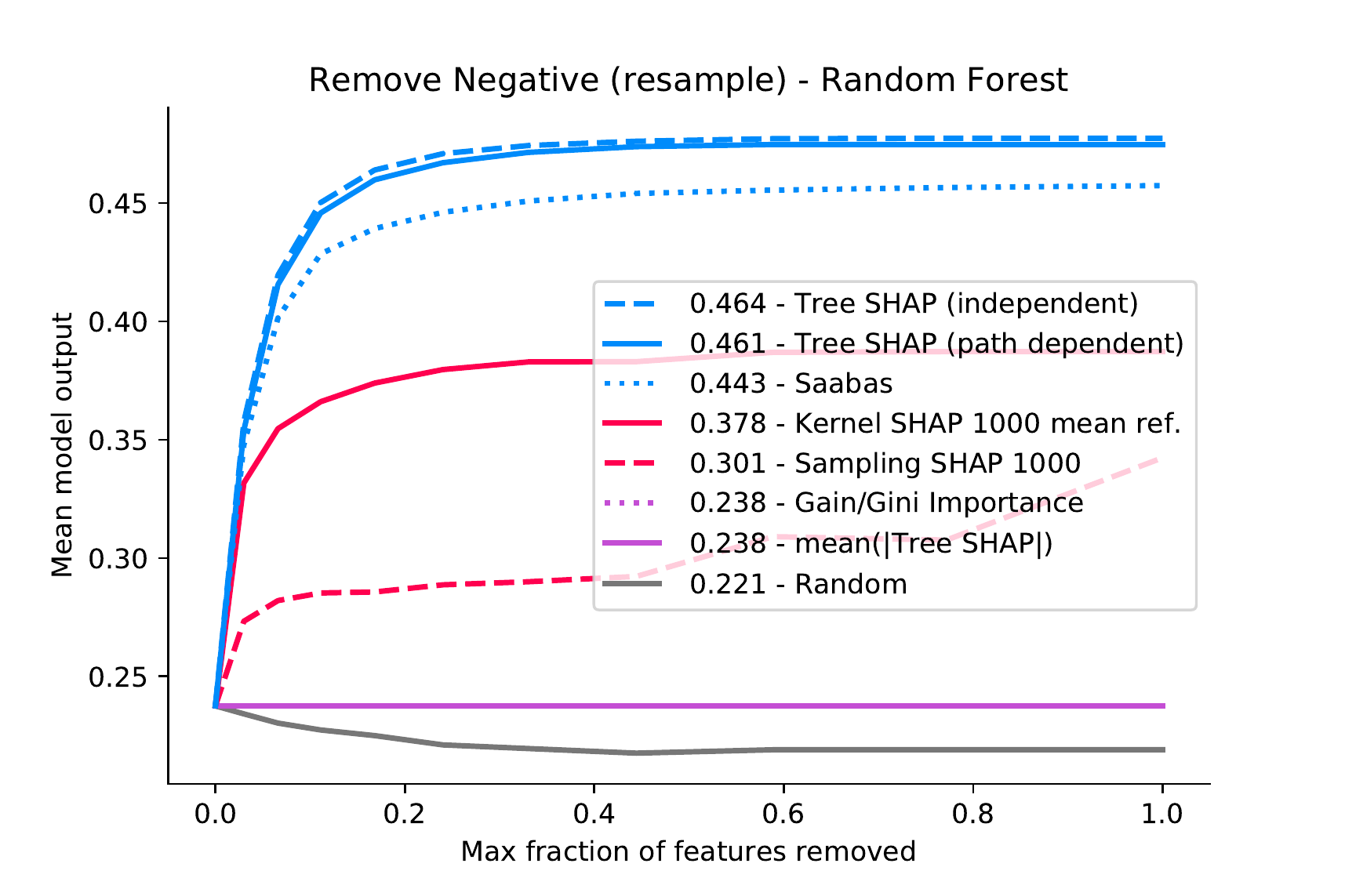}
  \caption{{\bf Remove negative (resample) metric for a random forest trained on the chronic kidney disease dataset.} This is just like Supplementary Figure~\ref{fig:plot_cric_random_forest_keep_negative_resample} except that we remove the most negative features instead of keeping them. Note that the Tree SHAP and Sampling SHAP algorithms correspond to TreeExplainer and IME \cite{vstrumbelj2014explaining}, respectively.}
  \label{fig:plot_cric_random_forest_remove_negative_resample}
\end{figure*}

\begin{figure*}
  \centering
  \includegraphics[width=1.0\textwidth]{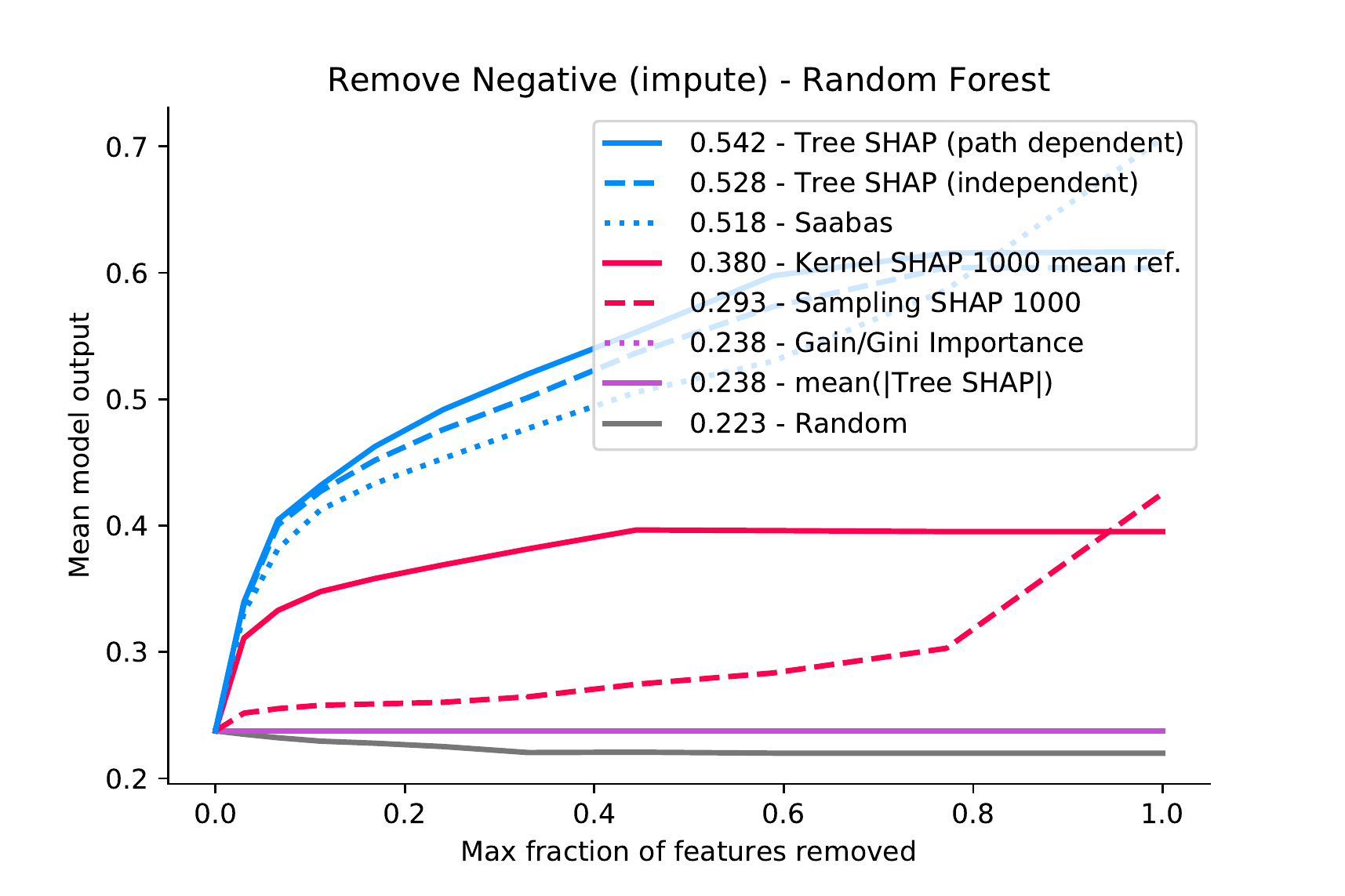}
  \caption{{\bf Remove negative (impute) metric for a random forest trained on the chronic kidney disease dataset.} This is just like Supplementary Figure~\ref{fig:plot_cric_random_forest_keep_negative_impute} except that we remove the most negative features instead of keeping them. Note that the Tree SHAP and Sampling SHAP algorithms correspond to TreeExplainer and IME \cite{vstrumbelj2014explaining}, respectively.}
  \label{fig:plot_cric_random_forest_remove_negative_impute}
\end{figure*}

\begin{figure*}
  \centering
  \includegraphics[width=1.0\textwidth]{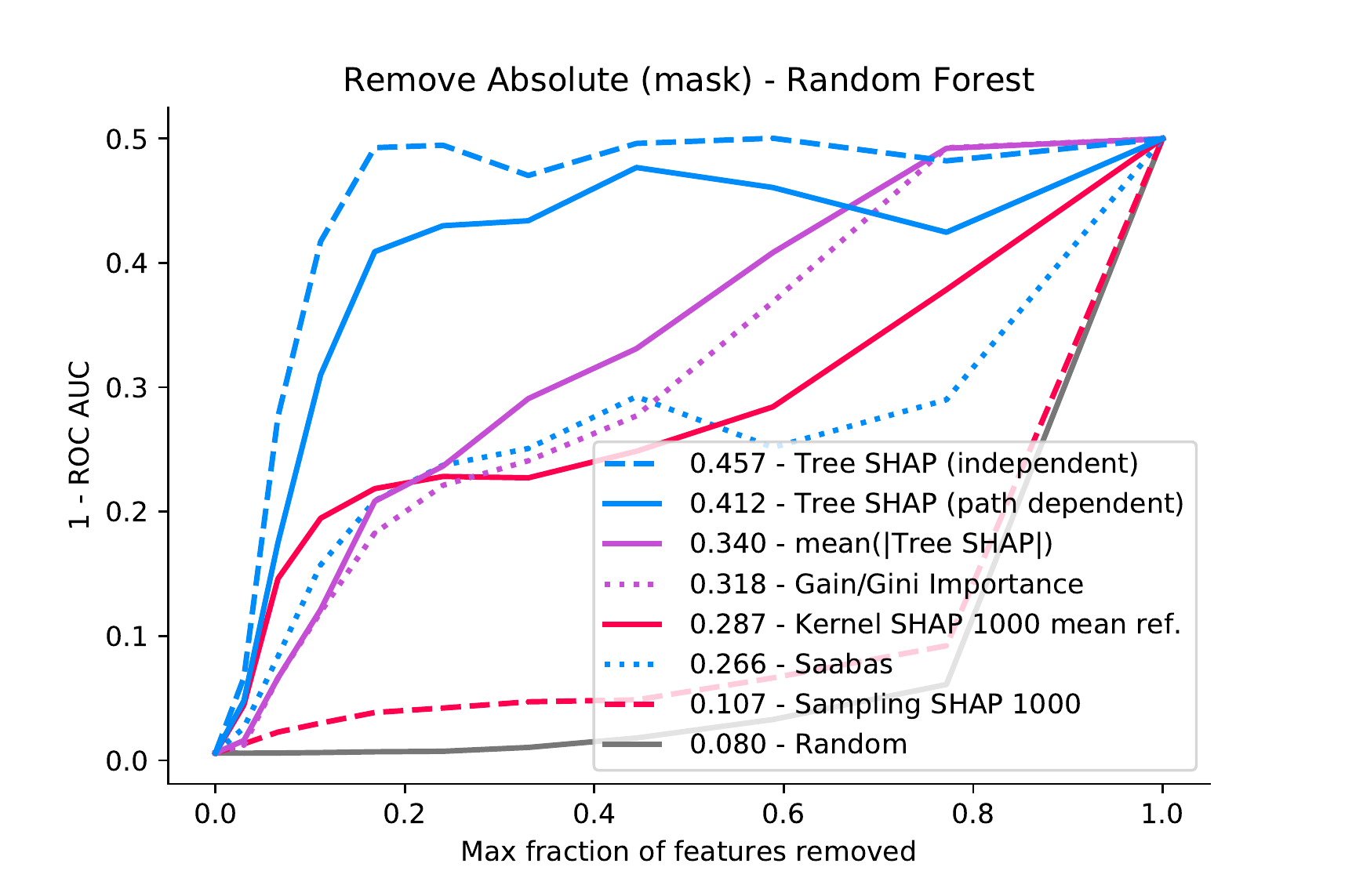}
  \caption{{\bf Remove absolute (mask) metric for a random forest trained on the chronic kidney disease dataset.} This is just like Supplementary Figure~\ref{fig:plot_cric_random_forest_keep_absolute_mask__roc_auc} except that we remove the most important features instead of keeping them. Note that the Tree SHAP and Sampling SHAP algorithms correspond to TreeExplainer and IME \cite{vstrumbelj2014explaining}, respectively.}
  \label{fig:plot_cric_random_forest_remove_absolute_mask__roc_auc}
\end{figure*}

\begin{figure*}
  \centering
  \includegraphics[width=1.0\textwidth]{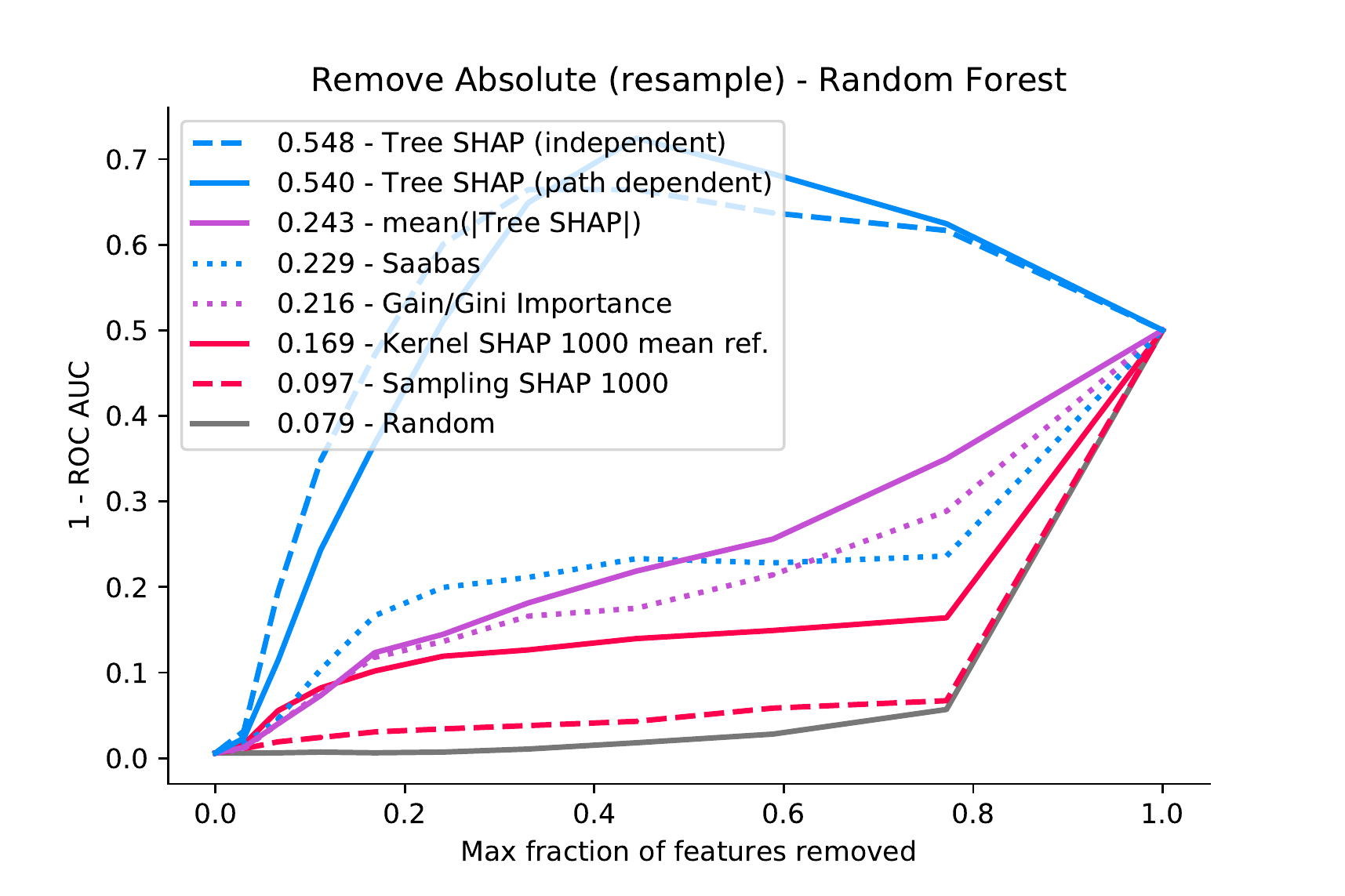}
  \caption{{\bf Remove absolute (resample) metric for a random forest trained on the chronic kidney disease dataset.} This is just like Supplementary Figure~\ref{fig:plot_cric_random_forest_keep_absolute_resample__roc_auc} except that we remove the most important features instead of keeping them. Note that the Tree SHAP and Sampling SHAP algorithms correspond to TreeExplainer and IME \cite{vstrumbelj2014explaining}, respectively.}
  \label{fig:plot_cric_random_forest_remove_absolute_resample__roc_auc}
\end{figure*}

\begin{figure*}
  \centering
  \includegraphics[width=1.0\textwidth]{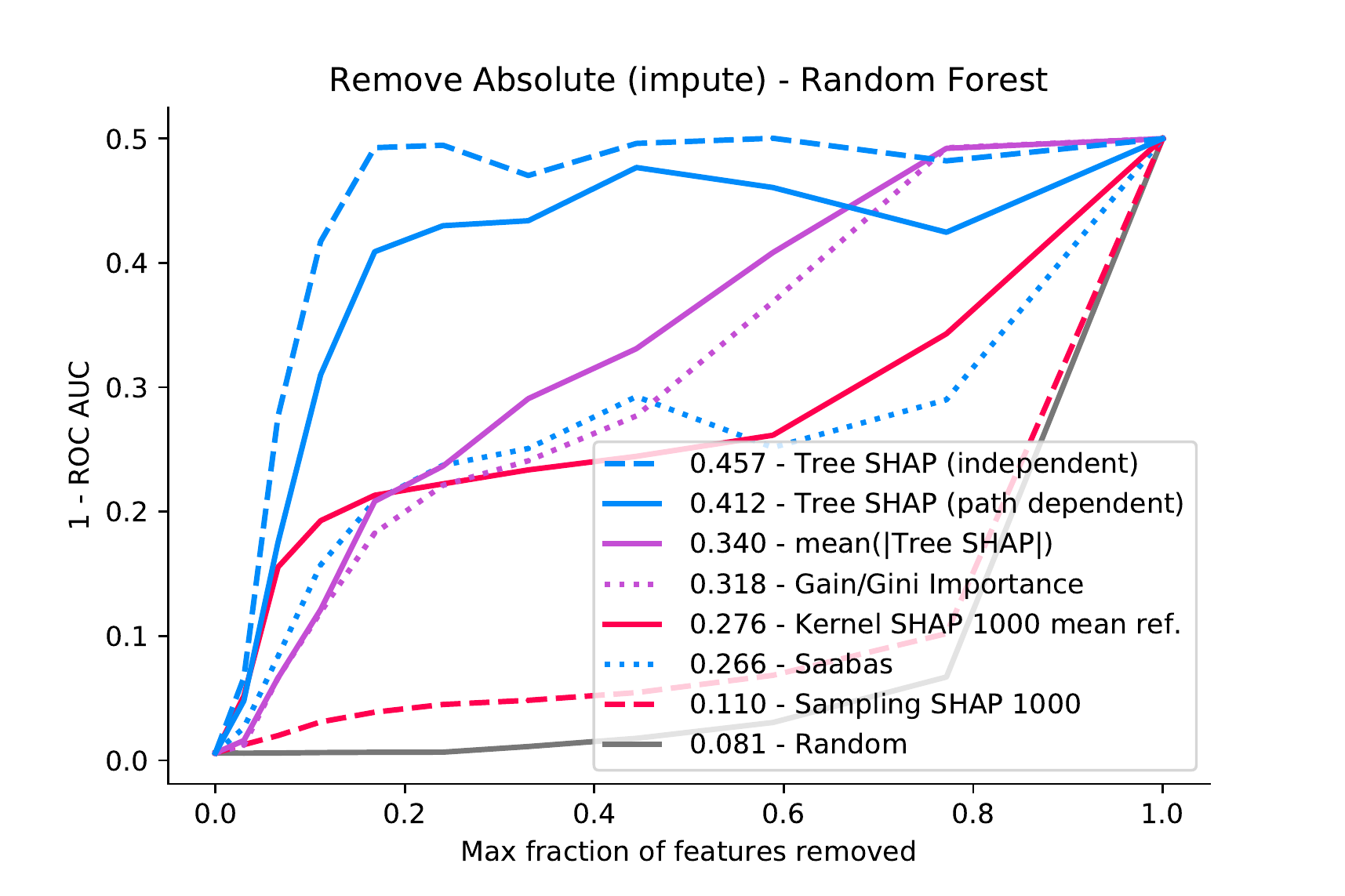}
  \caption{{\bf Remove absolute (impute) metric for a random forest trained on the chronic kidney disease dataset.} This is just like Supplementary Figure~\ref{fig:plot_cric_random_forest_keep_absolute_impute__roc_auc} except that we remove the most important features instead of keeping them. Note that the Tree SHAP and Sampling SHAP algorithms correspond to TreeExplainer and IME \cite{vstrumbelj2014explaining}, respectively.}
  \label{fig:plot_cric_random_forest_remove_absolute_impute__roc_auc}
\end{figure*}

\begin{figure*}
  \centering
  \includegraphics[width=1.0\textwidth]{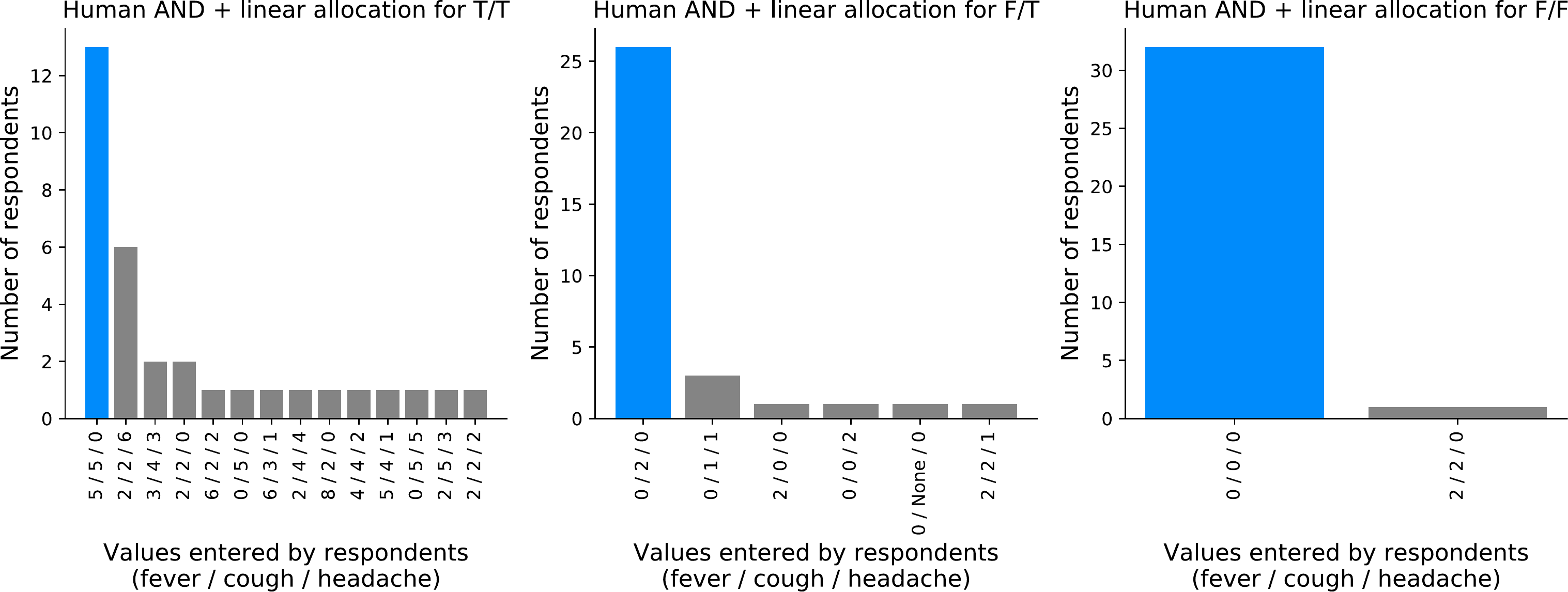}
  \caption{{\bf Consensus human intuition values for an AND function.} Human consensus values were measured by a user study over 33 participants for a simple AND-based function. The most popular allocation was chosen as the consensus (\ref{sec:methods_user_study}). The labels denote the allocation given by people as ``fever / cough / headache''. The title gives the input values for sample being explained as ``fever value/cough value'', where `T' is true and `F' is false; note that headache is always set to true.}
  \label{fig:human_and_survey}
\end{figure*}

\begin{figure*}
  \centering
  \includegraphics[width=1.0\textwidth]{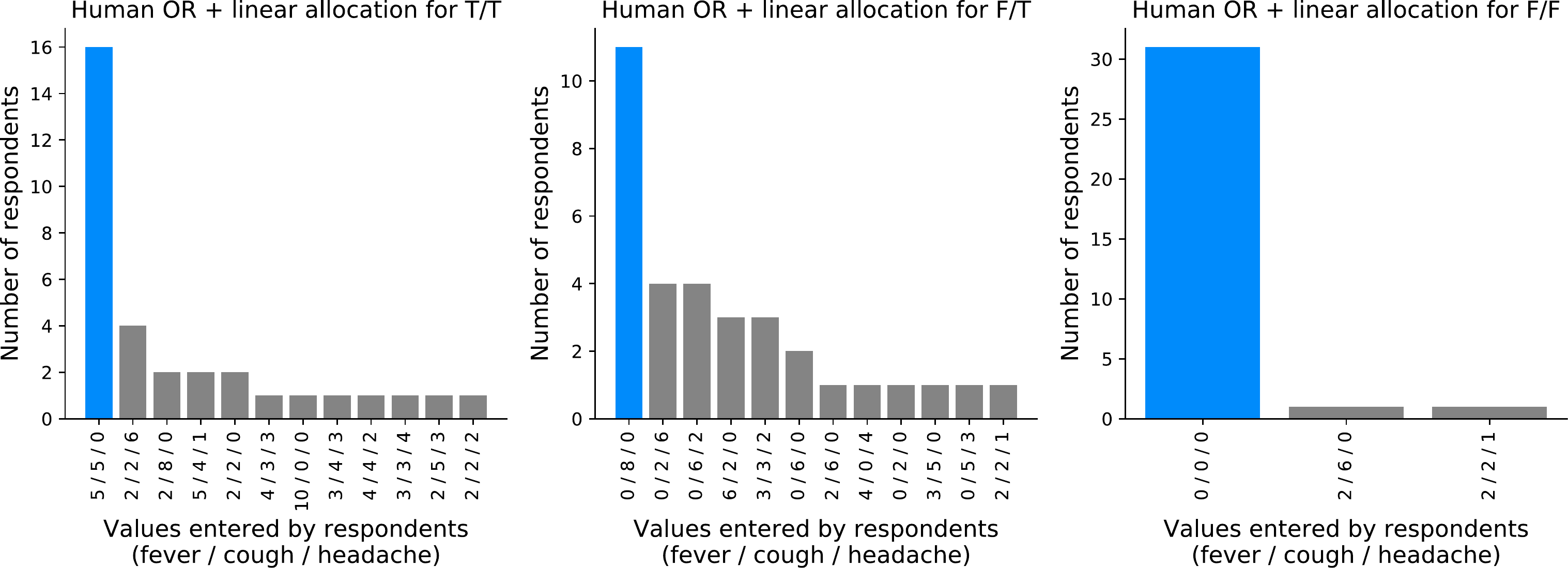}
  \caption{{\bf Consensus human intuition values for an OR function.} Human consensus values were measured by a user study over 33 participants for a simple OR-based function. The most popular allocation was chosen as the consensus (\ref{sec:methods_user_study}). The labels denote the allocation given by people as ``fever / cough / headache''. The title gives the input values for sample being explained as ``fever value/cough value'', where `T' is true and `F' is false; note that headache is always set to true.}
  \label{fig:human_or_survey}
\end{figure*}

\begin{figure*}
  \centering
  \includegraphics[width=1.0\textwidth]{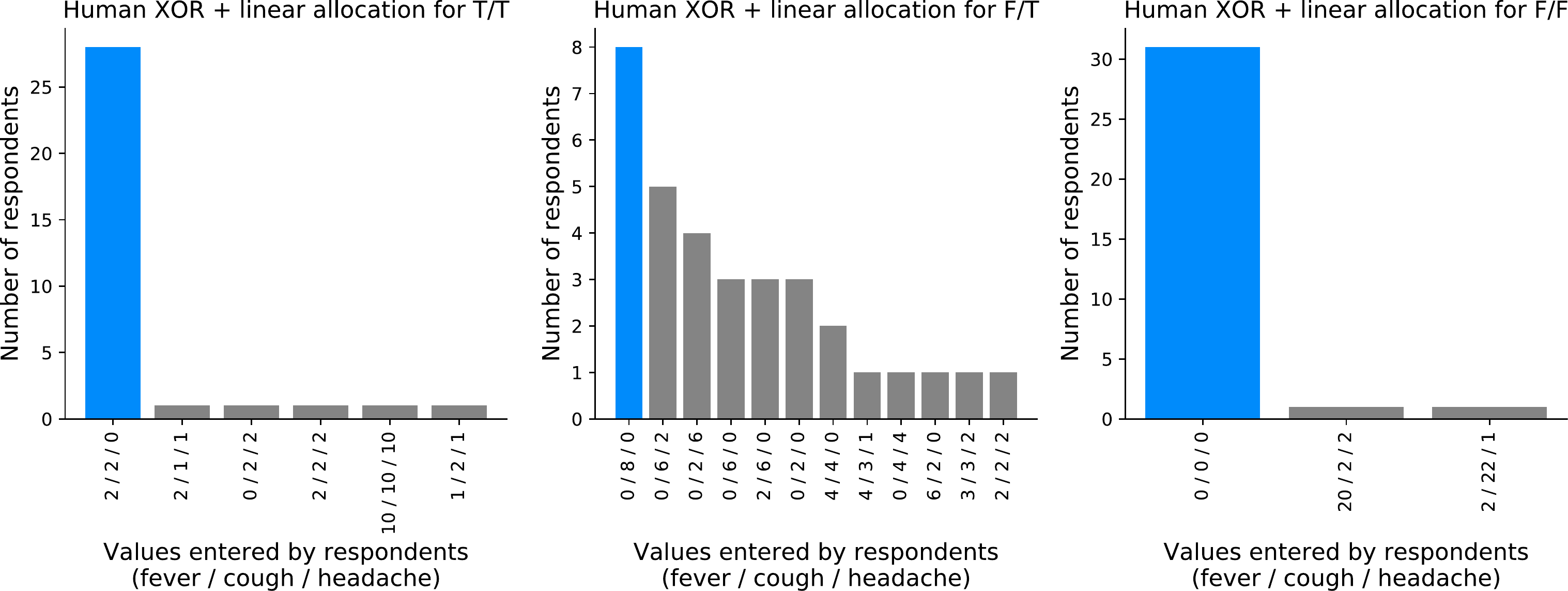}
  \caption{{\bf Consensus human intuition values for an eXclusive OR (XOR) function.} Human consensus values were measured by a user study over 33 participants for a simple XOR-based function. The most popular allocation was chosen as the consensus (\ref{sec:methods_user_study}). The labels denote the allocation given by people as ``fever / cough / headache''. The title gives the input values for sample being explained as ``fever value/cough value'', where `T' is true and `F' is false; note that headache is always set to true.}
  \label{fig:human_xor_survey}
\end{figure*}

\begin{figure*}
  \centering
  \includegraphics[width=1.0\textwidth]{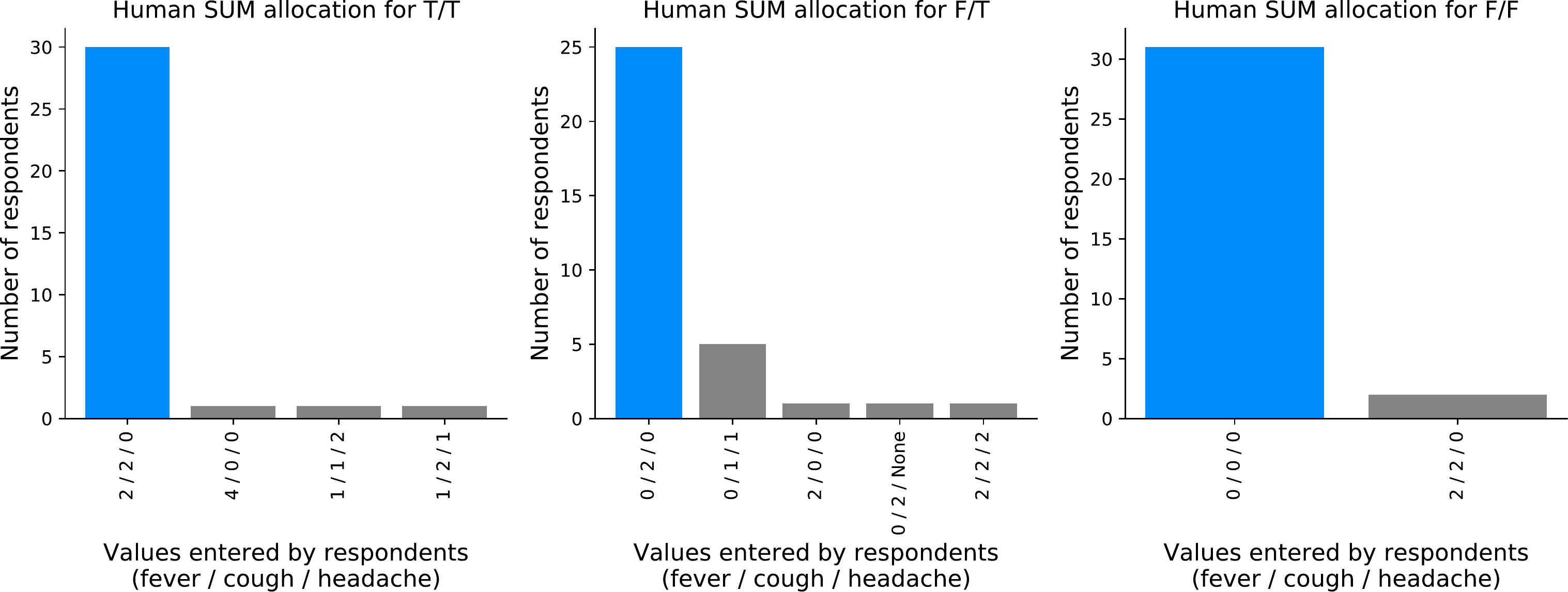}
  \caption{{\bf Consensus human intuition values for a SUM function.} Human consensus values were measured by a user study over 33 participants for a simple SUM function. The most popular allocation was chosen as the consensus (\ref{sec:methods_user_study}). The labels denote the allocation given by people as ``fever / cough / headache''. The title gives the input values for sample being explained as ``fever value/cough value'', where `T' is true and `F' is false; note that headache is always set to true.}
  \label{fig:human_sum_survey}
\end{figure*}

\begin{figure*}
  \centering
  \includegraphics[width=1.0\textwidth]{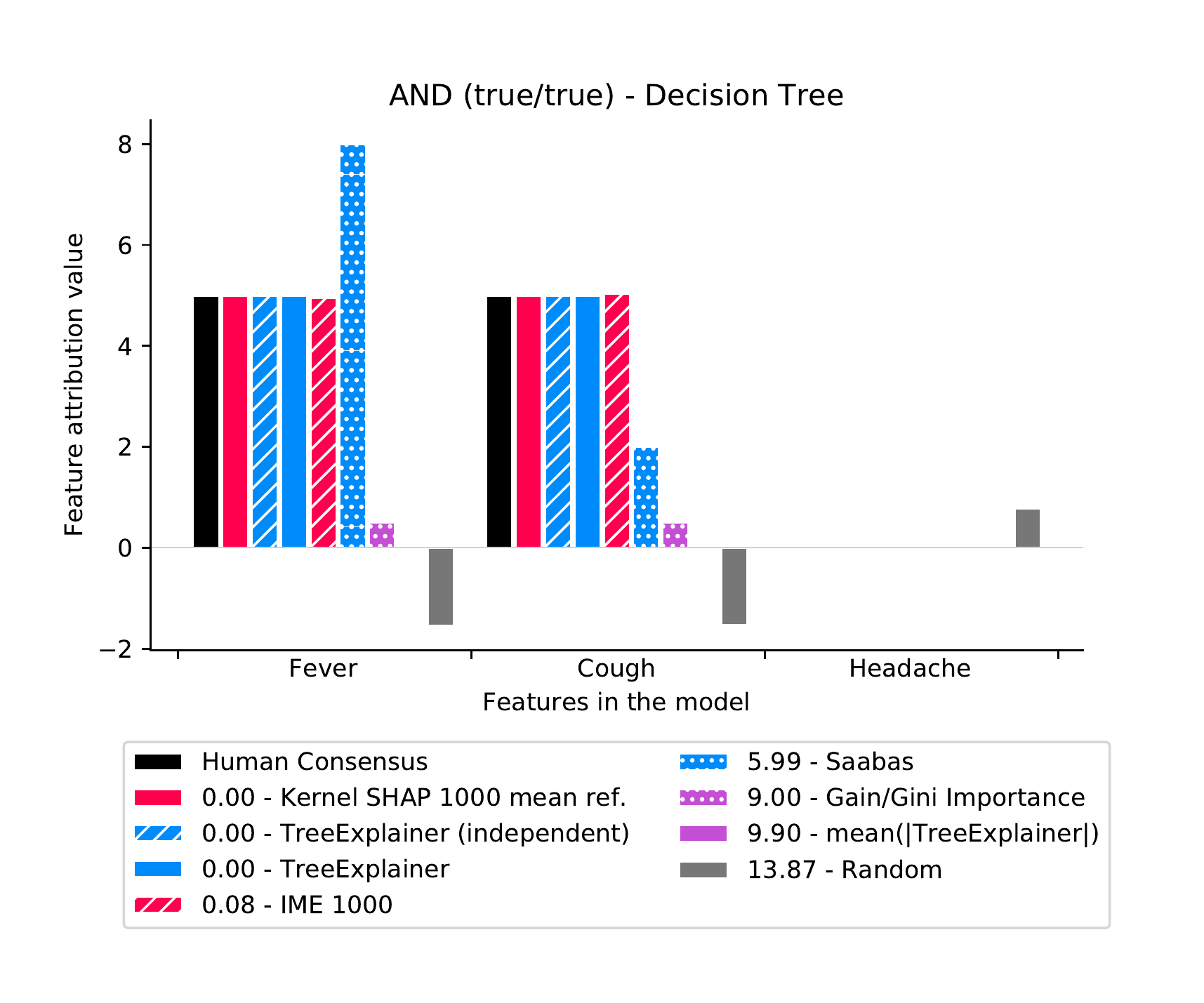}
  \caption{{\bf Comparison with human intuition for an AND function.} Human consensus values were measured by a user study for a simple AND-based function evaluated when all inputs were set to `true' (Supplementary Figure~\ref{fig:human_and_survey}). The results of different explanation methods were then compared to these consensus values to measure their consistency with human intuition (\ref{sec:methods_user_study}).}
  \label{fig:plot_human_decision_tree_human_and_11}
\end{figure*}

\begin{figure*}
  \centering
  \includegraphics[width=1.0\textwidth]{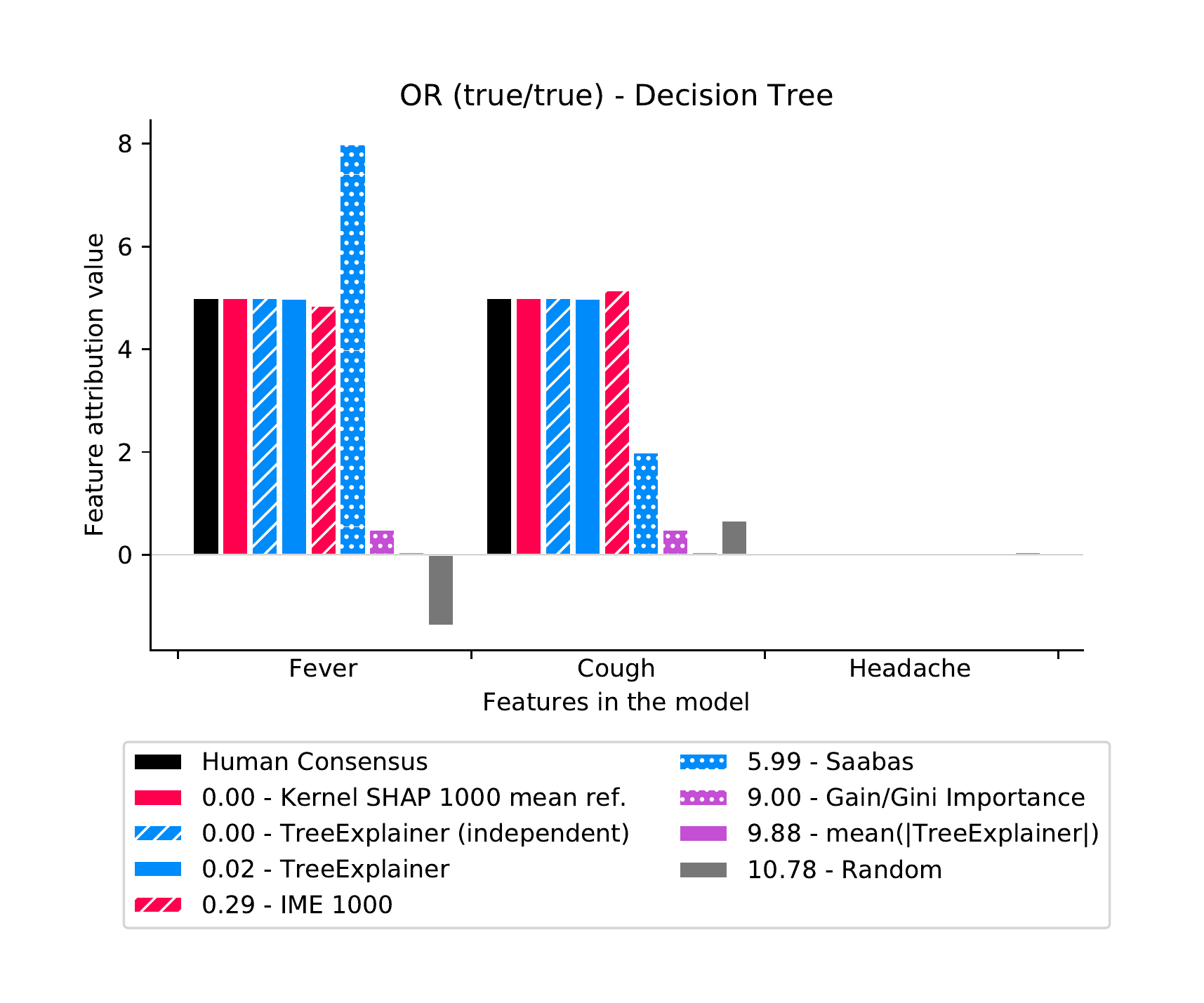}
  \caption{{\bf Comparison with human intuition for an OR function.} Human consensus values were measured by a user study for a simple OR-based function evaluated when all inputs were set to `true' (Supplementary Figure~\ref{fig:human_or_survey}). The results of different explanation methods were then compared to these consensus values to measure their consistency with human intuition (\ref{sec:methods_user_study}).}
  \label{fig:plot_human_decision_tree_human_or_11}
\end{figure*}

\begin{figure*}
  \centering
  \includegraphics[width=1.0\textwidth]{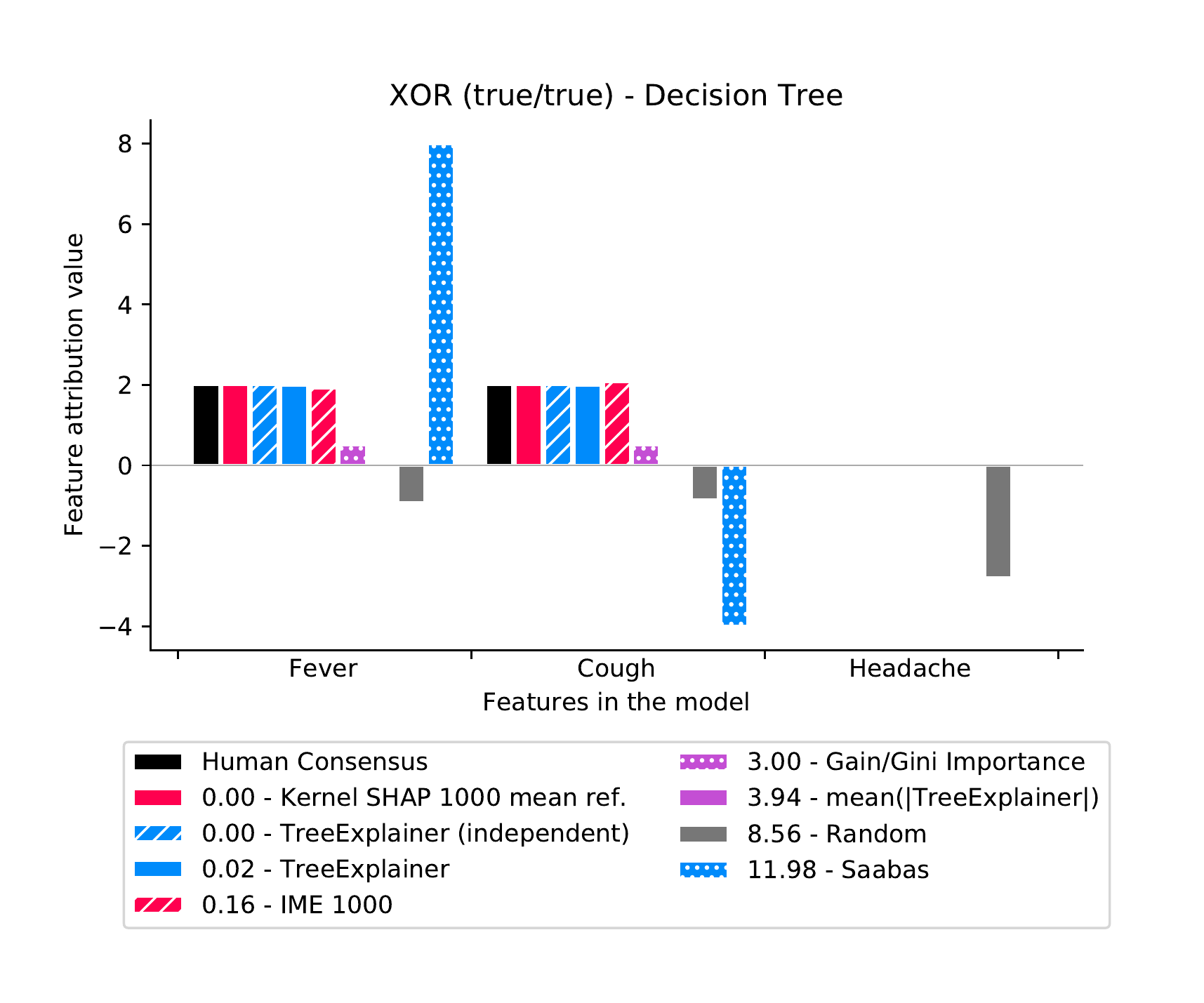}
  \caption{{\bf Comparison with human intuition for an eXclusive OR (XOR) function.} Human consensus values were measured by a user study for a simple XOR-based function evaluated when all inputs were set to `true' (Supplementary Figure~\ref{fig:human_xor_survey}). The results of different explanation methods were then compared to these consensus values to measure their consistency with human intuition (\ref{sec:methods_user_study}).}
  \label{fig:plot_human_decision_tree_human_xor_11}
\end{figure*}

\clearpage
\printbibliography

\end{document}